\documentclass[nohyperref]{article}

\usepackage{microtype}
\usepackage{graphicx}
\usepackage{booktabs} 

\usepackage{hyperref}

\usepackage[accepted]{icml2022}

\usepackage{amsmath}
\usepackage{amssymb}
\usepackage{mathtools}
\usepackage{amsthm}

\usepackage[capitalize,noabbrev]{cleveref}

\usepackage{nicefrac}
\usepackage{times}
\usepackage{commath}
\usepackage{color}
\usepackage{url}
\usepackage{listings}
\usepackage{accents}
\usepackage{multirow}
\usepackage{subcaption}

\usepackage{algorithm}
\usepackage{algorithmic}    

\newcommand\algorithmicprocedure{\textbf{procedure}}
\newcommand{\algorithmicendprocedure}{\algorithmicend\ \algorithmicprocedure}
\makeatletter
\newcommand\PROCEDURE[3][default]{%
  \ALC@it
  \algorithmicprocedure\ \textsc{#2}(#3)%
  \ALC@com{#1}%
  \begin{ALC@prc}%
}
\newcommand\ENDPROCEDURE{%
  \end{ALC@prc}%
  \ifthenelse{\boolean{ALC@noend}}{}{%
    \ALC@it\algorithmicendprocedure
  }%
}
\newenvironment{ALC@prc}{\begin{ALC@g}}{\end{ALC@g}}
\makeatother

\theoremstyle{plain}
\newtheorem{theorem}{Theorem}[section]

\newtheorem{lemma}[theorem]{Lemma}
\newtheorem{corollary}[theorem]{Corollary}
\theoremstyle{definition}

\newtheorem{assumption}[theorem]{Assumption}
\theoremstyle{remark}
\newtheorem{remark}[theorem]{Remark}

\usepackage[textsize=tiny]{todonotes}

\icmltitlerunning{Hessian-Free High-Resolution Nesterov Acceleration For Sampling}

\newcommand{\bs}{\boldsymbol}
\newcommand{\qbold}{\bs{q}}
\newcommand{\pbold}{\bs{p}}

\newcommand{\nablax}{\nabla_{\bs{x}}}
\newcommand{\nablap}{\nabla_{\bs{p}}}
\newcommand{\nablaq}{\nabla_{\bs{q}}}

\newcommand{\deltaqp}{\begin{bmatrix}\qbold_t - \tilde{\qbold}_t \\ \pbold_t - \tilde{\pbold}_t \end{bmatrix}}
\newcommand{\phipsi}{\begin{bmatrix} \bs{\phi}_t \\ \bs{\psi}_t \end{bmatrix}}

\newcommand{\tqbold}{\tilde{\qbold}}
\newcommand{\tpbold}{\tilde{\pbold}}
\newcommand{\m}[1]{\begin{bmatrix} #1 \end{bmatrix}}

\newcommand{\polylog}{\tilde{\mathcal{O}}}

\usepackage{tikz}
\newcommand*\circled[1]{\tikz[baseline=(char.base)]{
            \node[shape=circle,draw,inner sep=2pt] (char) {#1};}}

\DeclarePairedDelimiterX{\innerprod}[2]{\langle}{\rangle}{#1, #2}

\def\argmin{\mathop{\rm argmin}}

\begin{document}

\twocolumn[
\icmltitle{Hessian-Free High-Resolution Nesterov Acceleration For Sampling}



\icmlsetsymbol{equal}{*}

\begin{icmlauthorlist}


\icmlauthor{Ruilin Li}{xxx}
\icmlauthor{Hongyuan Zha}{yyy}
\icmlauthor{Molei Tao}{xxx}
\end{icmlauthorlist}

\icmlaffiliation{xxx}{School of Mathematics, Georgia Institute of Technology}
\icmlaffiliation{yyy}{School of Data Science, The Chinese University of Hong Kong, Shenzhen, Shenzhen Institute of Artificial Intelligence and Robotics for Society}

\icmlcorrespondingauthor{Molei Tao}{mtao@gatech.edu}

\icmlkeywords{MCMC, Nesterov Accelerated Gradient, Langevin-type dynamics, accelerated sampling}

\vskip 0.3in
]



\printAffiliationsAndNotice{}  

\begin{abstract}
Nesterov's Accelerated Gradient (NAG) for optimization has better performance than its continuous time limit (noiseless kinetic Langevin) when a finite step-size is employed \citep{shi2021understanding}. This work explores the sampling counterpart of this phenonemon and proposes a diffusion process, whose discretizations can yield  accelerated gradient-based MCMC methods. More precisely, we reformulate the optimizer of NAG for strongly convex functions (NAG-SC) as a Hessian-Free High-Resolution ODE, change its high-resolution coefficient to a hyperparameter, inject appropriate noise, and discretize the resulting diffusion process. The acceleration effect of the new hyperparameter is quantified and it is not an artificial one created by time-rescaling. Instead, acceleration beyond underdamped Langevin in $W_2$ distance is quantitatively established for log-strongly-concave-and-smooth targets, at both the continuous dynamics level and the discrete algorithm level. Empirical experiments in both log-strongly-concave and multi-modal cases also numerically demonstrate this acceleration.
\end{abstract}

\section{Introduction}
Optimization is a major machinery that drives the theory and practice of machine learning in recent years. Since the seminal work of \citet{nesterov1983method}, acceleration has played a key role in gradient-based optimization methods. A notable example is Nesterov's Accelerated Gradient (NAG), which is an instance of a more general family of ``momentum methods". NAG consists of multiple methods, including NAG-C and NAG-SC, respectively for convex and strongly convex functions. Both provably converge faster than vanilla gradient descent (GD) in their corresponding setups \citep{nesterov1983method, nesterov2013introductory}. Newer perspectives of acceleration continue to be revealed, e.g., \cite{su2014differential, wibisono2016variational, wilson2016lyapunov, hu2017dissipativity, attouch2018fast,shi2021understanding}, many based on the interplay between continuous and discrete times. This work aims at turning NAG-SC into a sampler based on this interplay.

In fact, approaches for sampling statistical distributions, such as gradient-based Markov Chain Monte Carlo (MCMC) methods, are also of great importance in machine learning, for example due to their links to statistical inference and abilities to represent uncertainties lacking in optimization-based methods. Although not entirely the same thing, optimization and sampling are closely related: besides seeing a large class of sampling dynamics as optimization dynamics with additional noise, viewing sampling as optimization in probability space also led to important discoveries (e.g., \citealp{jordan1998variational, liu2016stein, dalalyan2017further, wibisono2018sampling, zhang2018policy, frogner2018approximate, chizat2018global, chen2018unified, ma2019there, erdogdu2020convergence}). In fact, an unadjusted Euler-Maruyama discretization of overdamped Langevin dynamics (abbreviated as OLD here) is commonly considered as the analog of GD in sampling (although many other discretizations are also possible), and often referred to as Unadjusted Langevin Algorithm (ULA) \citep{roberts1996exponential} and/or Langevin Monte Carlo (LMC). The convergence properties of the continuous dynamics of OLD, as well as asymptotic and non-asymptotic analyses of its discretizations have been extensively studied (e.g., \citealp{roberts1996exponential, villani2008optimal, pavliotis2014stochastic,  dalalyan2017theoretical, durmus2016sampling, dalalyan2017further, durmus2019analysis, durmus2019high, vempala2019rapid, cheng2018convergence, dwivedi2019log, ma2019sampling, chewi2020optimal, erdogdu2020convergence}).

Meanwhile, the notion of acceleration is less quantified in sampling compared to that in optimization, although attention has been rapidly building up. Along this direction, one line is based on diffusion processes such as underdamped Langevin dynamics (ULD). For example, the convergence and nonasymptotics of discretized ULD have been studied by \citet{cheng2017underdamped, dalalyan2018sampling, ma2019there}, and were demonstrated provably faster than discretized OLD in suitable setups. These are not only great progresses but also forming perspectives complementary to the extensive studies of the convergence of continuous ULD in the mathematical community (e.g, \citealp{mattingly2002ergodicity, cao2019explicit, dolbeault2009hypocoercivity, dolbeault2015hypocoercivity, villani2009hypocoercivity, eckmann2003spectral, baudoin2013bakry, eberle2019couplings}). Another important line of research is related to accelerating particle-based approaches for optimization in probability spaces \citep{liu2019understanding, pmlr-v97-taghvaei19a, wang2019accelerated}, although we note there is no clear boundary between these two lines (e.g., \citealp{leimkuhler2018ensemble}). Additional interesting ideas also include \cite{chen2018accelerating,deng2020non,ding2021random,li2022mirror,liang2022proximal}. In general, it has been known that adding an irreversible part to the reversible dynamics of OLD\footnote{For irreversible-acceleration \emph{not} from OLD, see e.g., \cite{bierkens2019zig,bouchard2018bouncy}.} accelerates its convergence (e.g., \citealp{hwang2005accelerating,lelievre2013optimal,ohzeki2015langevin,rey2015irreversible,duncan2016variance}), and this work can be viewed to be under this umbrella. Note, though, the discretization of an accelerated continuous process is also important, and it will also be discussed.

Specifically, we propose a class of accelerated gradient-based MCMC algorithms termed HFHR. It is motivated by a simple question: how to appropriately inject noise to NAG algorithm \textbf{in discrete time}, so that it is turned into an algorithm for momentum-accelerated sampling? Note we don't add noise to the learning-rate$\to 0$ limit of NAG (this has been studied in \citealp{ma2019there}), because a finite-step-size discretization of this limiting ODE may not converge as fast as NAG with the same learning rate. However, we will still use continuous dynamics as intermediate steps.

More precisely, our first step is to combine existing tools to prepare a non-asymptotic formulation for the later steps. The goal is to better account for NAG's behavior when a finite (not infinitesimal) learning rate is used. As pointed out in \citet{shi2021understanding}, a low-resolution limiting ODE \citep{su2014differential}, albeit being a milestone leading to important research (e.g, \citealp{wibisono2016variational}), does not fully capture the acceleration enabled by NAG --- for example, it can't distinguish between NAG and another momentum method of heavy ball \citep{polyak1964some}. A reason is, the low-resolution ODE describes the $h\to 0$ limit of NAG, but in practice NAG uses a finite (nonzero) $h$. High-resolution ODE was thus proposed to include additional $\mathcal{O}(h)$ terms to account for the finite $h$ effect \citep{shi2021understanding}. 
The original form of high-resolution ODE involves Hessian of the objective function, which is computationally expensive to evaluate and store for high-dimensional problems, but this is a small obstacle that can be overcome (see e.g., \citealp{alvarez2002second, attouch2020first}), and we'll be able to derive a High-Resolution and Hessian-Free limiting ODE for NAG.

Then we replace the high-resolution term's coefficient in the HFHR ODE by a hyperparameter $\alpha \ge 0$, and then add noise to the resulting ODE in a specific way, which turns it into an SDE suitable for the sampling purpose. This SDE will be termed as HFHR dynamics.

To obtain an actual algorithm, the HFHR SDE is then discretized. We will see, both theoretically and empirically, that nonzero $\alpha$ can lead to accelerated convergence of the sampling algorithm; this acceleration is \textbf{not} an artificial consequence of time-rescaling, which would not give acceleration after discretization with an appropriate step size. For demonstrating this, we will be primarily  working with just a 1st-order discretization, which uses 1 (full-)gradient evaluation per iteration and thus suits particularly well low-to-medium-accuracy downstream applications; comparisons will be mainly against other methods that use 1 gradient per step as well. However, since high-order discretizations can improve statistical accuracy and even the speed of convergence (see e.g., \citealp{chen2015convergence,li2019stochastic,shen2019randomized}), 
we will also provide a high-order discretization in Appendix \ref{sec:RMA_HFHR}, which again exhibits acceleration and suits high-accuracy applications.

Our presentation is as follows: After detailing the construction of HFHR, we will analyze its convergence, at both the continuous level (HFHR dynamics) and the discrete level (HFHR algorithm). For precise theoretical results, we will consider the
setup of log-strongly-concave target distributions, which are commonly considered in the literature \citep{kim2016global, bubeck2018sampling, dalalyan2017theoretical, dalalyan2018sampling, dwivedi2019log, shen2019randomized}. 
The additional acceleration of HFHR when compared to ULD in continuous time will be demonstrated explicitly in Thm.\ref{thm:exp_convergence_coupling}. For our discretized HFHR algorithm, a non-asymptotic error bound will be obtained (Thm.\ref{thm:wasserstein}), which confirms that the additional acceleration in continuous time carries through to the discrete territory. Finally, numerical experiments are provided, verifying the validity and tightness of our theoretical results, and empirically showing HFHR remains advantageous for the nonconvex and high-dim. problems, e.g., Bayesian Neural Networks.

The main contribution of this article is the idea of turning NAG-SC optimizer into a sampler, and the introduction of a new dynamics that is neither overdamped or underdamped Langevin. Theoretical analyses (e.g., Thm.\ref{thm:wasserstein}, Cor.\ref{corollary:iteration complexity} \& Rmk.\ref{remark:iteration_complexity}) and numerical experiments (Sec.\ref{sec:experiment}) are provided for quantifying the effectiveness of this idea.

\section{Background: Langevin Dynamics}
Consider sampling from Gibbs measure $\mu$ whose density is $d\mu = \frac{1}{\int e^{- f(\bs{y})}d\bs{y}} e^{-f(\bs{x})} d\bs{x}$,
where $f: \mathbb{R}^d \mapsto \mathbb{R}$ will be called the potential function. Two diffusion processes popular for sampling (and modeling important physical processes too) are named after Langevin. One is overdamped Langevin dynamics (OLD), and the other is kinetic Langevin dynamics (abbreviated as ULD to comply with a convention of calling it underdamped Langevin). They are respectively given by
\begin{align*}
    &\mbox{(OLD)} \quad d\qbold_t = -\nabla f(\qbold_t)dt + \sqrt{2}d\bs{W}_t 
    \\
    &\mbox{(ULD)} \quad
    \begin{cases}
        d\qbold_t = \pbold_t dt\\
        d\pbold_t = -\gamma \pbold_t dt - \nabla f(\qbold_t) dt + \sqrt{2\gamma} d\bs{B}_t  
    \end{cases}         
\end{align*}
where $\qbold_t, \pbold_t \in \mathbb{R}^d$, $\bs{W}_t, \bs{B}_t$ are i.i.d. Wiener processes in $\mathbb{R}^d$, and $\gamma > 0$ is a friction coefficient. Under mild conditions (e.g., \citealp{pavliotis2014stochastic}), OLD converges to $\mu$ and ULD converges to $ d\pi(\qbold, \pbold) = d\mu(\qbold) \nu(\pbold) d\pbold, \, \text{where } \nu(\pbold) = (2\pi)^{-\frac{d}{2}} e^{ -\nicefrac{\|\pbold\|^2}{2}}$
, so its $\qbold$ marginal follows $\mu$.

OLD and ULD are closely related. In fact, OLD is the $\gamma\to\infty$ overdamping limit of ULD after time dilation (e.g., \citealp{pavliotis2014stochastic}). However, OLD is a reversible Markov process but ULD is irreversible, and thus both their equilibrium and non-equilibrium statistical mechanics are different, although closely related too. We will only focus on the convergence to statistical equilibrium (see e.g., \citealp{souza2019metastable} for non-equilibrium aspects).

Many celebrated approaches exist for establishing the exponential convergence (a.k.a. geometric ergodicity) of OLD, including the seminal work of \cite{roberts1996exponential}
, the ones using spectral gap (e.g., \citealp[Lemma 1]{dalalyan2017theoretical}), synchronous coupling (\citealp[p33-35]{villani2008optimal}; \citealp[Proposition 1]{durmus2019high}), functional inequalities such as Poincar\'e's inequality \citep[Theorem  4.4]{pavliotis2014stochastic} and log Sobolev inequality \citep[Theorem  1]{vempala2019rapid}. There are also fruitful results for ULD, including the ones leveraging Lyapunov function \citep[Theorem  3.2]{mattingly2002ergodicity}, hypocoercivity \citep{villani2009hypocoercivity,dolbeault2009hypocoercivity, dolbeault2015hypocoercivity, roussel2018spectral}, coupling (\citealp[Theorem 5]{cheng2017underdamped}; \citealp[Theorem 1]{dalalyan2018sampling}; \citealp[Theorem 2.3]{eberle2019couplings}), LSI \citep[Section 3.1]{ma2019there}, modified Poincar\'e's inequality \citep[Theorem  1]{cao2019explicit}, and spectral analysis \citep{kozlov1989effective, eckmann2003spectral}.


The study of asymptotic convergence of discretized OLD dates back to at least the 1990s \citep{meyn1994computable, roberts1996exponential}. The non-asymptotic analysis of LMC discretization of OLD 
can be found in \cite{dalalyan2017theoretical} and it shows the discretization achieves $\epsilon$ error, in TV distance, in $\polylog(\nicefrac{d}{\epsilon^2})$ steps. Subsequent results include $\polylog(\nicefrac{d}{\epsilon^2})$ in $W_2$ \citep{durmus2016sampling}, $\polylog(\nicefrac{d}{\epsilon})$ in KL \citep{cheng2018convergence}, $\polylog(\nicefrac{d}{\epsilon})$ in $W_2$ under additional 3rd-order regularity \citep{durmus2019high}, and $\polylog(\nicefrac{\sqrt{d}}{\epsilon})$ in $W_2$ under additional 3rd-order regularity \citep{li2021sqrt}.
For discretized ULD, one has  $\polylog(\nicefrac{\sqrt{d}}{\epsilon})$ iteration complexity in $W_2$ \citep{cheng2017underdamped, dalalyan2018sampling} and   $\polylog(\nicefrac{\sqrt{d}}{\sqrt{\epsilon}})$ in KL \citep{ma2019there}. 
ULD is still generally conceived to be advantageous over OLD and sometimes understood as its momentum-accelerated version.



\section{Notations and Conditions}
\label{sec:notation}

We will use 2-Wasserstein distance to quantify convergence, i.e. $W_2(\mu_1, \mu_2) = \left(\inf_{\pi \in \Pi(\mu_1, \mu_2)} \mathbb{E}_{(\bs{X}, \bs{Y}) \sim \pi} \norm{\bs{X} - \bs{Y}}^2 \right)^\frac{1}{2}$ where $\Pi(\mu_1, \mu_2)$ is the set of all couplings of $\mu_1$ and $\mu_2$.

Assume WLOG that  $ \bs{0} \in \argmin_{\bs{x} \in \mathbb{R}^d} f(\bs{x})$.
The following condition will also be frequently used.
\begin{assumption}\label{asp:standard} (\textbf{Standard Strong-Convexity and Smoothness Condition})
    A function $f \in \mathcal{C}^1(\mathbb{R}^d,\mathbb{R})$ is $m$-stronly-convex and $L$-smooth, if there exist constants $m, L > 0$ such that $\forall  \bs{x}, \bs{y} \in \mathbb{R}^d$, we have 
    \begin{align*}
        & \|\nabla f(\bs{y}) - \nabla f(\bs{x})\| \le L\|\bs{y} - \bs{x}\| \qquad \mbox{ and } \\
        & f(\bs{y})  \geq f(\bs{x}) + \innerprod{\nabla f(\bs{x})}{\bs{y} - \bs{x}} + \frac{m}{2} \|\bs{y} - \bs{x}\|^2
    \end{align*}
\end{assumption}
For $f\in\mathcal{C}^2$, this is equivalent to $mI \preceq \nabla^2 f \preceq LI$. 

The condition number of $f$ is defined as $\kappa := \nicefrac{L}{m}$.

\section{The Construction of HFHR dynamics}
HFHR is obtained by formulating NAG-SC as a Hessian free high-resolution ODE, lifting the high-resolution term's coefficient as a free parameter, and adding appropriate noises.


More precisely, let's start with NAG-SC algorithm:
\begin{align}
    \bs{x}_{k+1} &= \bs{y}_k - s \nabla f(\bs{y}_k) \label{eq_NAG_SC_1} \\
    \bs{y}_{k+1} &= \bs{x}_{k+1} + c(\bs{x}_{k+1} - \bs{x}_k) \label{eq_NAG_SC_2}
\end{align}
where $s$ is the learning rate (also known as step size), and $c = \frac{1 - \sqrt{ms}}{1 + \sqrt{ms}}$ is a constant based on $s$ and the strong convexity coefficient $m$ of $f$; the method also works for non-strongly-convex $f$ though.

A high-resolution ODE description of Eq.\eqref{eq_NAG_SC_1} \& \eqref{eq_NAG_SC_2} is obtained in \citet[Section 2]{shi2021understanding}
\begin{equation}\label{eq:high_resolution_hessian}
    \ddot{\bs{y}} + \sqrt{s}\left(\frac{2(1-c)}{s(1+c)} + \nabla^2 f(\bs{y})\right)\dot{\bs{y}} + \frac{2}{1+c}\nabla f(\bs{y}) = \bs{0}, 
\end{equation}
which can better account for the effect of non-infinitesimal $s$ than the $s\to 0$ limit (note $c$ depends on $s$). However, in this original form, Hessian of $f$ is involved, which is expensive to compute and store for high-dimensional problems.

To obtain a Hessian-free high-resolution ODE description of Eq.\eqref{eq_NAG_SC_1} \& \eqref{eq_NAG_SC_2},  we first turn the iteration into a `mechanical' version by introducing position $\qbold_k = \bs{y}_k$ and momentum $\bs{p}_k=\nicefrac{(\bs{y}_k-\bs{x}_k)}{h}$. Replacing $\bs{x}_{k+1}$ in \eqref{eq_NAG_SC_1} and the first $\bs{x}_{k+1}$ in \eqref{eq_NAG_SC_2} by $\qbold_{k+1}$ and $\bs{p}_{k+1}$, the second $\bs{x}_{k+1}$ in \eqref{eq_NAG_SC_2} by $\bs{q}_k - s\nabla f(\bs{q}_k)$, and the $\bs{x}_k$ in \eqref{eq_NAG_SC_2} by $\bs{q}_k$ and $\bs{p}_k$, we obtain
\[ 
    \begin{cases}
        \bs{p}_{k+1} = c \bs{p}_k - c \frac{s}{h} \nabla f(\bs{q}_k) \\
        \bs{q}_{k+1} = \bs{q}_k + h \bs{p}_{k+1} - s \nabla f(\bs{q}_k)
    \end{cases} 
\]
Now, choose $\gamma$, $\alpha$ and $h$ as
$ 
    h = \sqrt{cs},\gamma = \frac{1-c}{h}, \alpha = \frac{s}{h}
$.
We see that $\gamma>0$, $\alpha >0$, and NAG-SC exactly rewrites as
\begin{equation}\label{eq_NAG_SC_rewritten}
    \begin{cases}
        \pbold_{k+1} = \pbold_k -h \gamma \pbold_k - h \nabla f(\qbold_k) \\
        \qbold_{k+1} = \qbold_k + h \pbold_{k+1} - h \alpha \nabla f(\qbold_k)
    \end{cases}.
\end{equation}
Note the technique for bypassing the Hessian without introducing any approximation is already well studied in the literature (e.g., \citealp{alvarez2002second, attouch2020first}).

So far, both $h$ and $\alpha$ are actually determined by the hyperparameter $s$ of NAG-SC. However, if we now consider $\alpha$ as an independent variable (i.e., `lift' it) and let $h\to 0$, we see \eqref{eq_NAG_SC_rewritten} is a 1st-order discretization (with step size $h$) of
\begin{equation}
\begin{cases}
    \dot{\qbold} = \pbold - \alpha \nabla f(\qbold) \\
    \dot{\pbold} = -\gamma \pbold - \nabla f(\qbold)
\end{cases}.
\label{eq_HFHRconti}
\end{equation}
Note $\alpha$, if inherited from NAG-SC, should be $\alpha=\sqrt{s/c}=\mathcal{O}(h)$, which, in a low-resolution ODE, will be discarded, and this eventually leads to ULD rather than HFHR. However, we now allow it to be a free parameter and will see that $\alpha\neq \mathcal{O}(h)$ can be advantageous.

Before quantifying these advantages, we finish the construction by appropriately injecting Gaussian noises to \eqref{eq_HFHRconti}. This is just like how OLD can be obtained by adding noise to gradient flow. The right amount and structure of noise turn the ODE into a Markov process that can serve the purpose of sampling, and the detailed form of our noise is given by:
\begin{equation}\label{eq:HFHR}
    \begin{cases}
        d\qbold_t = (\pbold_t - \alpha \nabla f(\qbold_t))dt + \sqrt{2\alpha} d\bs{W}_t \\
        d\pbold_t = (-\gamma \pbold_t - \nabla f(\qbold_t))dt + \sqrt{2\gamma} d\bs{B}_t
    \end{cases}.
\end{equation}
Here $\alpha \ge 0 , \gamma > 0$ are constant parameters, and $\bs{W}_t, \bs{B}_t$ are independent standard Brownian motions in $\mathbb{R}^d$. This irreversible process will be named as \textbf{Hessian-Free High-Resolution}(\textbf{HFHR}) dynamics. We write it as HFHR$(\alpha,\gamma)$ to emphasize the dependence on $\alpha$ and $\gamma$ when needed.

Substitution into Fokker-Planck PDE shows HFHR dynamics is unbiased (proof in Appendix \ref{sec:proofOfInvariantDistribution}):
\begin{theorem}\label{thm:invariant_distribution}
    $\pi$ is the invariant distribution of HFHR described in Eq.\eqref{eq:HFHR}, just like ULD.
\end{theorem}


 \begin{remark}
 Although the right hand side of \eqref{eq:HFHR} can be formally viewed as the sum of OLD and ULD's right hand sides, HFHR dynamics can be very different from both OLD and ULD. In fact, it is generally true that a differential equation, whose right hand side is the sum of the right hand sides of two other differential equations, can behave very differently from either of the two; this is studied under the subject of `operator splitting' (e.g., \citealp{Trotter:59}).
\end{remark}



\vskip -0.1cm
\section{Theoretical Analysis of the HFHR Dynamics and Algorithm} 
\vskip -0.1cm
\subsection{HFHR Dynamics in Continuous Time}\label{sec:theoryConti}

We now quantify the exponential convergence of HFHR dynamics and its additional acceleration over ULD, when the target measure has a strongly-convex and smooth potential.
\begin{theorem}\label{thm:exp_convergence_coupling}
Assume Conditions A\ref{asp:standard} holds and further assume $\gamma^2 > L + m$ and $\alpha \le \frac{\gamma^2 - L - m}{m\gamma}$. Denote the law of $\qbold_t$ by $\mu_t$. Then there exists $\kappa^\prime > 0$ depending only on $\alpha$ and $\gamma$, such that
\[
    W_2(\mu_t, \mu) \le \kappa^\prime e^{-(\frac{m}{\gamma} + m\alpha)t} W_2(\mu_0, \mu).
\]
Detailed expression of $\kappa^\prime$ can be found in Appendix \ref{app:notation}.
\end{theorem}


Thm. \ref{thm:exp_convergence_coupling} state that HFHR dynamics converges to the target distribution exponentially fast in log-strongly-concave-and-smooth setups. There is an additional acceleration created by $\alpha$ (the HFHR correction) in the exponent.

As a sanity check, note for ULD (i.e. HFHR($\alpha=0$,$\gamma$)), \citet[Theorem \ 1]{dalalyan2018sampling} obtained exponential convergence result in 2-Wasserstein distance with rate $\frac{\sqrt{m}}{\sqrt{\kappa} + \sqrt{\kappa - 1}}$ using a simple and elegant coupling approach, and showed this rate is optimal as it is achieved by the bivariate function $f(x,y) = \frac{m}{2}x^2 + \frac{L}{2}y^2$. In this case, Thm \ref{thm:exp_convergence_coupling} gives an (asymptotically) equivalent rate $\frac{\sqrt{m}}{2\sqrt{\kappa}}$, and thus our result passes the check. Also in this sense, we're not making a shaky claim of advantage by comparing bounds (as they may not be tight); instead, bounds that are being compared here can actually be attained (see Rmk.\ref{rmk:discretizationGaussianConditionNumImproved} for an analogue after discretization).


Now, given that both $\gamma$ and $\alpha$ are hyperparameters that affect the convergence rate and they are dependent due to the constraints, we illustrate the acceleration enabled by $\alpha$ more precisely by considering a low bound of it: set $\gamma = 2\sqrt{L}$ and push $\alpha$ to the upper bound specified in Thm. \ref{thm:exp_convergence_coupling}; then we obtain an $\mathcal{O}(\sqrt{L})$ rate in the log-strongly-concave setup.
Compared with the rate in \citep{dalalyan2018sampling}, this is a speed-up of order $\kappa$.

\subsection{HFHR Algorithm in Discrete Time}\label{sec:discretization}
To obtain an implementable method, we now discretize the time of HFHR dynamics. As our main goal is to show the acceleration enabled by $\alpha$ won't disappear after discretization (unlike a fake acceleration due to time rescaling), we'll just analyze a 1st-order discretization (but a high-accuracy discretization adapted from RMA \citep{shen2019randomized} will also be provided and compared with RMA, in Appendix \ref{sec:RMA_HFHR}).

For simplicity, we work with constant step size $h$. Inspired by Strang splitting for differential equations \citep{strang1968construction, mclachlan2002splitting}, consider a symmetric composition for update: $\bs{x}_{k+1} := \phi^{\frac{h}{2}} \circ \psi^h \circ \phi^{\frac{h}{2}} (\bs{x}_{k})$
where $\bs{x}_{k} = \begin{bmatrix} \bs{q}_{kh} \\ \bs{p}_{kh} \end{bmatrix}$, $\phi$ and $\psi$ correspond to solution flows of split SDEs, respectively given by
\begin{align*}
    &\phi: 
    \begin{cases}
        d\bs{q} = \bs{p} dt \\
        d\bs{p} = -\gamma \bs{p} dt + \sqrt{2\gamma} d\bs{B}
    \end{cases},
    \\
    &\psi:
    \begin{cases}
        d\bs{q} = -\alpha \nabla f(\bs{q}) dt + \sqrt{2\alpha} d\bs{W} \\
        d\bs{p} = -\nabla f(\bs{q}) dt
    \end{cases},
\end{align*}
and $\phi^t(\bs{x}_0)$ and  $\psi^t(\bs{x}_0)$ mean $x$'s value after evolving $\phi$ and $\psi$ for $t$ time with initial condition  $\bs{x}_0$.

Note that $\phi$ flow can be solved explicitly since the second equation is an Ornstein-Unlenbeck process and integrating the second equation followed by integrating the first one gives us an explicit solution
\begin{equation}\label{eq:phi}
    \begin{cases} 
        \bs{q}_t  = \bs{q}_0 + \frac{1 - e^{-\gamma t}}{\gamma} \bs{p}_0 + \sqrt{2\gamma} \int_0^t \frac{1 - e^{-\gamma(t - s)}}{\gamma} d\bs{B}(s),\\
        \bs{p}_t = e^{-\gamma t}\bs{p}_0 + \sqrt{2\gamma} \int_0^t e^{-\gamma (t - s)} d\bs{B}(s).  
    \end{cases} 
\end{equation}
For an implementation of the stochastic integral part in Equation \ref{eq:phi}, denoting $\bs{X} = \sqrt{2\gamma} \int_0^t \frac{1 - e^{-\gamma(t - s)}}{\gamma} d\bs{B}(s)$ and $\bs{Y} = \sqrt{2\gamma} \int_0^t e^{-\gamma (t - s)} d\bs{B}(s)$, and the covariance matrix of $(\bs{X}, \bs{Y})$ is 
$
    \mbox{Cov}(\bs{X}, \bs{Y}) = 
    {\small \begin{bmatrix}  \frac{\gamma h + 4e^{-\gamma \frac{h}{2} } - e^{-\gamma h} -3 }{\gamma^2} I_d & \frac{(1 - e^{-\gamma \frac{h}{2}})^2}{\gamma} I_d \\  \frac{(1 - e^{-\gamma \frac{h}{2}})^2}{\gamma} I_d & (1 - e^{-\gamma h}) I_d  \end{bmatrix}}.
$ As mean and covariance fully determine a Gaussian distribution, $\begin{bmatrix} \bs{X} \\ \bs{Y}\end{bmatrix} = M \bs{\xi}$ where $M$ is the Cholesky decomposition of $\mbox{Cov}(\bs{X}, \bs{Y})$, $\bs{\xi}$ is a $2d$ standard Gaussian random vector, i.i.d. at each step, and $\phi^t$ can thus be exactly simulated.

However, $\psi$ flow is generally not explicitly solvable unless $f$ is a quadratic function in $\bs{q}$. We simply choose to approximate $\psi^h(\bs{x}_0)$ with one-step Euler-Maruyama integration
$
    \psi^h(\bs{x}_0) \approx \widetilde{\psi}^h(\bs{x}_0) \text{ given by } \begin{cases}
    \bs{q}_h = \bs{q}_0 -\alpha \nabla f(\bs{q}_0) h + \sqrt{2\alpha h} \bs{\eta} \\
    \bs{p}_h = \bs{p}_0 -\nabla f(\bs{q}_0) h
    \end{cases}
$
where $\bs{\eta}$ is a standard $d$-dimensional Gaussian random vector, again i.i.d. each time $\tilde{\psi}$ is called.

Altogether, one step of an implementable Strang's splitting of HFHR is hence $ \phi^\frac{h}{2} \circ \widetilde{\psi}^h \circ \phi^\frac{h}{2}$
and we call this numerical scheme the HFHR algorithm, summarized in Alg.\ref{alg:HFHR}.

\begin{algorithm}[h]
\caption{A 1st-order HFHR Algorithm}\label{alg:HFHR}
\begin{algorithmic}[1]
\STATE \textbf{Input}: potential function $f$ and its gradient $\nabla f$, damping coefficients $\alpha$ and $\gamma$, step size $h$, initial condition $(\qbold_0, \pbold_0)$
\PROCEDURE{1st-order HFHR}{$f, \nabla f, \alpha, \gamma, h, \qbold_0, \pbold_0$}
    \STATE $k=0$ and initialize $\begin{bmatrix} \bs{q}_0 \\ \bs{p}_0 \end{bmatrix}$
    \WHILE{not converge}
        \STATE Generate independent standard Gaussian random vectors $\bs{\eta}_{k+1}\in \mathbb{R}^d, \bs{\xi}^1_{k+1}, \bs{\xi}^2_{k+1} \in \mathbb{R}^{2d}$
        \STATE Run $\phi^\frac{h}{2}$ :  $ \begin{bmatrix} \bs{q}_1 \\ \bs{p}_1 \end{bmatrix} = \begin{bmatrix} \bs{q}_{kh} + \frac{1 - e^{-\gamma \frac{h}{2}}}{\gamma} \bs{p}_{kh}  \\ e^{-\gamma \frac{h}{2}}\bs{p}_{kh} \end{bmatrix} + M \bs{\xi}^1_{k+1} $
        \STATE Run $\widetilde{\psi}^h$ : $ \begin{bmatrix} \bs{q}_2 \\ \bs{p}_2 \end{bmatrix} = \begin{bmatrix} \bs{q}_1 - \alpha \nabla f(\bs{q}_1) h + \sqrt{2\alpha h} \bs{\eta}_{k+1}  \\ \bs{p}_1 - \nabla f(\bs{q}_1) h \end{bmatrix} $
        \STATE Run $\phi^\frac{h}{2}$ :  $ \begin{bmatrix} \bs{q}_3 \\ \bs{p}_3 \end{bmatrix} = \begin{bmatrix} \bs{q}_2 + \frac{1 - e^{-\gamma \frac{h}{2}}}{\gamma} \bs{p}_2  \\ e^{-\gamma \frac{h}{2}}\bs{p}_2 \end{bmatrix} + M \bs{\xi}^2_{k+1} $
        \STATE $ \begin{bmatrix} \bs{q}_{(k+1)h} \\ \bs{p}_{(k+1)h} \end{bmatrix} \gets \begin{bmatrix} \bs{q}_3 \\ \bs{p}_3 \end{bmatrix} $
        \STATE $k \gets k + 1$
    \ENDWHILE
\ENDPROCEDURE
\end{algorithmic}
\end{algorithm}

As $\psi$ in Strang splitting is replaced by a 1st-order approximation $\tilde{\psi}$, the method is of order 1, however with good constant. This is rigorously established by the following theorem (interested readers are referred to Appendix \ref{sec_localError1}-\ref{sec_localError3} and \cite{li2021sqrt} for more technical details):


 
\begin{theorem}\label{thm:wasserstein}
Under Assumption \ref{asp:standard}, we further assume $\gamma - \frac{L+m}{\gamma} \ge m\alpha$ and $\nabla \Delta f$ satisfies a third-order growth condition, i.e., $\norm{\nabla \Delta f(\bs{q})} \le G\sqrt{1 + \norm{\bs{q}}^2 }, \forall \bs{q} \in \mathbb{R}^d$ for some $G > 0$. 
If $(\bs{q}_0, \bs{p}_0) \sim \pi_0$, then there exists $h_0, C > 0$ such that when $0 < h < h_0$, we have
\begin{equation}\label{eq:discretization_wasserstein}
     W_2(\mu_k, \mu) \le \kappa^\prime e^{-(\frac{m}{\gamma} + m\alpha) kh }W_2(\pi_0, \pi) + Ch
\end{equation}
where $\kappa^\prime$ is a constant depending only on $L, m, \gamma, \alpha$ (details in Appendix \ref{app:notation}), $\mu_k$ is the law of the $q$ marginal of the $k$-th iterate in Alg.\ref{alg:HFHR}, and $\mu$ is the $q$ marginal of the invariant distribution $\pi$. In particular, $C = \mathcal{O}(\sqrt{d})$ and there exists $b > 0$, independent of $\alpha$ and is of order $\mathcal{O}(\sqrt{d})$, s.t.
\begin{equation}
    C \le \frac{b}{m}(\alpha^2 - \frac{\alpha}{\gamma} + \frac{1}{\gamma^2}).
    \label{eq:CdependenceOnAlpha}
\end{equation}
\end{theorem}
\begin{remark}
The linear growth (at infinity) condition on $\nabla \Delta f$ is actually not as restrictive as it appears. For example, for monomial potentials, i.e.,  $f(x) = x^p, p\in \mathbb{Z}_+$, our linear growth condition is met when $p \le 4$, whereas a standard condition \citep[Theorem 3.1]{pavliotis2014stochastic} for the existence of SDE solutions holds only when $p \le 2$. In addition, our condition is related to the Hessian Lipschitz condition commonly used in the literature (e.g., \citealp{durmus2019high,ma2019there}). Smoothness and Hessian Lipschitzness imply the growth condition. Meanwhile, examples that satisfy linear growth condition but are not Hessian Lipschitz exist, e.g., $f(x)=x^4$, and thus linear growth condition is not necessarily stronger than Hessian Lipschitzness.
\end{remark}

Inspecting the role of $\alpha$ in Equation \eqref{eq:discretization_wasserstein}, we see it clearly increases the rate of exponential decay, but at the same time it can also increase the discretization error (see \eqref{eq:CdependenceOnAlpha}; assuming $h$ is fixed). However, as the following Cor.\ref{corollary:iteration complexity} and its remark will show, the net effect of having a positive $\alpha > 0$, at least for some $\alpha^\star$, is reduced iteration complexity.

\begin{corollary}\label{corollary:iteration complexity}
Consider the same assumption as in Thm. \ref{thm:wasserstein}. If $(\bs{q}_0, \bs{p}_0) \sim \pi_0$, then there exists $h_0, C > 0$ (same as that in Theorem \ref{thm:wasserstein}; recall $C =\mathcal{O}(\sqrt{d})$) such that for any target error tolerance $\epsilon > 0$, if we choose $h = h^\star \triangleq \min\{ h_0, \frac{\epsilon}{2C}\}$, then for $\epsilon < 2C h_0$,  after 
\begin{equation}\label{eq:iteration_complexity}
    k^\star =  2\frac{C}{\frac{m}{\gamma} + m\alpha} \frac{1}{\epsilon} \log \frac{2 \kappa^\prime W_2(\pi_0, \pi) }{\epsilon} = \tilde{\mathcal{O}}\left(\frac{\sqrt{d}}{\epsilon}\right).
\end{equation}
steps, we have $W_2(\mu_k, \mu) \le \epsilon$. 
\end{corollary}

\begin{remark}\label{remark:iteration_complexity}
Recall from Thm.\ref{thm:wasserstein} that $C \le \frac{b}{m}(\alpha^2 - \frac{\alpha}{\gamma} + \frac{1}{\gamma^2})$, so if we consider the minimizer $\alpha^\star$ of an upper bound of $\frac{C}{\frac{m}{\gamma} + m\alpha}$, 
$
    \alpha^\star = \argmin_{\alpha \ge 0} \frac{b}{m^2}\frac{\alpha^2 - \frac{\alpha}{\gamma} + \frac{1}{\gamma^2}}{ \frac{1}{\gamma} + \alpha} = \frac{\sqrt{3} - 1}{\gamma}
$. This suggests that by choosing an optimal $\alpha > 0$, one could effectively reduce iteration complexity. Note, however, that this $\alpha^\star$ may not be the true optimal one as bounds may not be tight. If they were, $k^\star_{\alpha^\star}=(2\sqrt{3}-3)k^\star_{\alpha=0}\approx 0.46k^\star_{\alpha=0}$; i.e., steps needed by ULD (discretized by Alg.\ref{alg:HFHR} with $\alpha=0$) can be halved by HFHR (discretized by Alg.\ref{alg:HFHR}).

\end{remark}

Rmk.\ref{remark:iteration_complexity} 
shows HFHR algorithm can lead to a similar bound on iteration complexity as ULD algorithm but with an improved constant, and thus having $\alpha \neq 0$ is advantageous. It also shows that the acceleration of HFHR carries through from continuous to discrete time. The same conclusion has been consistently observed in numerical experiments too.

\begin{remark}
Readers interested in more explicit condition number dependence are referred to Appendix \ref{sec_onAlpha}, where we show, for 2D Gaussian target with condition number $\kappa\gg 1$, the convergences of Euler discretizations of ULD under optimal parameters and HFHR under suboptimal parameters are, respectively, $(1-1/\kappa+o(1/\kappa))^n$ and $(1-2/\kappa+o(1/\kappa))^n$, where $n$ is the number of iterations. The latter (HFHR) is faster despite of suboptimal parameters. Also, like discussed in Sec.\ref{sec:theoryConti}, this result is also based on not comparing bounds but exact estimates, and thus trustworthy.
\label{rmk:discretizationGaussianConditionNumImproved}
\end{remark}

\vskip -0.1cm
\section{Numerical Experiments}\label{sec:experiment}
\vskip -0.1cm

We now empirically validate the acceleration enabled by $\alpha\neq 0$ by comparing HFHR algorithm and the popular KLMC discretization of ULD  \citep{dalalyan2018sampling}. For fairness, discretizations of the same order and number of gradient evaluations are compared. Appendix \ref{sec:RMA_HFHR} has an additional comparison based on RMA.

\subsection{A First Impression via Simple Target Distributions}\label{subsec:simple}


\begin{table}[!h]
\small
\caption{Test potentials. We use the shorthand notation $G^{d}_{m, \kappa}(\bs{x}) = \frac{m}{2}(\kappa x_d^2 + \sum_{i=1}^{d-1} x_i^2)$. `S', `C' and `N' mean strongly convex, convex, and non-convex, respectively.} \label{tab:test_function}
    \centering
    \begin{tabular}{c|c|c|c}
        \hline
        S & S & S & S \\
        $f_1=x^2/2$ & $f_2=G^2_{0.1,10}$ & $f_3=G^2_{10,10}$ & $f_4=G^{100}_{1,100}$
        \\
        \hline
    \end{tabular}
    \begin{tabular}{c|c}
        C & N (perturbed) \\
        $f_5=x^4/4$ & $f_6=(5x^2 + \sin(10x))/10$ \\
    \end{tabular}
    \begin{tabular}{c|c}
        \hline
        N (bimodal) & N (Rosenbrock) \\
        $f_7=5(x^4 - 2x^2)$ & $f_8=((x-1)^2 + 10(y-x^2)^2)/2$\\
        \hline
    \end{tabular}
\end{table}

\begin{figure}[!h]
\centering
    \begin{subfigure}{0.23\textwidth}
		\centering
		\includegraphics[width=\textwidth]{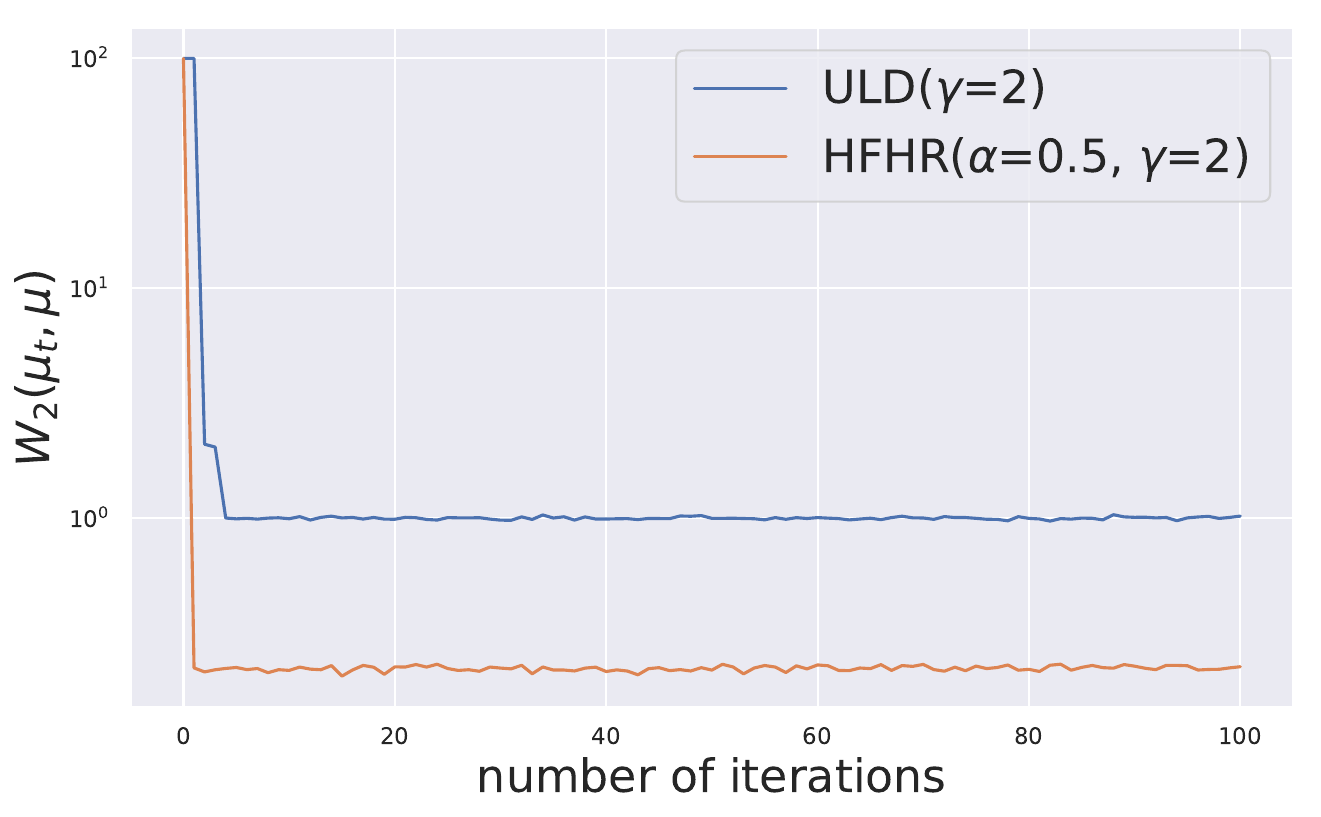}
		\caption{1D Gaussian} \label{fig:gaussian_1d}	
	\end{subfigure}
    \begin{subfigure}{0.23\textwidth}
		\centering
		\includegraphics[width=\textwidth]{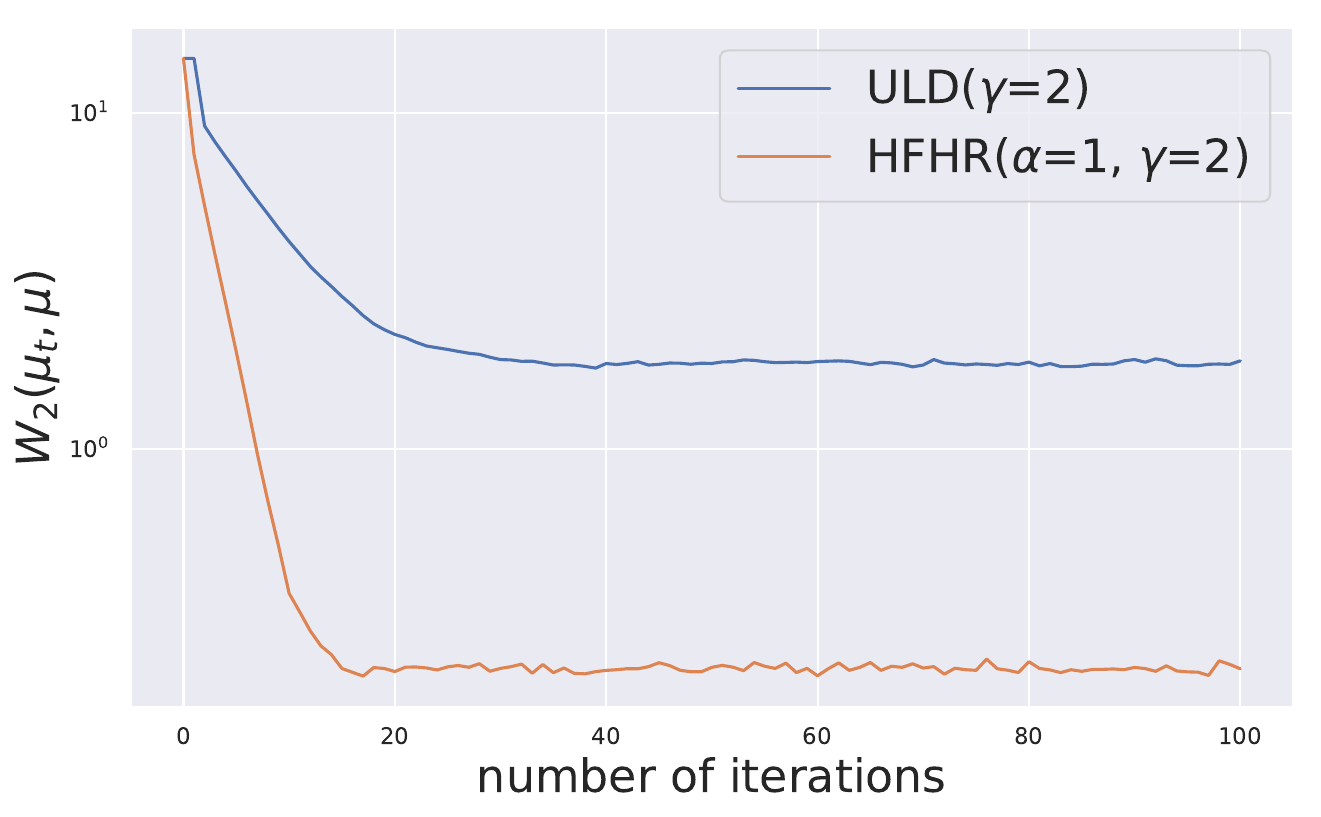}
		\caption{2D Gaussian($m=0.1$)} \label{fig:gaussian_2d_small_m}
	\end{subfigure}
    \begin{subfigure}{0.23\textwidth}
		\centering
		\includegraphics[width=\textwidth]{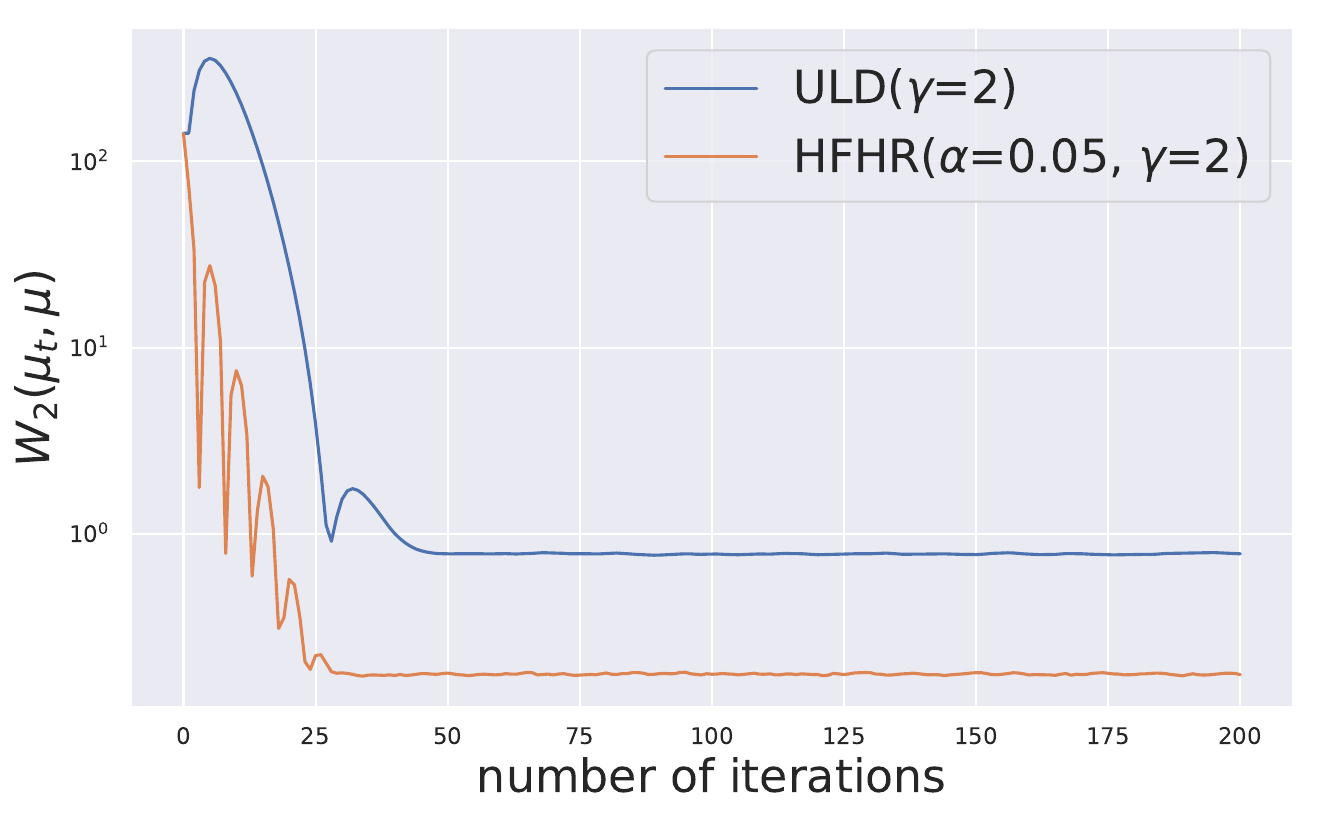}
		\caption{2D Gaussian($m=10$)} \label{fig:gaussian_2d_large_m}
	\end{subfigure}
    \begin{subfigure}{0.23\textwidth}
		\centering
		\includegraphics[width=\textwidth]{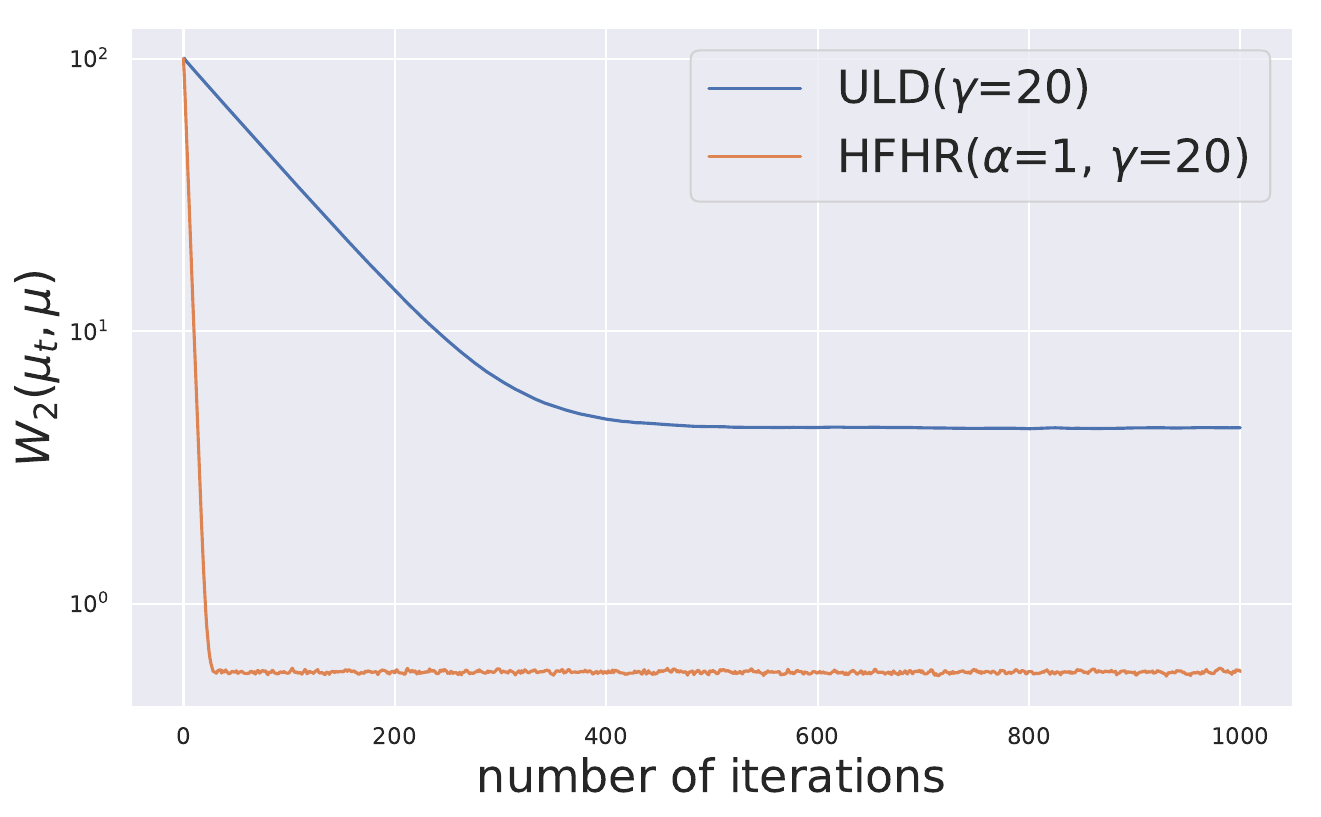}
		\caption{100D Gaussian} \label{fig:gaussian_100d}
	\end{subfigure}
	
    \begin{subfigure}{0.23\textwidth}
		\centering
		\includegraphics[width=\textwidth]{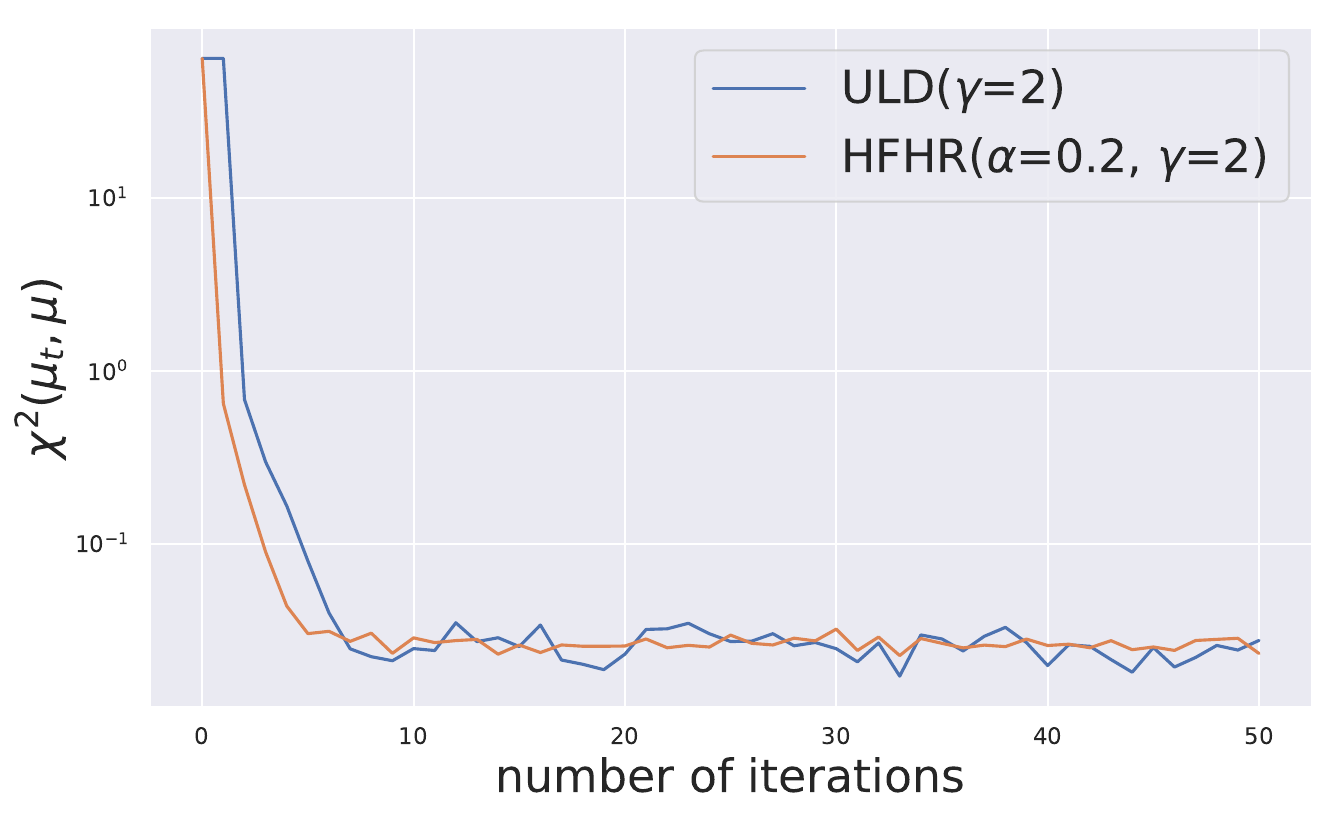}
		\caption{Convex} \label{fig:convex_1d}
	\end{subfigure}
    \begin{subfigure}{0.23\textwidth}
		\centering
		\includegraphics[width=\textwidth]{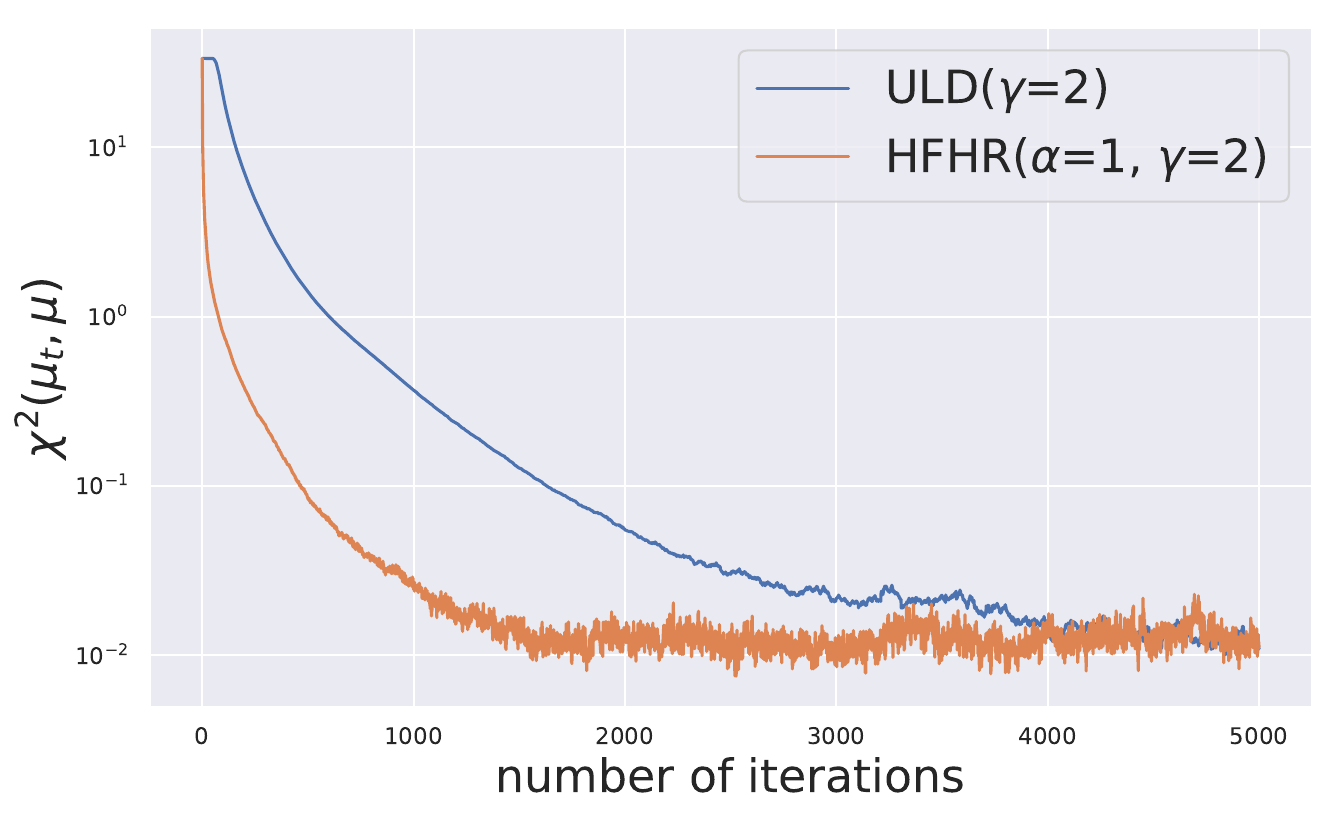}
		\caption{Perturbed} \label{fig:perturbed_1d}
	\end{subfigure}
    \begin{subfigure}{0.23\textwidth}
		\centering
		\includegraphics[width=\textwidth]{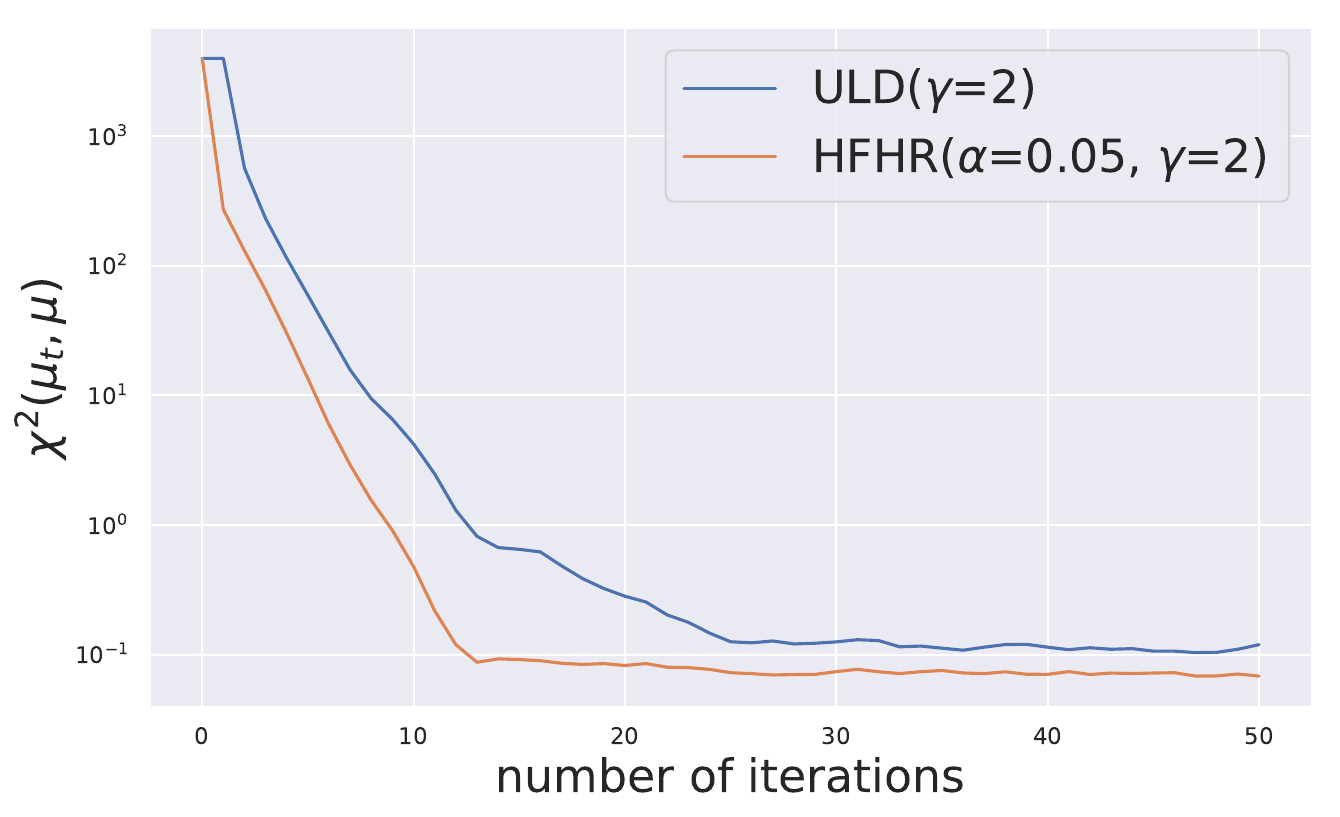}
		\caption{Bi-modal} \label{fig:bimodal_1d}
	\end{subfigure}
    \begin{subfigure}{0.23\textwidth}
		\centering
		\includegraphics[width=\textwidth]{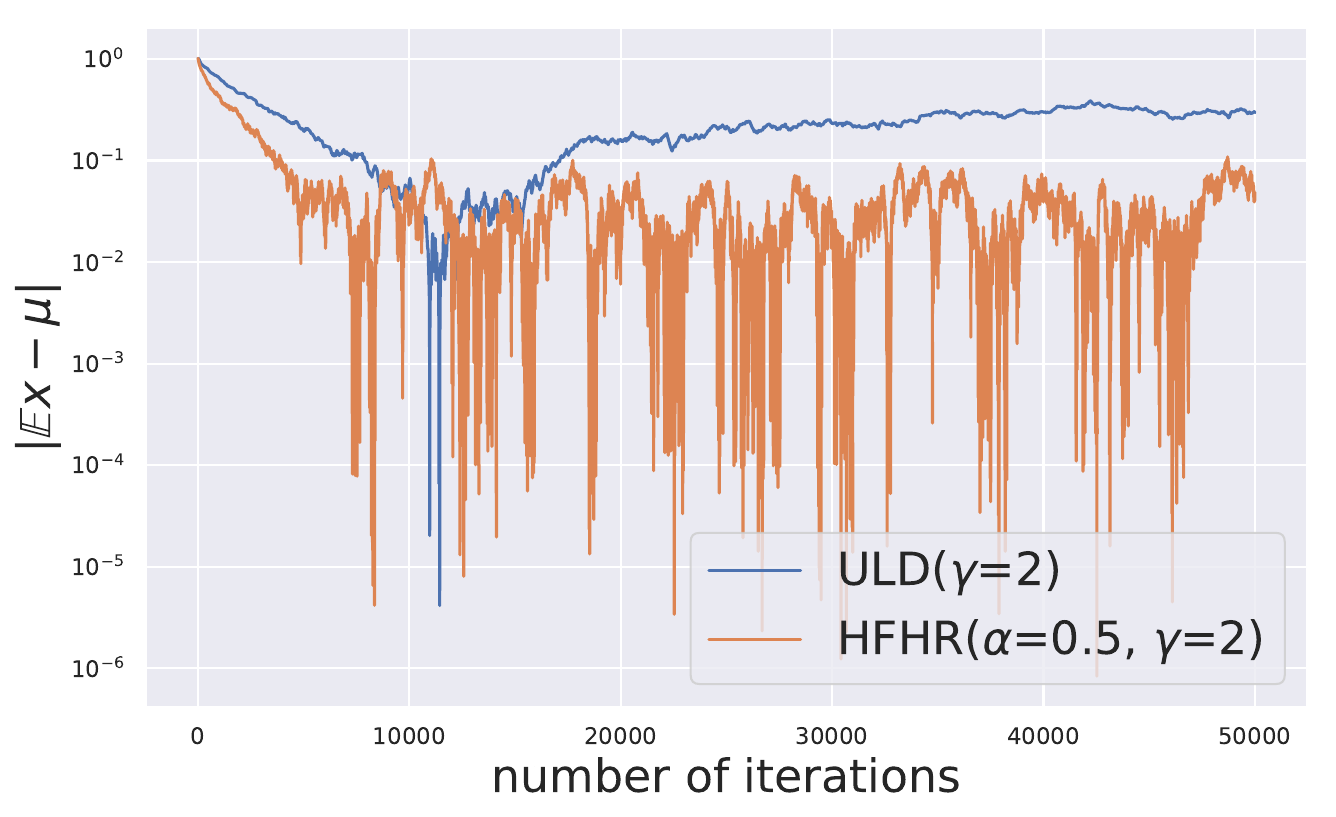}
		\caption{2D Rosenbrock} \label{fig:rosenbrock_2d}
	\end{subfigure}
	
    \caption{(a) $f_1\, (h=2)$. (b) $f_2\, (h=2)$. (c) $f_3\, (h=2.5)$. (d)$f_4\, (h=0.2)$. (e) $f_5\, (h=0.5)$. (f) $f_6\, (h=0.001)$. (g) $f_7\, (h=0.1)$. (h) $f_8\, (h=0.005)$.  $y$-axes are in log scale. 
    }\label{fig:simple}
\end{figure}

We first test 8 target distributions with simple, yet representative potential functions, summarized in Table \ref{tab:test_function}. For Gaussian targets, smoothness coefficient $L$ is available, hence we take $\gamma=2\sqrt{L}$ as suggested in \citet{dalalyan2018sampling}. To be consistent with Thm.\ref{thm:wasserstein}, closeness is measured in $W_2$ which has closed-form expression between Gaussians. For non-Gaussians, we empirically set $\gamma=2$ and measure sample quality by $\chi^2$ divergence with densities empiricially approximated by histograms. For the special case of $f_8(x,y)$, note approximating its density using a uniform-mesh-based histogram is either inaccurate or requiring the mesh to be very fine due to high nonconvexity, and we thus report the error in the $x$ component $|\mathbb{E}x - \mu|$ instead, where $\mu$ is the true mean of $x$-component.
Each algorithm uses 10,000 independent realizations for
empirical estimations.

Results are in Fig.\ref{fig:simple}. The improvement by HFHR correction can be clearly seen, although note that we did not optimize over $\alpha,\gamma$ values but simply chose the same $\gamma$ across ULD and HFHR and an arbitrary $\alpha$ additionally for HFHR. Step size $h$ however is tuned so that it is near the \textit{stability limit} of ULD algorithm, and then HFHR uses the same $h$. In the next section we'll optimize over all possible parameters so that ULD at its best performance can be compared with.

\subsection{A Nonlinear Case Study: Consistency with Theory}\label{subsec:verify_algorithm}

This section numerically verifies, more systematically, that $\alpha \neq 0$ (i.e. HFHR correction) accelerates the convergence, and optimal $\alpha$ exists (see Rmk.\ref{remark:iteration_complexity}), for which the acceleration is rather significant. In addition, how HFHR algorithm scales with the dimension is also of importance in a machine learning context, and thus the $\mathcal{O}(\sqrt{d})$ dependence given by Thm.\ref{thm:wasserstein} (in $C$, which is also inherited by Cor.\ref{corollary:iteration complexity} in the iteration complexity) will also be confirmed. 

For the purpose of checking dimension dependence, we will \emph{not} use Gaussian targets, because otherwise HFHR will decouple across different (orthogonal) dimensions, in which case an $\mathcal{O}(\sqrt{d})$ dependence is trivially true as a consequence of using $W_2$ for quantifying statistical accuracy.
Instead, we consider the potential in \citet{li2021sqrt} which is not additive across dimensions, namely
$
    f(\bs{x}) = \log \left( e^{x_1} + \cdots + e^{x_d}\right) + \frac{1}{2}\norm{\bs{x}}^2.
$
This is still a strongly convex function satisfying the assumption in Thm.\ref{thm:wasserstein}. The corresponding target is not Gaussian, we no longer have a closed form expression for $W_2$ distance, and it is computationally expensive to approximate this distance by samples. Therefore, we follow \citet{li2021sqrt} and use the error of mean instead as a surrogate because 
$
    \displaystyle
    \norm{\mathbb{E}_{\mu_k} \bs{q} - \mathbb{E}_{\mu} \bs{q} } \le W_2(\mu_k, \mu)
$
and hence the bound in Eq.\eqref{eq:discretization_wasserstein} also applies to the error in mean, and so does the iteration complexity bound in Eq.\eqref{eq:iteration_complexity}.

Fig.\ref{fig::verify_iteration_complexity} compares HFHR (Alg.\ref{alg:HFHR}) with ULD (KLMC) in terms of iteration complexity. To show that the acceleration of HFHR is not an artifact of time rescaling (which would disappear after discretization as the stability limit changes accordingly), we optimize over $h$ (by pushing both ULD and HFHR to their respective largest $h$ values that still allow monotonic convergence at a large scale), as well as $\gamma$ values, and compare the resulting best mixing times.

\begin{figure}[!h]
      	\centering
    	\includegraphics[width=0.3\textwidth]{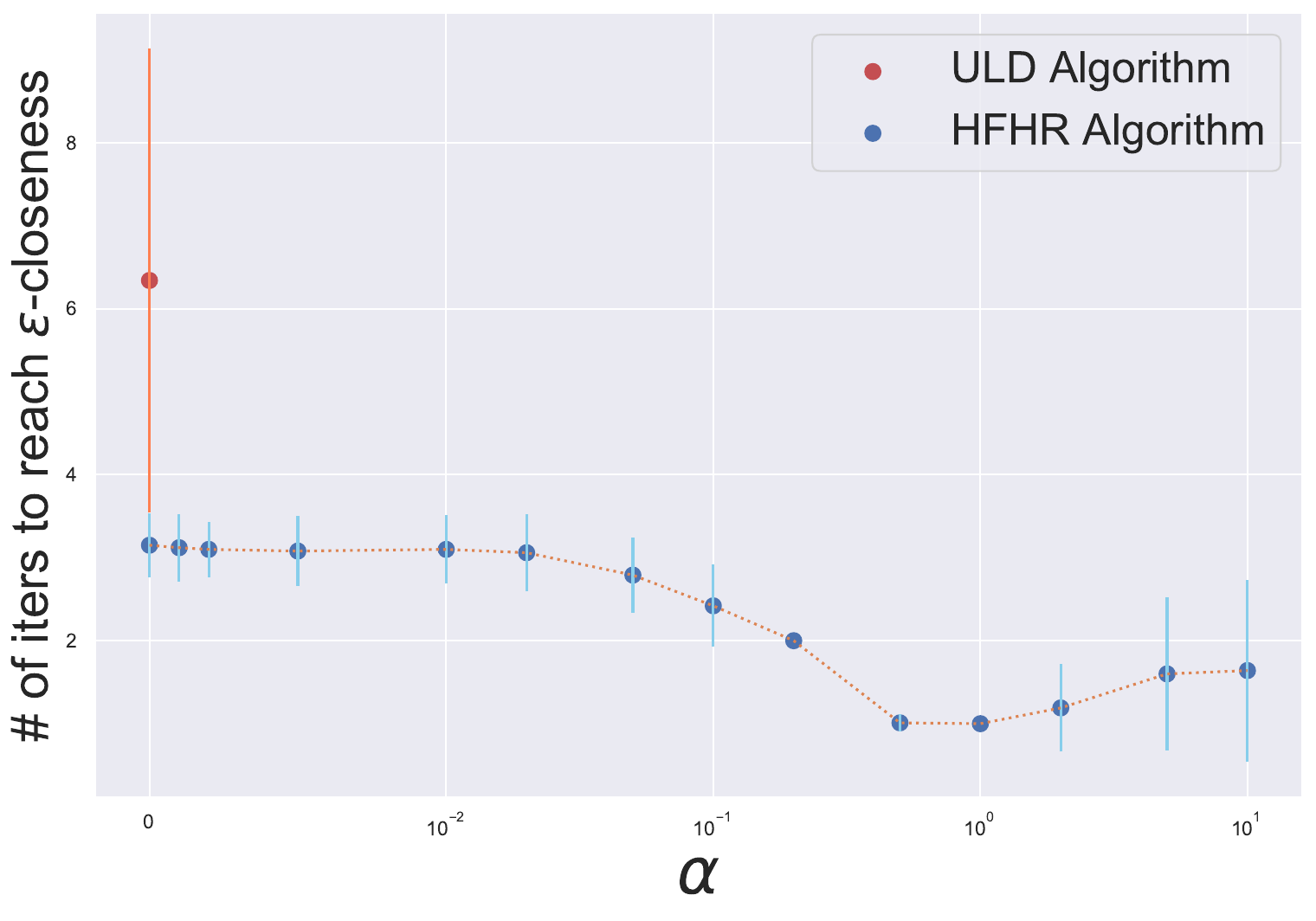}
    	\caption{Improvement of Algorithm \ref{alg:HFHR} over ULD algorithm in iteration complexity. (vertical bar = 1 standard deviation)} \label{fig::verify_iteration_complexity}
\end{figure}

More specifically, we choose the initial measure to be Dirac at $(100 \times \bs{1}_d, \bs{0}_d)$, where $\bs{1}_d, \bs{0}_d$ are $d$-dim. vectors filled with $1$ and $0$ respectively. $d=10$. We pick threshold $\epsilon=0.1$, and for each $\alpha \in \{$0, 0.001, 0.002, 0.005, 0.01, 0.02, 0.05, 0.1, 0.2, 0.5, 1, 2, 5, 10, 20, 50, 100$\}$, we try all combinations of $(\gamma, h) \in \{$0.1, 0.2, 0.5, 1, 2, 5, 10, 20, 50, 100$\} \times \left\{0.1 \times [50] \right\}$ for Algorithm \ref{alg:HFHR} (we also run ULD algorithm when $\alpha=0$), and empirically find the best combination that requires the fewest iterations to meet $\norm{\mathbb{E}_{\mu_k} \bs{q} - \mathbb{E}_{\mu} \bs{q} }  \le \epsilon$. We find that $h=5$ already surpasses the stability limit of ULD algorithm, hence the range of step size covers the largest step size that are practically usable for ULD algorithm. 100,000 independent realizations are used (evenly spread to 100 different randomization seeds).

When $\alpha > 0$, HFHR algorithm consistently outperforms ULD algorithm under optimized parameters (note it also does so when $\alpha=0$ because Alg.\ref{alg:HFHR} uses a efficiency-wise comparable but more accurate discretization than ULD algorithm). In particular, when $\alpha=0.5$ and $1$, 
which are empirically best values found for this experiment, HFHR achieves the specified $\epsilon$-closeness nearly 6$\times$ times faster than ULD, and its decreased mixing time (compared to $\alpha=0$ for the same algorithm) 
is consistent with the $\approx 0.46$ factor in Rmk.\ref{remark:iteration_complexity}).
These corroborate that the $\alpha\neq 0$ HFHR correction effect is genuine, and the resulting acceleration can be significant.

\begin{figure}[!h]
    		\centering
    		\includegraphics[width=0.3\textwidth]{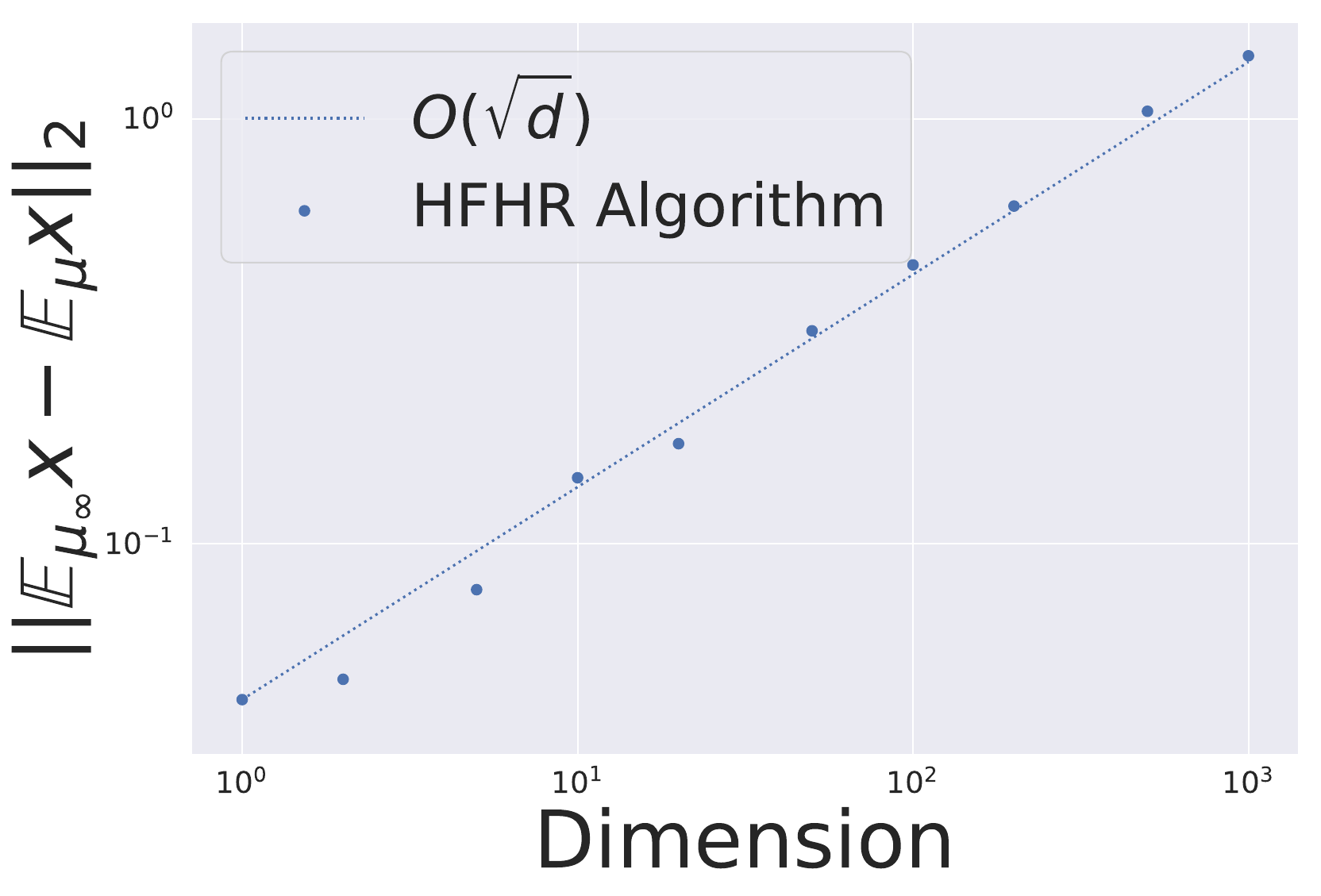}
    		\caption{$\sqrt{d}$ dependence of sampling error of Alg.\ref{alg:HFHR}} \label{fig::verify_corollary_14_d}
\end{figure}

Regarding dimension dependence, Thm.\ref{thm:wasserstein} states the HFHR sampling error is upper bounded by its discretization error, which is linear in $\sqrt{d}$. This is consistent with empirical observation in Fig.\ref{fig::verify_corollary_14_d}, where we experiment with $d \in \left\{1, 2, 5, 10, 20, 50, 100, 200, 500, 1000\right\}$. For each $d$, we fix $\gamma=2, \alpha=1, h=0.1$, choose a large enough $T=10$, run 1,000 independent realizations of HFHR algorithm, and estimate the sampling error using the surrogate.

\subsection{Bayesian Neural Network}\label{subsec:bnn}
To test the efficacy of HFHR on practical non-convex problems, we consider Bayesian neural network (BNN) which is a compelling learning model \citep{wilson2020case}; however, the focus won't be on its learning capability, and instead we just consider its training, which amounts to a real-life, high-dimensional, multi-modal example of sampling tasks. It no longer satisfies the conditions of our analysis, and our goal is to show HFHR still accelerates. We use fully-connected network with [22, 10, 2] neurons, ReLU, standard Gaussian prior for all parameters, and compare ULD and HFHR on UCI data set Parkinson
\citep{Dua:2019}.

Choices of hyper-parameter for Algorithm \ref{alg:HFHR} and ULD algorithm are systematically investigated. For each pair $(\gamma, \alpha) \in \{0.1, 0.5, 1,  5, 10,  50, 100\}^2$, we empirically tune the step size to the stability limit of ULD algorithm, simulate 1,000 independent realizations, and use the ensemble to conduct Bayesian posterior prediction. HFHR will then use the same step size. For each $\gamma$, we plot the negative log likelihood of HFHR algorithm (with different $\alpha$ choices) and ULD algorithm on training and test data in Figure \ref{fig:bnn_2}.  

Fig.\ref{fig:bnn_2} indicates that HFHR converges significantly faster than ULD in a wide range of setups. Obviously, the log-strongly-concave assumption required in Thm.\ref{thm:wasserstein} does not hold for multimodal target distributions. However, this numerical result shows that HFHR still accelerates ULD for highly complex models such as BNN, even when there is no obvious theoretical guarantee.
It showcases the applicability and effectiveness of HFHR as a general sampling algorithm.

\begin{figure}[!h]
\centering
    \begin{subfigure}{0.23\textwidth}
		\centering
		\includegraphics[width=\textwidth]{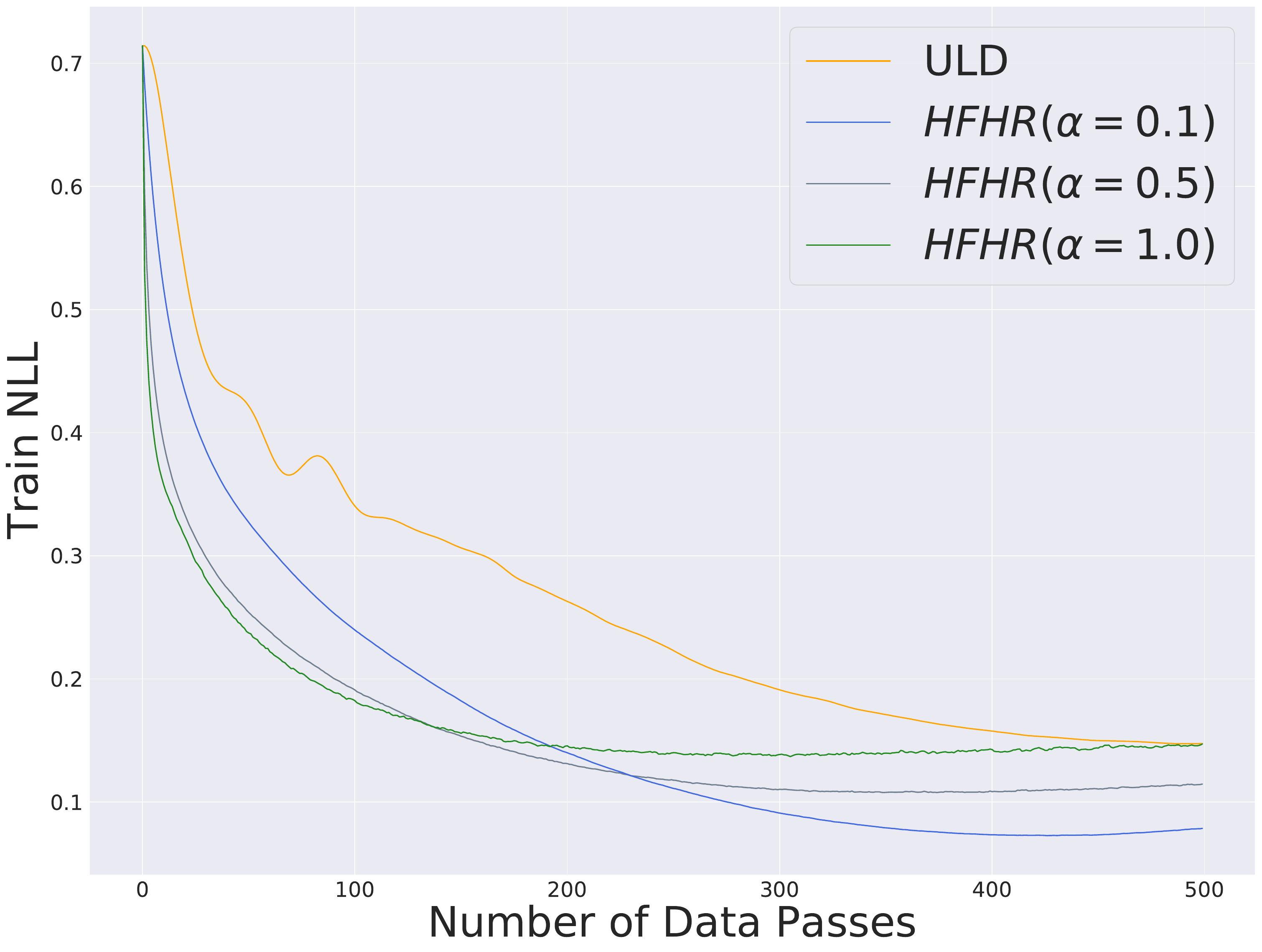}
		\caption{$\gamma=0.1 \, (h=0.005)$} 	
	\end{subfigure}
    \begin{subfigure}{0.23\textwidth}
		\centering
		\includegraphics[width=\textwidth]{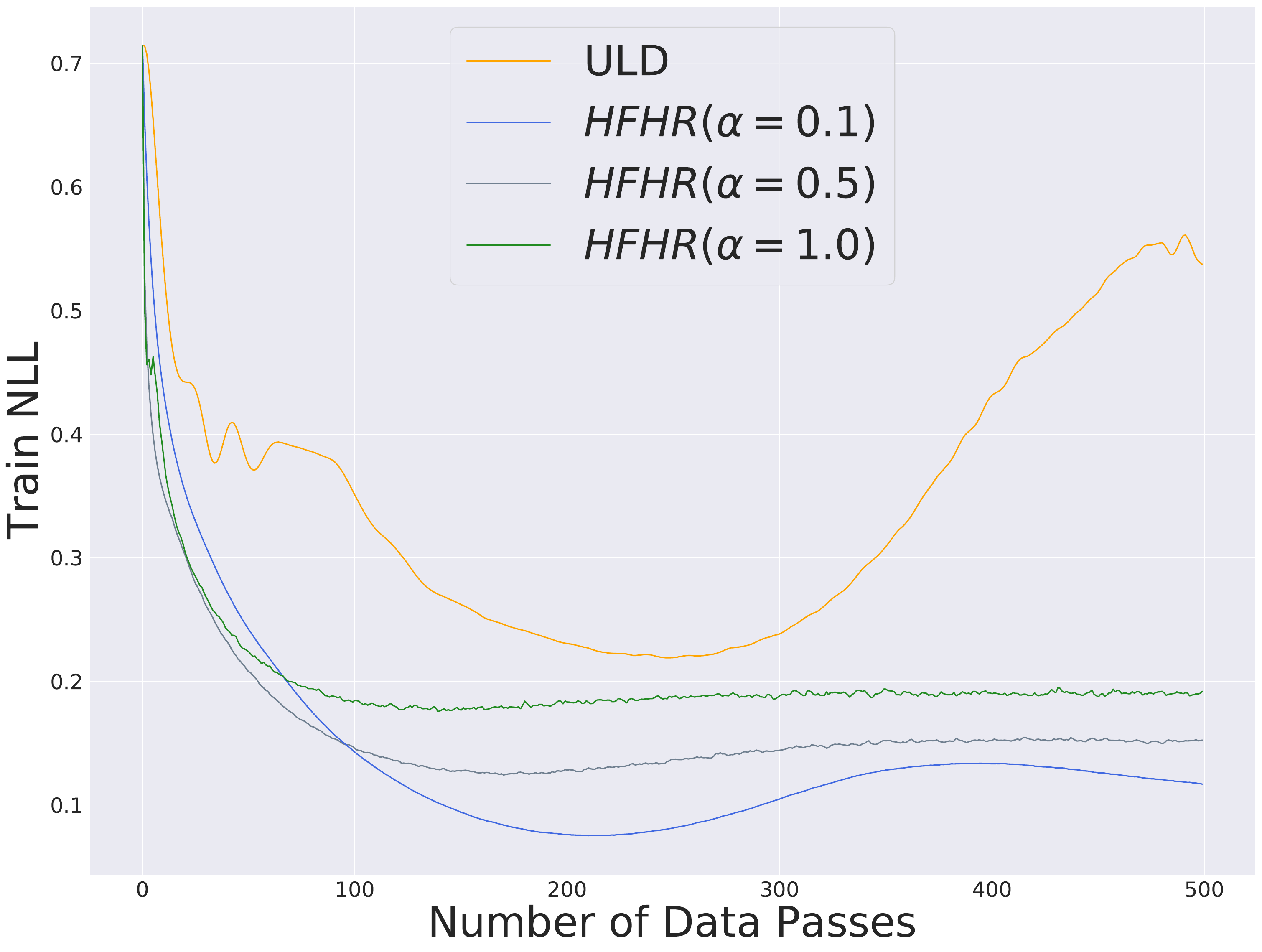}
		\caption{$\gamma=0.1 \, (h=0.01)$} \label{fig:r1c2}	
	\end{subfigure}
    \begin{subfigure}{0.23\textwidth}
		\centering
		\includegraphics[width=\textwidth]{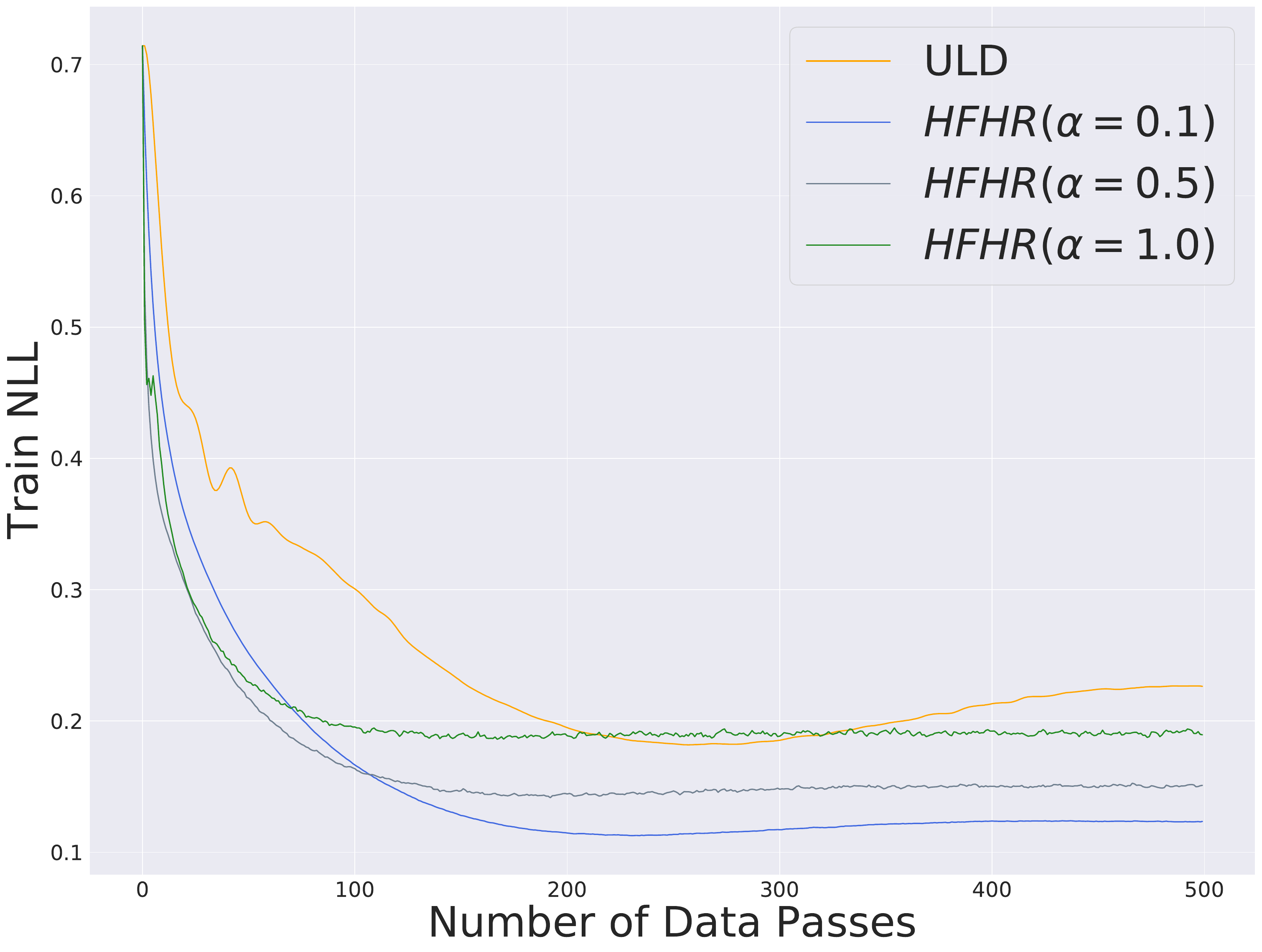}
		\caption{$\gamma=1 \,  (h=0.01)$} 
	\end{subfigure}
    \begin{subfigure}{0.23\textwidth}
		\centering
		\includegraphics[width=\textwidth]{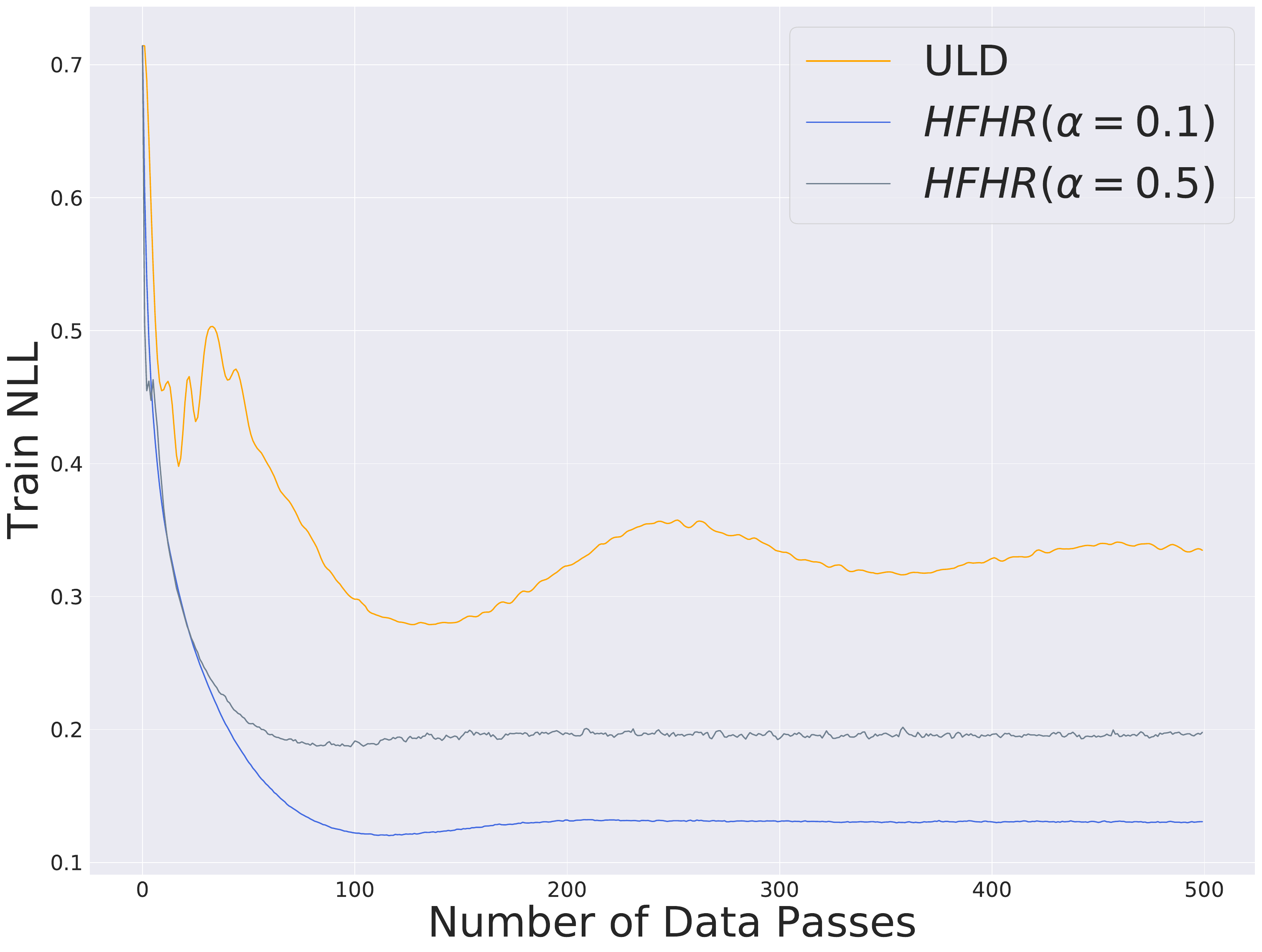}
		\caption{$\gamma=1 \, (h=0.02)$} \label{fig:r2c2}	
	\end{subfigure}
	\begin{subfigure}{0.23\textwidth}
		\centering
		\includegraphics[width=\textwidth]{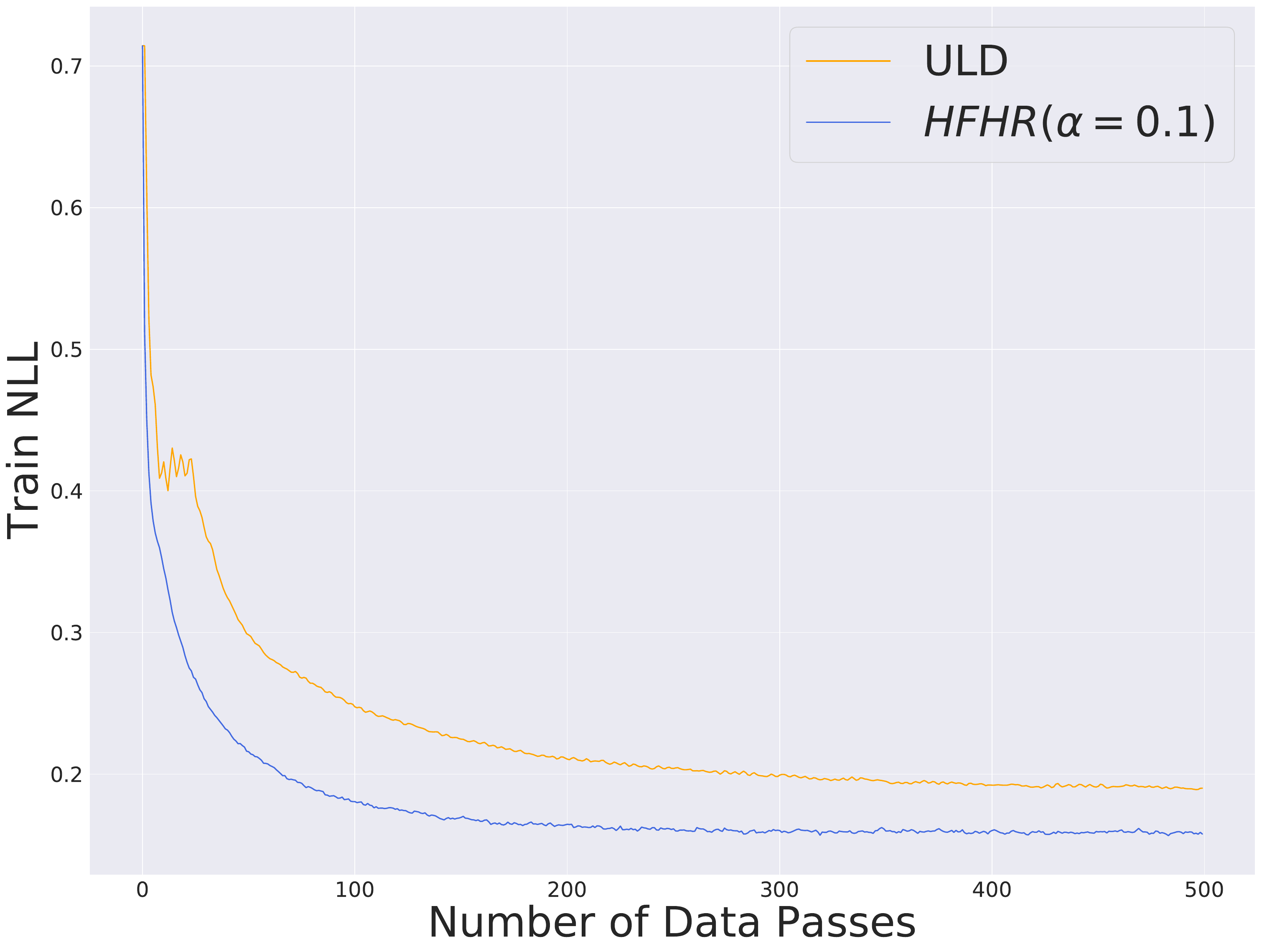}
		\caption{$\gamma=10 \, (h=0.05)$}
	\end{subfigure}
    \begin{subfigure}{0.23\textwidth}
		\centering
		\includegraphics[width=\textwidth]{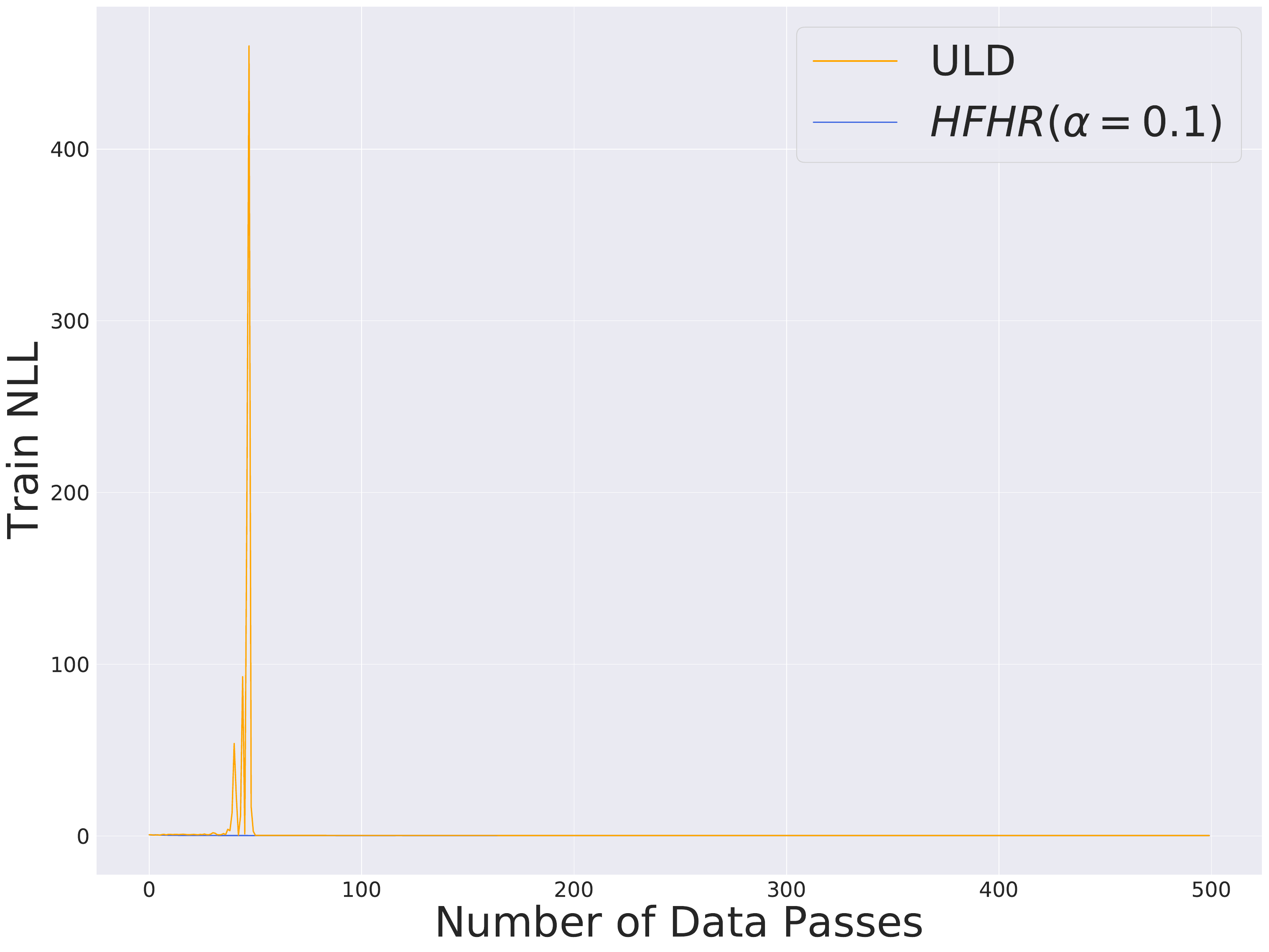}
		\caption{$\gamma=10 \, (h=0.1)$} 
	\end{subfigure}
    \begin{subfigure}{0.23\textwidth}
		\centering
		\includegraphics[width=\textwidth]{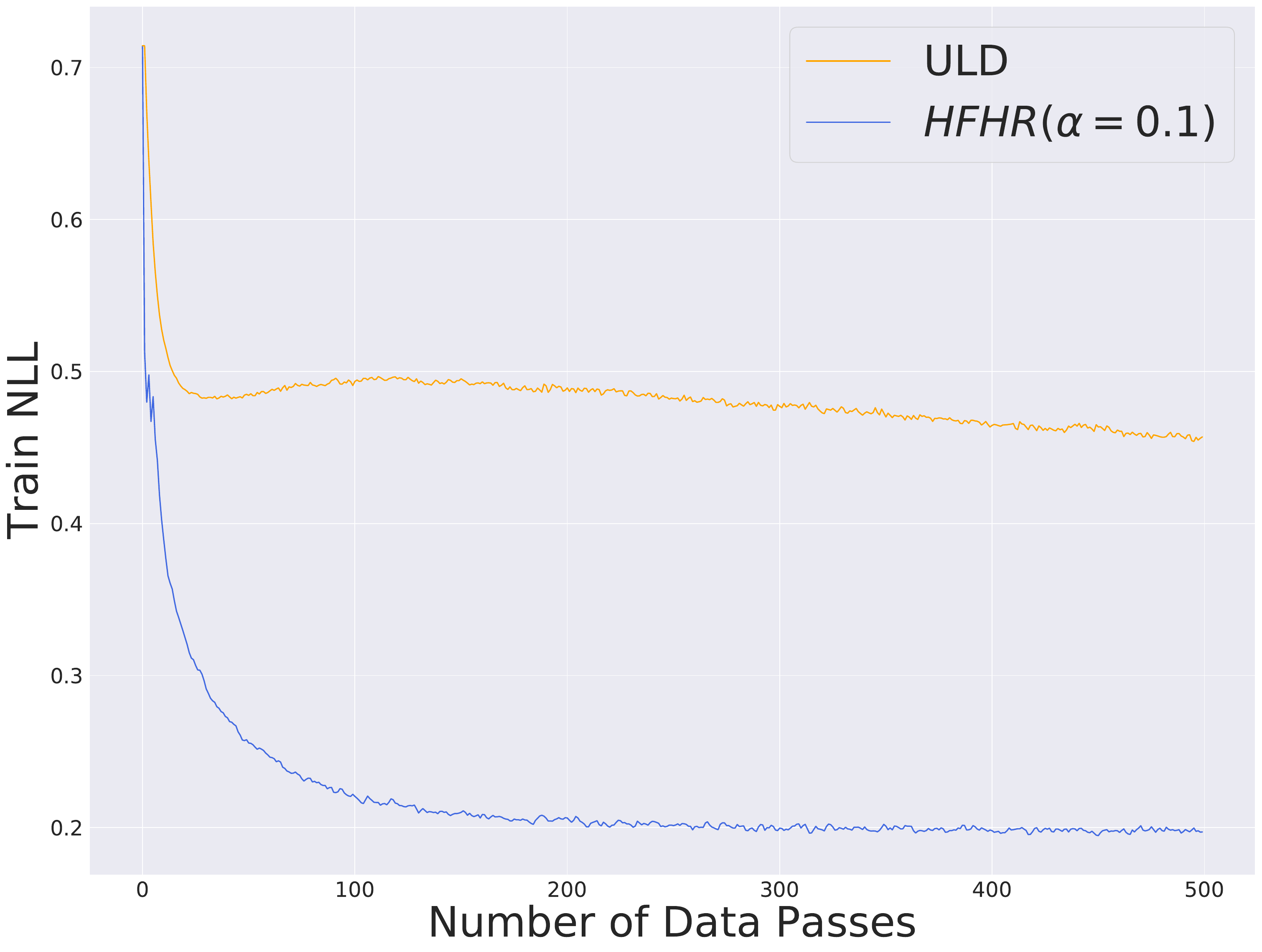}
		\caption{$\gamma=100 \, (h=0.1)$} 
	\end{subfigure}
    \begin{subfigure}{0.23\textwidth}
		\centering
		\includegraphics[width=\textwidth]{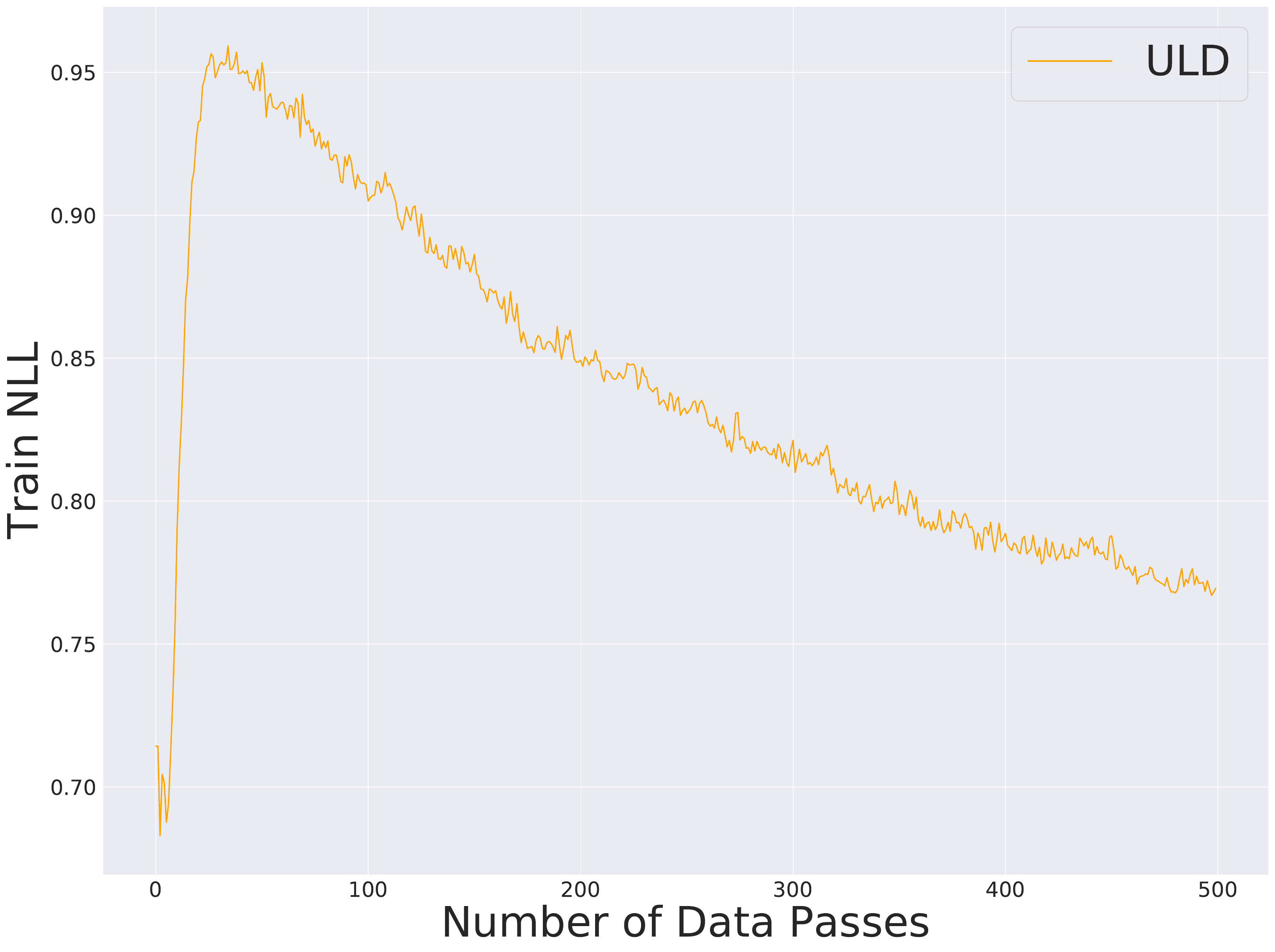}
		\caption{$\gamma=100 \, (h=0.2)$} 	
	\end{subfigure}
    \caption{Training Negative Log-Likelihood (NLL) for various $\gamma$. Left column uses step sizes that are close to the stability limit of ULD algorithm, as further increased step size in right column no longer gives stability/monotonicity. Unstable cases where $\alpha$ is too large are not drawn (recall $\alpha,\gamma,h$ constrain each other; see e.g., Rmk.\ref{remark:iteration_complexity} for intuitions in convex setups).}\label{fig:bnn_2}
\end{figure}

\vskip -0.1cm
\section{Conclusion and Discussion}
\vskip -0.1cm
This paper proposes HFHR dynamics, a NAG-optimizer-based diffusion process. Its discretizations give a family of accelerated sampling algorithms. To demonstrate the acceleration enabled by HFHR, the geometric ergodicity of HFHR (both the continuous and discretized versions) is quantified, and its convergence is provably faster than Underdamped Langevin Dynamics, which by itself is often already considered as an accelerated version of Overdamped Langevin Dynamics. Since HFHR adopts a new perspective, which is to turn the \textbf{finite} learning rate advantage of NAG-SC optimizer into a sampling counterpart, there are a number of directions in which this work can be extended: (i) HFHR dynamics can be discretized in different ways resulting in different algorithms. Two popular discretizations are considered here and one theoretically analyzed, but other discretizations could also be used and possibly lead to favorable performances. (ii) To scale HFHR up to large data sets, full gradient may be replaced by stochastic gradient (SG) --- how to quantify, and hence optimize the performance of SG-HFHR? (iii) Can the generalization ability of HFHR-trained learning models (e.g., BNN) be quantified, and how does it compare with that by LMC, KLMC, or other dynamics-based samplers? These will be future work. 


\section*{Acknowledgements}
The authors sincerely thank Michael Tretyakov, Yian Ma, Wenlong Mou, and Lingjiong Zhu for helpful discussions. MT was partially supported by NSF grants DMS-1847802 and ECCS-1936776. This work was initiated when HZ was a professor at Georgia Tech.

\bibliography{reference}
\bibliographystyle{icml2022}

\newpage
\appendix
\onecolumn
\section{Additional Notations}\label{app:notation}
We introduce a few notations that are used in the main text as well as some proof. When $\nabla f$ is $L$-Lipschitz, the drift term $\begin{bmatrix} \bs{p} - \alpha \nabla f(\bs{q}) \\ -\gamma \bs{p} - \nabla f(\bs{q})\end{bmatrix}$ in HFHR dynamics is also $L^\prime$-Lipschitz, as proved in Lemma \ref{lemma:lipschitz}, where
\[
    L^\prime = \sqrt{2} \max\left\{ \sqrt{1 + \alpha^2} \max\left\{\frac{1}{\sqrt{2}}, L \right\} , \sqrt{1 + \gamma^2}\right\}.
\]

We show in Lemma \ref{lemma:contraction} that a linear-transformed HFHR dynamics satisfies the nice contraction property, the linear transformation $P$ we use is defined as
\[
    P = \begin{bmatrix} \gamma I & I \\ 0 & \sqrt{1+\alpha\gamma}I\end{bmatrix} \in \mathbb{R}^{2d \times 2d}.
\]
Denote the largest and the smallest singular value of $P$ by
\begin{align*}
    \sigma_{\text{max}} =& \sqrt{\frac{\alpha \gamma}{2}+\frac{\gamma^{2}}{2} + \frac{\sqrt{\alpha^{2} \gamma^{2}-2 \alpha \gamma^{3}+4 \alpha \gamma+\gamma^{4}+4}}{2}+1}, \\
    \sigma_{\text{min}} =&s \sqrt{\frac{\alpha \gamma}{2}+\frac{\gamma^{2}}{2}-\frac{\sqrt{\alpha^{2} \gamma^{2}-2 \alpha \gamma^{3}+4 \alpha \gamma+\gamma^{4}+4}}{2}+1}
\end{align*}
and its condition number by 
\[
    \kappa^\prime = \frac{\sigma_{\text{max}}}{\sigma_{\text{min}}} = \sqrt{\frac{\frac{\alpha \gamma}{2}+\frac{\gamma^{2}}{2} + \frac{\sqrt{\alpha^{2} \gamma^{2}-2 \alpha \gamma^{3}+4 \alpha \gamma+\gamma^{4}+4}}{2}+1}{\frac{\alpha \gamma}{2}+\frac{\gamma^{2}}{2}-\frac{\sqrt{\alpha^{2} \gamma^{2}-2 \alpha \gamma^{3}+4 \alpha \gamma+\gamma^{4}+4}}{2}+1}}.
\]
The rate $\lambda^\prime$ of exponential convergence of transformed HFHR dynamics is characterized in Lemma \ref{lemma:contraction} and is defined as
\[
    \lambda^\prime = \min\left\{ \frac{m}{\gamma} + \alpha m, \frac{\gamma^2 - L}{\gamma} \right\}
\]
given that $\gamma^2 > L$.

\section{Proofs for the Continuous Dynamics}
\emph{Notations and definitions can be found in Sec.\ref{sec:notation}}.
\subsection{Proof of Theorem \ref{thm:invariant_distribution}}
\label{sec:proofOfInvariantDistribution}
\begin{proof}
The Fokker-Plank equation of HFHR is given by
\[
    \partial_t \rho_t = - \nablax \cdot \left(\m{\pbold \\ -\nabla f(\qbold)} \rho_t \right) 
    + \alpha\left(\nablaq \cdot ( \nabla f(\qbold) \rho_t ) + \Delta_{\qbold} \rho_t\right)
    + \gamma\left(\nablap \cdot ( \pbold \rho_t ) + \Delta_{\pbold} \rho_t\right)
\]
where $\nablax = (\nablaq, \nablap)$. For $\pi \propto e^{-f(\qbold) - \frac{1}{2} \|\pbold\|^2}$, we have 
\begin{align*}
    & \nablax \cdot \left(\m{\pbold \\ -\nabla f(\qbold)} \pi \right) = \innerprod{\m{\pbold \\ -\nabla f(\qbold)}}{\nablax \pi} = 0, \\
    & \Delta_{\qbold} \pi = -\nablaq \cdot(\pi \nabla f(\qbold)) \\
    & \Delta_{\pbold} \pi = -\nablap \cdot(\pi \pbold)
\end{align*}
Therefore $\partial_t \pi = 0$ and hence $\pi$ is the invariant distribution of HFHR.
\end{proof}

\subsection{Proof of Theorem \ref{thm:exp_convergence_coupling}}
\begin{proof}
Consider two copies of HFHR that are driven by the same Brownian motion
\[
    \begin{cases}
        d\qbold_t = (\pbold_t - \alpha \nabla f(\qbold_t))dt + \sqrt{2\alpha} d\bs{B}^1_t \\
        d\pbold_t = (-\gamma \pbold_t - \nabla f(\qbold_t))dt + \sqrt{2\gamma} d\bs{B}^2_t
    \end{cases},
    \quad
    \begin{cases}
        d\tilde{\qbold}_t = (\tpbold_t - \alpha \nabla f(\tilde{\qbold}_t))dt + \sqrt{2\alpha} d\bs{B}^1_t \\
        d\tilde{\pbold_t} = (-\gamma \tpbold_t - \nabla f(\tilde{\qbold}_t))dt + \sqrt{2\gamma} d\bs{B}^2_t
    \end{cases},
\]
where we set $(\tilde{\qbold}_0, \tpbold_0) \sim \pi$,  $\pbold_0 = \tpbold_0$ and $\qbold_0$ such that
\[
    W_2^2(\mu_0, \mu) = \mathbb{E} \left[ \| \qbold_0 - \tilde{\qbold}_0 \|^2_2 \right], \quad \qbold_0 \sim \mu_0
\]
Denote $\phipsi = P\deltaqp$ where $P$ is defined in Appendix \ref{app:notation}. By Lemma \ref{lemma:contraction} and the assumption on $\alpha, \gamma$, we have
\[
    \norm{\phipsi}^2 \le e^{-2(\frac{m}{\gamma} + m \alpha) t} \norm{\begin{bmatrix} \bs{\phi}_0 \\ \bs{\psi}_0\end{bmatrix}}^2.
\]
Therefore we obtain
\begin{align*}
    W_2^2(\mu_t, \mu) =& \inf_{(\bs{q}_t, \tilde{\bs{q}}_t) \sim \Pi(\mu_t, \mu)} \mathbb{E}\norm{\bs{q}_t - \tilde{\bs{q}}_t}^2 \\
    \le& \inf_{(\bs{q}_t, \tilde{\bs{q}}_t) \sim \Pi(\mu_t, \mu), (\bs{p}_t, \tilde{\bs{p}}_t) \sim \Pi(\nu_t, \nu)} \mathbb{E} \norm{ \begin{bmatrix} \bs{q}_t - \tilde{\bs{q}}_t \\ \bs{p}_t - \tilde{\bs{p}}_t\end{bmatrix} }^2 \\
    \le& \mathbb{E}\|P^{-1}\|_2^2 \norm{\phipsi}^2 \\
    \le& \mathbb{E} \|P^{-1}\|_2^2 e^{-2(\frac{m}{\gamma} + m\alpha)t} \norm{\begin{bmatrix} \bs{\phi}_0 \\ \bs{\psi}_0\end{bmatrix}}^2 \\
    \le& (\kappa^\prime)^2  e^{-2(\frac{m}{\gamma} + m\alpha)t} \norm{\begin{bmatrix} \bs{q}_0 - \tilde{\bs{q}}_0  \\ \bs{p}_0 - \tilde{\bs{p}}_0\end{bmatrix}}^2 \\
    =& (\kappa^\prime)^2  e^{-2(\frac{m}{\gamma} + m\alpha)t} W_2^2(\mu_0, \mu)
\end{align*}
Taking square root yields the desired result.

\end{proof}

\section{Arbitrary Long Time Discretization Error of Algorithm \ref{alg:HFHR} }
\begin{theorem}\label{thm:discretization}
Under Conditions A\ref{asp:standard} and further assume the function $\nabla \Delta f$ grows at most linearly, i.e., $\norm{\nabla \Delta f(\bs{q})} \le G\sqrt{1 + \norm{\bs{q}}^2 }, \forall \bs{q} \in \mathbb{R}^d$. Also suppose $\gamma$ in HFHR dynamics satisfy $\gamma^2 > L$. Then there exist $C, h_0 > 0$, such that for  $0 < h \le h_0$, we have
\[
    \left(\mathbb{E}\norm{\bs{x}_k - \bar{\bs{x}}_k}^2\right)^\frac{1}{2}  \le C h 
\]
where $\bar{\bs{x}}_k$ is the $k$-th iterate of Algorithm \ref{alg:HFHR} with step size $h$ starting from $\bs{x}_0$, $\bs{x}_k$ is the solution of HFHR dynamics at time $kh$, starting from $\bs{x}_0$. This result holds uniformly for all $k \ge 0$ and $k$ can go to $\infty$. In particular, $C = \mathcal{O}(\sqrt{d})$ and if $\gamma - \frac{L+m}{\gamma}\ge m\alpha$, then there exists $b > 0$, independent of $\alpha$ and is of order $\mathcal{O}(\sqrt{d})$, such that  
\begin{equation}
    C \le \frac{b}{m}(\alpha^2 - \frac{\alpha}{\gamma} + \frac{1}{\gamma^2}).
\end{equation}
\end{theorem}
\begin{proof}
Denote $t_k = kh$, the solution of the HFHR dynamics at time $t$ by $\bs{x}_{0, \bs{x}_0}(t)$, the $k$-th iterates of the Strang's splitting method of HFHR dynamics by $\bar{\bs{x}}_{0, \bs{x}_0}(kh)$. Both $\bs{x}_{0, \bs{x}_0}(t)$ and $\bar{\bs{x}}_{0, \bs{x}_0}(kh)$ start from the same initial value $\bs{x}_0$. The linear transformation $P$ defined in Appendix \ref{app:notation}, transforms the solution of HFHR dynamics into $\bs{y}_{0, P\bs{x}_0}(t) = P \bs{x}_{0, \bs{x}_0}(t)$ and the Strang's splitting discretization of HFHR into $\bar{\bs{y}}_{0, P \bs{x}_0}(t) = P\bar{\bs{x}}_{0, \bs{x}_0}(t)$.

For the ease of notation, we write $\bs{y}_{0, \bs{y}_0}(t_k)$ as $\bs{y}_k$ and $\bar{\bs{y}}_{0, \bs{y}_0}(t_k)$ as $\bar{\bs{y}}_k$. We have the following identity
\begin{align*}
    \mathbb{E} \norm{\bs{y}_{k+1} - \bar{\bs{y}}_{k+1}}^2
    =& \mathbb{E} \norm{\bs{y}_{t_k, \bs{y}_k}(h) - \bar{\bs{y}}_{t_k, \bar{\bs{y}}_k}(h) }^2\\
    =& \mathbb{E} \norm{\bs{y}_{t_k, \bs{y}_k}(h) - \bs{y}_{t_k, \bar{\bs{y}}_k}(h) + \bs{y}_{t_k, \bar{\bs{y}}_k}(h) - \bar{\bs{y}}_{t_k, \bar{\bs{y}}_k}(h) }^2 \\
    =& \underbrace{\mathbb{E}\norm{\bs{y}_{t_k, \bs{y}_k}(h) - \bs{y}_{t_k, \bar{\bs{y}}_k}(h)}^2}_{\circled{1}} + \underbrace{\mathbb{E}\norm{\bs{y}_{t_k, \bar{\bs{y}}_k}(h) - \bar{\bs{y}}_{t_k, \bar{\bs{y}}_k}(h)}^2}_{\circled{2}} \\
    +& 2 \underbrace{\mathbb{E} \left\langle \bs{y}_{t_k, \bs{y}_k}(h) - \bs{y}_{t_k, \bar{\bs{y}}_k}(h) , \bs{y}_{t_k, \bar{\bs{y}}_k}(h) - \bar{\bs{y}}_{t_k, \bar{\bs{y}}_k}(h)\right\rangle}_{\circled{3}}
\end{align*}
By Lemma \ref{lemma:contraction}, when $0 < h < \frac{1}{2\lambda^\prime}$, term \circled{1} can be upper bounded as
\begin{align*}
    \mathbb{E}\norm{\bs{y}_{t_k, \bs{y}_k}(h) - \bs{y}_{t_k, \bar{\bs{y}}_k}(h)}^2 \le& e^{-2\lambda^\prime h} \mathbb{E}\norm{\bs{y}_k - \bar{\bs{y}}_k}^2 \\
    \le& \left(1 -2\lambda^\prime h + 2(\lambda^\prime)^2h^2 \right)\mathbb{E}\norm{\bs{y}_k - \bar{\bs{y}}_k}^2 \\
    \le& \left(1 - \lambda^\prime h  \right)\mathbb{E}\norm{\bs{y}_k - \bar{\bs{y}}_k}^2 
\end{align*}
where the second inequality is due to $e^{-x} \le 1 - x + \frac{x^2}{2}, \forall x > 0$.

For term \circled{2}, we have by Lemma \ref{lemma:local} that
\begin{align*}
    \mathbb{E}\norm{\bs{y}_{t_k, \bar{\bs{y}}_k}(h) - \bar{\bs{y}}_{t_k, \bar{\bs{y}}_k}(h)}^2 \le \sigma_{\text{max}}^2 \, \mathbb{E}\norm{\bs{x}_{t_k, \bar{\bs{x}}_k}(h) - \bar{\bs{x}}_{t_k, \bar{\bs{x}}_k}(h)}^2 \le \sigma_{\text{max}}^2 \, C_2^2 h^3
\end{align*}
where $\sigma_{\text{max}}$ is the largest singular value of matrix $P$.  

For term \circled{3}, we have by Lemma \ref{lemma:Z} that 
\begin{align*}
    &2\mathbb{E} \left\langle \bs{y}_{t_k, \bs{y}_k}(h) - \bs{y}_{t_k, \bar{\bs{y}}_k}(h) , \bs{y}_{t_k, \bar{\bs{y}}_k}(h) - \bar{\bs{y}}_{t_k, \bar{\bs{y}}_k}(h)\right\rangle \\
    =& 2\mathbb{E} \left\langle \bs{y}_k - \bar{\bs{y}}_k + \bs{z} , \bs{y}_{t_k, \bar{\bs{y}}_k}(h) - \bar{\bs{y}}_{t_k, \bar{\bs{y}}_k}(h)\right\rangle \\
    =& \underbrace{2\mathbb{E} \left\langle \bs{y}_k - \bar{\bs{y}}_k , \bs{y}_{t_k, \bar{\bs{y}}_k}(h) - \bar{\bs{y}}_{t_k, \bar{\bs{y}}_k}(h)\right\rangle}_{\circled{3a}} + \underbrace{2\mathbb{E} \left\langle  \bs{z} , \bs{y}_{t_k, \bar{\bs{y}}_k}(h) - \bar{\bs{y}}_{t_k, \bar{\bs{y}}_k}(h)\right\rangle}_{\circled{3b}}
\end{align*}

For term \circled{3a}, by the tower property of conditional expectation, we have
\begin{align*}
    2\mathbb{E} \left\langle \bs{y}_k - \bar{\bs{y}}_k , \bs{y}_{t_k, \bar{\bs{y}}_k}(h) - \bar{\bs{y}}_{t_k, \bar{\bs{y}}_k}(h)\right\rangle
    =& 2\mathbb{E}  \left[\mathbb{E} \left[ \left\langle \bs{y}_k - \bar{\bs{y}}_k , \bs{y}_{t_k, \bar{\bs{y}}_k}(h) - \bar{\bs{y}}_{t_k, \bar{\bs{y}}_k}(h)\right\rangle \bigg| \mathcal{F}_k \right]\right] \\
    =& 2\mathbb{E} \left\langle \bs{y}_k - \bar{\bs{y}}_k ,  \mathbb{E} \left[   \bs{y}_{t_k, \bar{\bs{y}}_k}(h) - \bar{\bs{y}}_{t_k, \bar{\bs{y}}_k}(h) \bigg| \mathcal{F}_k \right]\right\rangle \\
    \le& 2 \sqrt{\mathbb{E} \norm{\bs{y}_k - \bar{\bs{y}}_k}^2}  \sqrt{\mathbb{E} \norm{ \mathbb{E} \left[   \bs{y}_{t_k, \bar{\bs{y}}_k}(h) - \bar{\bs{y}}_{t_k, \bar{\bs{y}}_k}(h) \bigg| \mathcal{F}_k \right]}^2} \\
    \le& 2 \sqrt{ \mathbb{E} \norm{\bs{y}_k - \bar{\bs{y}}_k}^2}  \sqrt{ \sigma_{\text{max}}^2 \mathbb{E} \norm{ \mathbb{E} \left[   \bs{x}_{t_k, \bar{\bs{x}}_k}(h) - \bar{\bs{x}}_{t_k, \bar{\bs{x}}_k}(h) \bigg| \mathcal{F}_k \right]}^2} \\
    \le& 2 \sqrt{ \mathbb{E} \norm{\bs{y}_k - \bar{\bs{y}}_k}^2}  \sqrt{ \sigma_{\text{max}}^2 C_1^2  h^4} \\
    \le& 2\sigma_{\text{max}} C_1 \sqrt{ \mathbb{E} \norm{\bs{y}_k - \bar{\bs{y}}_k}^2}   h^2.
\end{align*}

For term \circled{3b}, when $0 < h < \frac{1}{4L^{\prime\prime}}$we have by Lemma \ref{lemma:Z} and Lemma \ref{lemma:local}
\begin{align*}
    2\mathbb{E} \left\langle  \bs{z} , \bs{y}_{t_k, \bar{\bs{y}}_k}(h) - \bar{\bs{y}}_{t_k, \bar{\bs{y}}_k}(h)\right\rangle
    \le& 2 \sqrt{\mathbb{E} \norm{\bs{z}}^2} \sqrt{\mathbb{E}\norm{\bs{y}_{t_k, \bar{\bs{y}}_k}(h) - \bar{\bs{y}}_{t_k, \bar{\bs{y}}_k}(h)}^2} \\
    =& 2 \sqrt{\mathbb{E} \norm{\bs{z}}^2} \sqrt{\mathbb{E}\left[\mathbb{E} \left[\norm{\bs{y}_{t_k, \bar{\bs{y}}_k}(h) - \bar{\bs{y}}_{t_k, \bar{\bs{y}}_k}(h)}^2 \bigg| \mathcal{F}_k \right] \right]} \\
    =& 2 \sqrt{\mathbb{E} \norm{\bs{z}}^2} \sqrt{\sigma_{\text{max}}^2 \mathbb{E}\left[\mathbb{E} \left[\norm{\bs{x}_{t_k, \bar{\bs{x}}_k}(h) - \bar{\bs{x}}_{t_k, \bar{\bs{x}}_k}(h)}^2 \bigg| \mathcal{F}_k \right] \right]} \\
    \le& 2 \sigma_{\text{max}} \sqrt{\tilde{C}\mathbb{E}\norm{\bs{y}_k - \bar{\bs{y}}_k}^2 h^2} \sqrt{C_2^2 h^3 } \\
    \le& 2 \sigma_{\text{max}} C_2 \sqrt{\tilde{C}} \sqrt{\mathbb{E} \norm{\bs{y}_k - \bar{\bs{y}}_k}^2} h^\frac{5}{2}
\end{align*}
where $\tilde{C} = 2\left( L^{\prime\prime}\right )^2 = 2(\kappa^\prime)^2 \left(L^\prime\right)^2$ is from Lemma \ref{lemma:Z} and Lemma \ref{lemma:lipschitz}.

Recall both $C_1$ and $C_2$ depend on $\norm{\bs{x}_k}$ and we would like to upper bound this term. To this end,  consider $\tilde{\bs{x}}(t)$, a solution of HFHR dynamics with initial value $\tilde{\bs{x}}_0$ that follows the invariant distribution $\tilde{\bs{x}}_0 \sim \pi$ and realizes $W_2(\pi_0, \pi)$, i.e., $\mathbb{E}\norm{\tilde{\bs{x}}_0 - \bs{x}_0}^2 = W_2^2(\pi_0, \pi)$. 

Denote $\tilde{\bs{x}}_k = \tilde{\bs{x}}(kh)$ and $e_k = \left( \mathbb{E} \norm{\bs{y}_k - \bar{\bs{y}}_k }^2\right)^\frac{1}{2}$, we then have
\begin{align*}
    \mathbb{E}\norm{\bar{\bs{x}}_k}^2 =& \mathbb{E} \norm{\bs{x}_k + \bar{\bs{x}}_k - \bs{x}_k}^2 \\
    \le& 2\mathbb{E}\norm{\bs{x}_k}^2 + 2\mathbb{E}\norm{\bar{\bs{x}}_k - \bs{x}_k}^2 \\
    \le& 4\mathbb{E}\norm{\tilde{\bs{x}}_k}^2 + 4\mathbb{E}\norm{\tilde{\bs{x}}_k - \bs{x}_k}^2 + 2\mathbb{E}\norm{\bar{\bs{x}}_k - \bs{x}_k}^2 \\
    =& 4\mathbb{E}\norm{\tilde{\bs{x}}_k}^2 + 4\mathbb{E}\norm{P^{-1} P (\tilde{\bs{x}}_k - \bs{x}_k)}^2 + 2\mathbb{E}\norm{P^{-1} P(\bar{\bs{x}}_k - \bs{x}_k)}^2 \\
    \le& 4\left( \int_{\mathbb{R}^d} \norm{\bs{q}}^2 d\mu + d \right) + \frac{4}{\sigma_{\text{min}}^2}\mathbb{E}\norm{P (\tilde{\bs{x}}_k - \bs{x}_k)}^2 + \frac{2}{\sigma_{\text{min}}^2} \mathbb{E}\norm{\bar{\bs{y}}_k - \bs{y}_k}^2 \\
    \stackrel{(i)}{\le}& 4\left( \int_{\mathbb{R}^d} \norm{\bs{q}}^2 d\mu + d \right) + \frac{4}{\sigma_{\text{min}}^2} e^{-2\lambda^\prime kh}\mathbb{E}\norm{P (\tilde{\bs{x}}_0 - \bs{x}_0)}^2 + \frac{2}{\sigma_{\text{min}}^2} e_k^2 \\
    \le& 4\left( \int_{\mathbb{R}^d} \norm{\bs{q}}^2 d\mu + d \right) + 4\kappa^2 W_2^2(\pi_0, \pi) + \frac{2}{\sigma_{\text{min}}^2} e_k^2 \\
    \triangleq& F e_k^2 + G
\end{align*}
where $(i)$ is due to Lemma \ref{lemma:contraction}. Recall from Lemma \ref{lemma:local}, we have
\begin{align*}
    C_1 \le A_1 \sqrt{\mathbb{E} \norm{\bar{x}_k}^2} + B_1 \le A_1 \sqrt{F} e_k + (A_1 \sqrt{G} + B_1) \triangleq U_1 e_k + V_1\\
    C_2 \le A_2 \sqrt{\mathbb{E} \norm{\bar{x}_k}^2} + B_2 \le A_2 \sqrt{F} e_k + (A_2 \sqrt{G} + B_2) \triangleq U_2 e_k + V_2
\end{align*}
where 
\begin{align*}
    A_1 =& (L+G)\max\{ \alpha + 1.25, \gamma + 1\}(1.74 + 0.71\alpha) \\
    B_1 =& (L+G)\max\{ \alpha + 1.25, \gamma +1 \}\left[ 0.5\alpha + (1.26\sqrt{\alpha} + 1.14 \alpha \sqrt{\alpha} + 2.32\sqrt{\gamma})\sqrt{hd}\right] \\
    A_2 =& L\max\{\alpha + 1.25, \gamma + 1\}(1.92 + 2.30\alpha L )\sqrt{h}\\
    B_2 =& L\max\{\alpha + 1.25, \gamma + 1\}(2.60\sqrt{\alpha} + 3.34\sqrt{\gamma}h)\sqrt{d}    
\end{align*}
Combine the above and bounds for terms \circled{1}, \circled{2}, \circled{3a} and \circled{3b}, we then obtain
\begin{align*}
    e_{k+1}^2 \le& (1 -\lambda^\prime h) e_k^2 +  \sigma_{\text{max}}^2 \, C_2^2 h^3 + 2\sigma_{\text{max}} C_1 e_k   h^2 + 2\sigma_{\text{max}}C_2 \sqrt{\tilde{C}} e_k h^\frac{5}{2} \\
    \le& (1 - \lambda^\prime h) e_k^2 + \sigma_{\text{max}}^2 2(U_2^2 e_k^2 + V_2^2)h^3 + 2\sigma_{\text{max}} (U_1 e_k + V_1) e_k h^2  + 2\sigma_{\text{max}} (U_2 e_k + V_2) \sqrt{\tilde{C}} e_k h^\frac{5}{2}\\
    =& \left( 1 - \lambda^\prime h + 2\sigma_{\text{max}}^2 U_2^2 h^3 + 2\sigma_{\text{max}} U_1 h^2 + 2\sigma_{\text{max}} U_2 \sqrt{\tilde{C}} h^\frac{5}{2} \right) e_k^2 \\
    &+ \left(2\sigma_{\text{max}} V_1 + 2\sigma_{\text{max}} V_2 \sqrt{\tilde{C} h}  \right) e_k h^2 + 2\sigma_{\text{max}}^2 V_2^2 h^3 \\
    \le& \left( 1 - \lambda^\prime h + 2\sigma_{\text{max}}^2 U_2^2 h^3 + 2\sigma_{\text{max}} U_1 h^2 + 2 \sigma_{\text{max}} U_2 \sqrt{\tilde{C}} h^\frac{5}{2} \right) e_k^2 + \frac{\lambda^\prime}{8}  h e_k^2  \\
    &+ \frac{2\left(2\sigma_{\text{max}} V_1 + 2 \sigma_{\text{max}} V_2 \sqrt{\tilde{C} h}  \right)^2}{\lambda^\prime}h^3 + 2\sigma_{\text{max}}^2 V_2^2 h^3 \\
    =& \left( 1 - \frac{7}{8}\lambda^\prime h + 2\sigma_{\text{max}}^2 U_2^2 h^3 + 2\sigma_{\text{max}} U_1 h^2 + 2 \sigma_{\text{max}} U_2 \sqrt{\tilde{C}} h^\frac{5}{2} \right) e_k^2 \\
    &+ \left( \frac{2\left(2\sigma_{\text{max}} V_1 + 2 \sigma_{\text{max}} V_2 \sqrt{\tilde{C} h}  \right)^2}{\lambda^\prime} + 2\sigma_{\text{max}}^2 V_2^2 \right) h^3 \\
    \stackrel{(i)}{\le}& (1 - \frac{1}{2}\lambda^\prime h) e_k^2 + \left( \frac{2\left(2\sigma_{\text{max}} V_1 + 2 \sigma_{\text{max}} V_2 \sqrt{\tilde{C} h}  \right)^2}{\lambda^\prime} + 2\sigma_{\text{max}}^2 V_2^2 \right) h^3 \\
    \triangleq& (1 - \frac{1}{2}\lambda^\prime h) e_k^2 + K h^3
\end{align*}
where $(i)$ is due to $h < \min\{ h_1, h_2,  h_3 \}$ and 
\begin{align*}
    h_1 =& \frac{\sqrt{\lambda^\prime}}{4\sqrt{2} \kappa^\prime L \max\{\alpha + 1.25, \gamma + 1\} (1.92 + 2.30\alpha L)}, \\
    h_2 =& \frac{\lambda^\prime}{16\sqrt{2} \kappa^\prime (L+G) \max\{\alpha+1.25, \gamma + 1\}(1.74 + 0.71\alpha)}, \\
    h_3 =& \frac{\lambda^\prime}{8\kappa^\prime L \max\{\alpha + 1.25, \gamma + 1\} (1.92 + 2.30\alpha L) }.
\end{align*}
Unfolding the above inequality, we arrive at
\begin{align*}
    e_k^2 \le& \left(1 - \frac{\lambda^\prime}{2}h\right)^{k} e_0^2 + \left( 1 + (1 - \frac{\lambda^\prime}{2}h) + \cdots + (1 - \frac{\lambda^\prime}{2}h)^{k-1} \right) K h^3 \\
    \stackrel{(i)}{\le}&K h^3 \sum_{i=0}^\infty \left(1 - \frac{\lambda^\prime}{2}h\right)^i  \\
    =& \frac{2K}{\lambda^\prime} h^2
\end{align*}
where $(i)$ is due to $e_k = 0$. Therefore
\[
    \left(\mathbb{E}\norm{\bs{x}_k - \bar{\bs{x}}_k}^2\right)^\frac{1}{2} = \left(\mathbb{E}\norm{P^{-1}(\bs{y}_k - \bar{\bs{y}}_k)}^2\right)^\frac{1}{2} \le \frac{1}{\sigma_{\text{min}}} e_k \le \frac{1}{\sigma_{\text{min}}} \sqrt{\frac{2K}{\lambda^\prime}}h 
\]
Collecting all the constants and we have
\begin{align*}
    \frac{1}{\sigma_{\text{min}}} \sqrt{\frac{2K}{\lambda^\prime}} \le&\frac{8\kappa^\prime}{\lambda^\prime} (L+G) \max\{\alpha+1.25, \gamma+1\} (1.74+0.71\alpha) \left( \sqrt{\int_{\mathbb{R}^d} \norm{\bs{q}}^2 d\mu + d} + \kappa^\prime W_2(\pi_0, \pi)\right) \\
    +& \frac{4\kappa^\prime}{\lambda^\prime} (L+G)\max\{\alpha+1.25,\gamma+1\}\left(0.5\alpha + (1.26\sqrt{\alpha} + 1.14\alpha \sqrt{\alpha} + 2.32\sqrt{\gamma}) \sqrt{d} \right) \\
    +& \frac{8\kappa^\prime}{\sqrt{\lambda^\prime}} \left(\frac{\sqrt{\kappa^\prime L^\prime}}{\sqrt{\lambda^\prime}} + 1 \right) L\max\{\alpha+1.25,\gamma+1\}(1.92+2.30\alpha L) \left( \sqrt{\int_{\mathbb{R}^d} \norm{\bs{q}}^2 d\mu + d} + \kappa^\prime W_2(\pi_0, \pi)\right) \\
    +& \frac{4\kappa^\prime}{\sqrt{\lambda^\prime}} \left(\frac{\sqrt{\kappa^\prime L^\prime}}{\sqrt{\lambda^\prime}} + 1 \right) L\max\{\alpha+1.25,\gamma+1\}(2.60\sqrt{\alpha} + 3.34\sqrt{\gamma})\sqrt{d} \\
    \triangleq C
\end{align*}

It is clear that in terms of the dependence on dimension $d$, we have $C = \mathcal{O}(\sqrt{d})$. In the regime where $\frac{\gamma^2 - L}{\gamma} \ge \frac{m}{\gamma} + m \alpha$, then $\lambda^\prime = \frac{m}{\gamma} + m\alpha$. Recall the definition of $\kappa^\prime$ and there exist $A^\prime, B^\prime > 0$ such that $\kappa^\prime \le A^\prime \sqrt{\alpha} + B^\prime$. It follows that 
\[
    C \le \frac{a_1 \alpha^3 + a_2 \alpha^\frac5{2} + a_3 \alpha^2 + a_4 \alpha^\frac{3}{2} + a_5 \alpha + a_6 \alpha^\frac{1}{2} + a_7}{\lambda^\prime} \le b\frac{ \alpha^3 + \frac{1}{\gamma^3}}{\lambda^\prime} =  b\frac{\alpha^3 + \frac{1}{\gamma^3}}{ \frac{m}{\gamma} + m\alpha} = \frac{b}{m} (\alpha^2 - \frac{\alpha}{\gamma} \alpha + \frac{1}{\gamma^2})
\]
for some positive constants $a_1, a_2, a_3, a_4, a_5, a_6, a_7, b > 0$ and independent of $\alpha$, in particular, we have $b = \mathcal{O}(\sqrt{d})$.
\end{proof}

\subsection{Proof of Theorem \ref{thm:wasserstein}}
\begin{proof}
Denote the $k$-th iterate of the Strang's splitting method of HFHR by $\bar{\bs{x}}_k$ with time step $h$, the solution of HFHR dynamics at time $hk$ by $\bs{x}_k$. Both $\bar{\bs{x}}_k$ and $\bs{x}_k$ start from $\bs{x}_0 = \begin{bmatrix} \bs{q}_0 \\ \bs{p}_0\end{bmatrix}$. Also denote the solution of HFHR dynamics starting from $\tilde{\bs{x}}_0$ at time $kh$ by $\tilde{\bs{x}}_k$ where $\tilde{\bs{x}}_0 = \begin{bmatrix} \tilde{\bs{q}}_0 \\ \bs{p}_0\end{bmatrix}$, $(\tilde{\bs{q}}_0, \tilde{\bs{p}}_0) \sim \pi$ and  $\mathbb{E}\norm{\begin{bmatrix}\bs{q}_0 - \tilde{\bs{q}}_0 \\ \bs{p}_0 - \tilde{\bs{p}}_0\end{bmatrix}}^2 = W_2^2(\pi_0, \pi)$. Since $\pi$ is the invariant distribution of HFHR dynamics, it follows that $\tilde{\bs{x}}_k \sim \pi$. 

By Lemma \ref{lemma:contraction} and Theorem \ref{thm:discretization}, we have 
\begin{align*}
    \displaystyle W_2(\mu_k, \mu) \leq& W_2(\pi_k, \pi) \\
    \leq& W_2(\pi_k, \text{Law}(\tilde{\bs{x}}_k)) + W_2(\text{Law}(\tilde{\bs{x}}_k), \pi) \\
    \leq& \left\{ \mathbb{E} \norm{ \bar{\bs{x}}_k - \tilde{\bs{x}}_k}^2 \right\}^\frac{1}{2} + \left\{ \mathbb{E} \norm{ \tilde{\bs{x}}_k - \bs{x}_k}^2 \right\}^\frac{1}{2} \\
    \leq& Ch + \|P^{-1}\|_2 \left\{ \mathbb{E} \norm{ P (\bs{x}_k - \tilde{\bs{x}}_k)}^2 \right\}^\frac{1}{2} \\
    \leq& Ch + \|P^{-1}\|_2 \left\{ e^{-2\lambda^\prime kh }\mathbb{E} \norm{ P (\bs{x}_0 - \tilde{\bs{x}}_0)}^2 \right\}^\frac{1}{2} \\
    =& Ch + \kappa^\prime e^{-\lambda^\prime kh } \left\{ \mathbb{E} \norm{ (\bs{x}_0 - \tilde{\bs{x}}_0)}^2 \right\}^\frac{1}{2} \\
    =& Ch + \kappa^\prime e^{-\lambda^\prime kh } W_2(\pi_0, \pi)
\end{align*}
which completes the proof.

\end{proof}

\subsection{Proof of Corollary \ref{corollary:iteration complexity}}
\begin{proof}
By Theorem \ref{thm:wasserstein},  we have 
\[
    W_2(\mu_k, \mu) \le Ch + \kappa^\prime e^{-\lambda^\prime kh }W_2(\pi_0, \pi).
\]
Given any target accuracy $\epsilon > 0$, if we run the Strang's splitting method of HFHR with $h^\star = \min\{h_0, \frac{\epsilon}{2C}\}$, then after $k^\star = \frac{1}{\lambda^\prime} \max\{\frac{1}{h_0}, \frac{2C}{\epsilon}\} \log \frac{2\kappa^\prime W_2(\pi_0, \pi)}{\epsilon}$, we have
\[
    W_2(\mu_{k^\star}, \mu) \le Ch + \kappa^\prime e^{-\lambda^\prime kh }W_2(\mu_0, \mu) \le \frac{\epsilon}{2} + \frac{\epsilon}{2} = \epsilon.
\]
Recall $C = \mathcal{O}(\sqrt{d})$, when high accuracy is needed, e.g. $\epsilon < 2Ch_0$, the iteration complexity to reach $\epsilon$-accuracy under 2-Wasserstein distance is $k^\star = \mathcal{O}(\frac{\sqrt{d}}{\epsilon} \log\frac{1}{\epsilon}) = 2 \frac{C}{\lambda^\prime} \frac{1}{\epsilon} \log \frac{2 \kappa^\prime W_2(\pi_0, \pi)}{\epsilon} = \tilde{\mathcal{O}}(\frac{\sqrt{d}}{\epsilon})$. Recall from Theorem \ref{thm:discretization}, $C \le \frac{b}{m} (\alpha^2 - \frac{\alpha}{\gamma} + \frac{1}{\gamma^2})$, we have
\[
    \frac{C}{\lambda^\prime} \le \frac{b}{m^2} \frac{\alpha^2 - \frac{\alpha}{\gamma} + \frac{1}{\gamma^2}}{\frac{1}{\gamma} + \alpha}
\]
Denote $g(\alpha) = \frac{b}{m^2} \frac{\alpha^2 - \frac{\alpha}{\gamma} + \frac{1}{\gamma^2}}{\frac{1}{\gamma} + \alpha}$, simple calculation shows that $\alpha^\star = \argmin_{\alpha \ge 0} g(\alpha) = \frac{\sqrt{3} - 1}{\gamma} = \mathcal{O}(\frac{1}{\gamma})$.
\end{proof}

\section{Technical/Auxiliary Lemmas and Their Proofs}\label{sec:auxiliary results}

\subsection{Dependence of error of SDE on initial values}
\begin{lemma}\label{lemma:Z}
    Consider the following two SDE with different initial condition
    \[
        \begin{cases}
        d\bs{x}_t = \bs{a}(\bs{x}_t) dt + \bs{\sigma} d \bs{W}_t,\\
        \bs{x}(0) = \bs{x}_0
        \end{cases}\quad 
        \begin{cases}
        d\bs{y}_t = \bs{a}(\bs{y}_t) dt + \bs{\sigma} d \bs{W}_t,\\
        \bs{y}(0) = \bs{y}_0
        \end{cases}
    \]
    where $\bs{a}(\bs{u}) \in \mathbb{R}^d$ is $L$-Lipschitz, and $\bs{\sigma} \in \mathbb{R}^{n \times n}$ is a constant matrix. For $0 < h < \frac{1}{4L}$, we have the following representation
    \[
        \bs{x}_h - \bs{y}_h = \bs{x}_0 - \bs{y}_0 + \bs{z}
    \]
    with 
    \[
        E\norm{\bs{z}}^2 \le 2L^2 \norm{\bs{x}_0 - \bs{y}_0}^2 h^2
    \]
\end{lemma}
\begin{proof}
Let $\bs{z} = (\bs{x}_h - \bs{y}_h) - (\bs{x}_0 - \bs{y}_0) = \int_0^h \bs{a}(\bs{x}_s) - \bs{a}(\bs{y}_s) ds$.
Ito's lemma readily implies that
\begin{align*}
    \mathbb{E}\norm{\bs{x}_h - \bs{y}_h}^2 =& \norm{\bs{x}_0 - \bs{y}_0}^2 + 2 \mathbb{E} \int_0^h \innerprod{\bs{x}_s - \bs{y}_s}{ \bs{a}(\bs{x}_s) - \bs{a}(\bs{y}_s)} ds \\
    \le& \norm{\bs{x}_0 - \bs{y}_0}^2 + 2L \int_0^h \mathbb{E} \norm{\bs{x}_s - \bs{y}_s}^2 ds
\end{align*}
By Gronwall's inequality, it follows that 
\[
    \mathbb{E} \norm{\bs{x}_h - \bs{y}_h}^2 \le \norm{\bs{x}_0 - \bs{y}_0}^2 e^{2Lh} \le 2\norm{\bs{x}_0 - \bs{y}_0}^2, \mbox{ for } 0 < h < \frac{1}{4L}
\]
and
\begin{align*}
    \mathbb{E}\norm{\bs{z}}^2 =  \norm{ \mathbb{E} \left[\int_0^h \bs{a}(\bs{x}_s) - \bs{a}(\bs{y}_s) ds \right] }^2 
    \le&  \left( \int_0^h \norm{\mathbb{E} \left[\bs{a}(\bs{x}_s) - \bs{a}(\bs{y}_s) \right]} ds \right) ^2  \\
    \le&  \int_0^h 1^2 ds \int_0^h \norm{\mathbb{E} \left[\bs{a}(\bs{x}_s) - \bs{a}(\bs{y}_s)\right]}^2 ds \\
    \le& h  \int_0^h \mathbb{E}  \norm{\bs{a}(\bs{x}_s) - \bs{a}(\bs{y}_s)}^2 ds   \\
    \le& L^2 h \int_0^h \mathbb{E} \norm{\bs{x}_s - \bs{y}_s}^2 ds \\
    \le& 2L^2 \norm{\bs{x}_0 - \bs{y}_0}^2 h^2
\end{align*}
\end{proof}

\subsection{Growth bound of SDE with additive noise}
\begin{lemma}\label{lemma:growth-bound}
Consider the following SDE with constant diffusion
\[
    \begin{cases}
    d\bs{x}_t = \bs{a}(\bs{x}_t) dt + \bs{\sigma} d \bs{W}_t,\\
    \bs{x}(0) = \bs{x}_0
    \end{cases}
\]
where $\bs{a}(\bs{x}) \in \mathbb{R}^d$ is $L$-smooth, i.e., $|\bs{a}(\bs{y} ) - \bs{a}(\bs{x})| \le L |\bs{y} - \bs{x}|$, $\bs{a}(\bs{0}) = \bs{0}$  and $\bs{\sigma} \in \mathbb{R}^{d \times d}$ is a constant matrix independent of time $t$ and $\bs{x}_t$. Then for $0 < h < \frac{1}{4L}$, we have
\[
     \mathbb{E} \norm{\bs{x}_h - \bs{x}_0}^2 \le 2.57 \left( \|\bs{\sigma}\|_F^2 + 2hL^2\norm{\bs{x}_0}^2 \right) h.
\]
\end{lemma}
\begin{proof}
We have
\begin{align*}
    \mathbb{E}\norm{\bs{x}_h - \bs{x}_0}^2 =& \mathbb{E} \norm{ \int_0^h \bs{a}(\bs{x}_t)dt + \int_0^h \bs{\sigma} d\bs{W}_t }^2\\
    \le& 2\mathbb{E}\norm{\int_0^h \bs{a}(\bs{x}_t)dt}^2 + 2\mathbb{E} \norm{\int_0^h \bs{\sigma}d\bs{W}_t}^2 \\
    \stackrel{(i)}{=}& 2\mathbb{E}\norm{\int_0^h \bs{a}(\bs{x}_t)dt}^2 + 2 \int_0^h \|\bs{\sigma}\|_F^2 dt \\
    \le& 2\mathbb{E}\left[ \left(\int_0^h \norm{\bs{a}(\bs{x}_t)} dt \right)^2 \right] + 2h \|\bs{\sigma}\|_F^2  \\
    \le& 2\mathbb{E}\left[ \left(\int_0^h \norm{\bs{a}(\bs{x}_t) - \bs{a}(\bs{x}_0)} dt + \int_0^h \norm{\bs{a}(\bs{x}_0)} dt \right)^2 \right] + 2h \|\bs{\sigma}\|_F^2 \\
    \le& 2\mathbb{E}\left[ \left(L \int_0^h \norm{\bs{x}_t - \bs{x}_0} dt + h\norm{\bs{a}(\bs{x}_0)}\right)^2 \right] + 2h \|\bs{\sigma}\|_F^2 \\
    \le& 4\mathbb{E}\left[ L^2\left( \int_0^h \norm{\bs{x}_t - \bs{x}_0} dt \right)^2 + h^2\norm{\bs{a}(\bs{x}_0)}^2 \right] + 2h \|\bs{\sigma}\|_F^2 \\ 
    \stackrel{(ii)}{\le}& 2h \|\bs{\sigma}\|_F^2 + 4h^2\norm{\bs{a}(\bs{x}_0)}^2 + 4L^2 h \int_0^h \mathbb{E}\norm{\bs{x}_t - \bs{x}_0 }^2dt
\end{align*}
where $(i)$ is due to Ito's isometry, $(ii)$ is due to Cauchy-Schwarz inequality and $\|\bs{\sigma}\|_F$ is the Frobenius norm of $\bs{\sigma}$. By Gronwall's inequality, we obtain
\[
    \mathbb{E} \norm{\bs{x}_h - \bs{x}_0}^2 \le \left( 2h \|\bs{\sigma}\|_F^2 + 4h^2\norm{\bs{a}(\bs{x}_0)}^2 \right) \exp\left\{ 4L^2h^2 \right\}.
\]
Since $\norm{\bs{a}(\bs{x}_0)} = \norm{\bs{a}(\bs{x}_0) - \bs{a}(\bs{0})} \le L \norm{\bs{x}_0}$, when $0 < h < \frac{1}{4L}$, we finally reach at 
\[
    \mathbb{E} \norm{\bs{x}_h - \bs{x}_0}^2 \le 2\left( \|\bs{\sigma}\|_F^2 + 2hL^2\norm{\bs{x}_0}^2 \right) e^\frac{1}{4} h \le 2.57 \left( \|\bs{\sigma}\|_F^2 + 2hL^2\norm{\bs{x}_0}^2 \right) h.
\]
\end{proof}

\subsection{Lipschitz continuity of the drift of HFHR dynamics}
\begin{lemma}\label{lemma:lipschitz}
Assume $\nabla f$ is $L$-Lipschitz, i.e. $\norm{\nabla f(\bs{x}) - \nabla f(\bs{y})} \le L \norm{\bs{x} - \bs{y}}$, then the drift term of HFHR dynamics
\[
    \begin{bmatrix} \bs{p} - \alpha \nabla f(\bs{q}) \\ -\gamma \bs{p} - \nabla f(\bs{q}) \end{bmatrix}
\]
is $L^\prime$-Lipschitz, where $L^\prime \triangleq \sqrt{2} \max\{ \sqrt{1 + \alpha^2} \max\{\frac{1}{\sqrt{2}}, L\}, \sqrt{1 + \gamma^2}\}$. Let $P$ be defined in Appendix \ref{app:notation} and $\begin{bmatrix} \bs{\phi} \\ \bs{\psi} \end{bmatrix} = P \begin{bmatrix} \bs{q} \\ \bs{p} \end{bmatrix}$, then $\begin{bmatrix} \bs{\phi} \\ \bs{\psi} \end{bmatrix}$ satisfies the following SDE
\[
    \begin{bmatrix} d\bs{\phi} \\ d\bs{\psi} \end{bmatrix} = P \begin{bmatrix} \bs{p}(\bs{\phi}, \bs{\psi}) - \alpha \nabla f(\bs{q}(\bs{\phi}, \bs{\psi})) \\ -\gamma \bs{p}(\bs{\phi}, \bs{\psi}) - \nabla f(\bs{q}(\bs{\phi}, \bs{\psi})) \end{bmatrix} dt + P \begin{bmatrix} \sqrt{2\alpha} I & 0 \\ 0 & \sqrt{2\gamma} I \end{bmatrix} \begin{bmatrix} d\bs{W} \\ d\bs{B} \end{bmatrix}
\]
and the drift term
\[
    P \begin{bmatrix} \bs{p}(\bs{\phi}, \bs{\psi}) - \alpha \nabla f(\bs{q}(\bs{\phi}, \bs{\psi})) \\ -\gamma \bs{p}(\bs{\phi}, \bs{\psi}) - \nabla f(\bs{q}(\bs{\phi}, \bs{\psi})) \end{bmatrix}
\]
is $L^{\prime \prime}$-Lipschitz, where $L^{\prime\prime} = \kappa^\prime L^\prime$ and $\kappa^\prime$ is the condition number of $P$.
\end{lemma}
\begin{proof}
By direct computation and Cauchy-Schwarz inequality, we have
\begin{align*}
    &\norm{\begin{bmatrix} \bs{p}_1 - \alpha \nabla f(\bs{q}_1) \\ -\gamma \bs{p}_1 - \nabla f(\bs{q}_1) \end{bmatrix} - \begin{bmatrix} \bs{p}_2 - \alpha \nabla f(\bs{q}_2) \\ -\gamma \bs{p}_2 - \nabla f(\bs{q}_2) \end{bmatrix}} \\
    =& \sqrt{\norm{-\alpha \left( \nabla f(\bs{q}_1) - \nabla f(\bs{q}_2) \right) + (\bs{p}_1 - \bs{p}_2) }^2 + \norm{-\left( \nabla f(\bs{q}_1) - \nabla f(\bs{q}_2) \right) -\gamma (\bs{p}_1 - \bs{p}_2)}^2} \\
    \le& \sqrt{2\alpha^2\norm{\nabla f(\bs{q}_1) - \nabla f(\bs{q}_2)} + 2\norm{\bs{p}_1 - \bs{p}_2 }^2 + 2\norm{ \nabla f(\bs{q}_1) - \nabla f(\bs{q}_2)} + 2\gamma^2 \norm{\bs{p}_1 - \bs{p}_2}^2} \\ 
    \le& \sqrt{(2\alpha^2L^2 + 2L^2)\norm{\bs{q}_1 -\bs{q}_2} + (2 + 2 \gamma^2)\norm{\bs{p}_1 - \bs{p}_2 }^2} \\
    \le& \sqrt{2} \max\{L \sqrt{1 + \alpha^2}, \sqrt{1 + \gamma^2}\} \norm{\begin{bmatrix} \bs{q}_1 - \bs{q}_2 \\ \bs{p}_1\ - \bs{p}_2 \end{bmatrix}}  \\
    \le& \sqrt{2} \max\{ \sqrt{1 + \alpha^2} \max\{\frac{1}{\sqrt{2}}, L \} , \sqrt{1 + \gamma^2}\} \norm{\begin{bmatrix} \bs{q}_1 - \bs{q}_2 \\ \bs{p}_1\ - \bs{p}_2 \end{bmatrix}} \\
    \triangleq & L^\prime \norm{\begin{bmatrix} \bs{q}_1 - \bs{q}_2 \\ \bs{p}_1\ - \bs{p}_2 \end{bmatrix}}
\end{align*}
By Ito's lemma, we have 
\[
    \begin{bmatrix} d\bs{\phi} \\ d\bs{\psi} \end{bmatrix} = P \begin{bmatrix} \bs{p}(\bs{\phi}, \bs{\psi}) - \alpha \nabla f(\bs{q}(\bs{\phi}, \bs{\psi})) \\ -\gamma \bs{p}(\bs{\phi}, \bs{\psi}) - \nabla f(\bs{q}(\bs{\phi}, \bs{\psi})) \end{bmatrix} dt + P \begin{bmatrix} \sqrt{2\alpha} I & 0 \\ 0 & \sqrt{2\gamma} I \end{bmatrix} \begin{bmatrix} d\bs{W} \\ d\bs{B} \end{bmatrix}
\]
Using the Lipschitz constant obtained for the drift of HFHR, we further have 
\begin{align*}
    & \norm{    P \begin{bmatrix} \bs{p}(\bs{\phi}_1, \bs{\psi}_1) - \alpha \nabla f(\bs{q}(\bs{\phi}_1, \bs{\psi}_1)) \\ -\gamma \bs{p}(\bs{\phi}_1, \bs{\psi}_1) - \nabla f(\bs{q}(\bs{\phi}_1, \bs{\psi}_1)) \end{bmatrix} -     P \begin{bmatrix} \bs{p}(\bs{\phi}_2, \bs{\psi}_2) - \alpha \nabla f(\bs{q}(\bs{\phi}_2, \bs{\psi}_2)) \\ -\gamma \bs{p}(\bs{\phi}_2, \bs{\psi}_2) - \nabla f(\bs{q}(\bs{\phi}_2, \bs{\psi}_2)) \end{bmatrix}} \\
    \le& \sigma_{\text{max}} \norm{\begin{bmatrix} \bs{p}_1 - \alpha \nabla f(\bs{q}_1) \\ -\gamma \bs{p}_1 - \nabla f(\bs{q}_1) \end{bmatrix} - \begin{bmatrix} \bs{p}_2 - \alpha \nabla f(\bs{q}_2) \\ -\gamma \bs{p}_2 - \nabla f(\bs{q}_2) \end{bmatrix}} \\
    \le& \sigma_{\text{max}} L^\prime \norm{\begin{bmatrix} \bs{q}_1 - \bs{q}_2 \\ \bs{p}_1\ - \bs{p}_2 \end{bmatrix}} \\
    \le& \sigma_{\text{max}} L^\prime \norm{P^{-1}\begin{bmatrix} \bs{\phi}_1 - \bs{\phi}_2 \\ \bs{\psi}_1\ - \bs{\psi}_2 \end{bmatrix}} \\
    \le& \sigma_{\text{max}} L^\prime \frac{1}{\sigma_{\text{min}}} \norm{\begin{bmatrix} \bs{\phi}_1 - \bs{\phi}_2 \\ \bs{\psi}_1\ - \bs{\psi}_2 \end{bmatrix}} \\
    =& \kappa^\prime L^\prime \norm{\begin{bmatrix} \bs{\phi}_1 - \bs{\phi}_2 \\ \bs{\psi}_1\ - \bs{\psi}_2 \end{bmatrix}}
\end{align*}
where $\sigma_{\text{max}}, \sigma_{\text{min}}$ and $\kappa^\prime$ are the largest, smallest singular values and the condition number (w.r.t. 2-norm) of matrix $P$.
\end{proof}
\begin{remark}
The following inequalities associated with $L^\prime$ will turn out to be useful in many proofs
\[
    L^\prime \ge 1, \, L^\prime \ge \sqrt{2}\gamma, L^\prime \ge \sqrt{2} \alpha, L \ge \sqrt{2} L \mbox{ and } L^\prime \ge \sqrt{2} \alpha L.
\]
\end{remark}

\subsection{Contraction of (Transformed) HFHR Dynamics}
\begin{lemma}\label{lemma:contraction}
Suppose $f$ is $L$-smooth, $m$-strongly convex and $\gamma^2 > L$. Consider two copies of HFHR dynamics $\begin{bmatrix}\qbold_t \\ \pbold_t \end{bmatrix}$, $\begin{bmatrix}\tqbold_t \\ \tpbold_t \end{bmatrix}$ (driven by the same Brownian motion) with initialization $\begin{bmatrix}\qbold_0 \\ \pbold_0 \end{bmatrix}$, $\begin{bmatrix}\tqbold_0 \\ \tpbold_0 \end{bmatrix}$ respectively, then we have
\[
   \norm{P\begin{bmatrix}
   \qbold_t - \tqbold_t \\
   \pbold_t - \tpbold_t
   \end{bmatrix}} \le e^{-\lambda^\prime t} \norm{P\begin{bmatrix}
   \qbold_0 - \tqbold_0 \\
   \pbold_0 - \tpbold_0
   \end{bmatrix}}
\]
where $P = \begin{bmatrix} \gamma I & I \\ 0 & \sqrt{1+\alpha\gamma}I\end{bmatrix}$ and $\lambda^\prime = \min\{ \frac{m}{\gamma} + \alpha m, \frac{\gamma^2 - L}{\gamma} \}$.
\end{lemma}
\begin{proof}
Consider two copies of HFHR that are driven by the same Brownian motion
\[
    \begin{cases}
        d\qbold_t = (\pbold_t - \alpha \nabla f(\qbold_t))dt + \sqrt{2\alpha} d\bs{B}^1_t \\
        d\pbold_t = (-\gamma \pbold_t - \nabla f(\qbold_t))dt + \sqrt{2\gamma} d\bs{B}^2_t
    \end{cases},
    \quad
    \begin{cases}
        d\tilde{\qbold}_t = (\tpbold_t - \alpha \nabla f(\tilde{\qbold}_t))dt + \sqrt{2\alpha} d\bs{B}^1_t \\
        d\tilde{\pbold_t} = (-\gamma \tpbold_t - \nabla f(\tilde{\qbold}_t))dt + \sqrt{2\gamma} d\bs{B}^2_t
    \end{cases}.
\]
Based on Taylor's expansion, the difference of the two copies is expressed as
\begin{align*}
    \frac{d}{dt}\begin{bmatrix}\qbold_t - \tilde{\qbold}_t \\ \pbold_t - \tilde{\pbold_t}\end{bmatrix} 
    =& -
    \begin{bmatrix}
    \alpha H_t & -I \\
    H_t & \gamma I
    \end{bmatrix}
    \begin{bmatrix}\qbold_t - \tilde{\qbold}_t \\ \pbold_t - \tilde{\pbold_t}\end{bmatrix} \triangleq - A \begin{bmatrix}\qbold_t - \tilde{\qbold}_t \\ \pbold_t - \tilde{\pbold_t}\end{bmatrix}
\end{align*}
where $H_t = \int_0^1 \nabla^2 f(\tilde{\qbold}_t + s(\qbold - \tilde{\qbold}_t)) ds$. Denote the eigenvalues of $H_t$ by $\eta_i, 1\le i \le d$, by strong convexity and smoothness assumption on $f$, we have $m \le \eta_i \le L, 1\le i \le d$.

Denote $\phipsi = P\deltaqp$ and consider $\mathcal{L}_t = \frac{1}{2}\norm{\phipsi}^2$, we have 
\begin{align*}
    \frac{d}{dt}\mathcal{L}_t =& - \phipsi^T P A P^{-1} \phipsi \\
    =& - \phipsi^T \frac{1}{2}(PAP^{-1} + (P^{-1})^T A^T P^T) \phipsi \\
    =& - \phipsi^T \frac{1}{\gamma} \begin{bmatrix} (1+\alpha\gamma)H_t & 0_{d\times d} \\ 0_{d\times d} & \gamma^2I - H_t \end{bmatrix} \phipsi \\
    \triangleq& -\phipsi^T B(\alpha) \phipsi
\end{align*}
It is easy to see that 
\begin{align*}
    \lambda_{\min}(B(\alpha)) = \min_{i=1,2,\cdots,d} \{ \min\{ \frac{\eta_i}{\gamma} + \alpha \eta_i, \gamma - \frac{\eta_i}{\gamma}\} \} \ge \min\{ \frac{m}{\gamma} + \alpha m, \frac{\gamma^2 - L}{\gamma} \}  \triangleq \lambda^\prime.
\end{align*}
Therefore we have
$
    \frac{d}{dt}\mathcal{L}_t \le - 2\lambda_{\min}{B(\alpha)} \mathcal{L}_t \le - 2 \lambda^\prime \mathcal{L}_t
$. By Gronwall's inequality, we obtain
\[
    \norm{\phipsi}^2 \le e^{-2\lambda^\prime t} \norm{\begin{bmatrix} \bs{\phi}_0 \\ \bs{\psi}_0 \end{bmatrix}}^2.
\]
and the desired inequality follows by taking square root.
\end{proof}

\subsection{Local error between the exact Strang's splitting method and HFHR dynamics}
\label{sec_localError1}
\begin{lemma}\label{lemma:local-exact}
Assume $f$ is $L$-smooth and $\bs{0} \in \argmin_{\bs{x} \in \mathbb{R}^d} f(\bs{x})$, i.e. $\nabla f(\bs{0}) = \bs{0}$. If $0 < h \le  \frac{1}{4 L^\prime} $, then compared with the HFHR dynamics, the exact Strang's splitting method has local mathematical expectation of deviation of order $p_1 = 2$ and local mean-squared error of order $p_2 = 2$, i.e. there exist constants $\widehat{C}_1, \widehat{C}_2 > 0$ such that 
\begin{equation*}
    \norm{\mathbb{E} \bs{x}(h)  - \mathbb{E} \hat{\bs{x}}(h) } \le \widehat{C}_1  h^{p_1}
\end{equation*}
\begin{equation*}
    \left(\mathbb{E} \left[ \norm{ \bs{x}(h) - \hat{\bs{x}}(h) }^2 \right]\right)^\frac{1}{2} \le \widehat{C}_2 h^{p_2}
\end{equation*}
where $\bs{x}(h) = \begin{bmatrix} \bs{q}(h) \\ \bs{p}(h) \end{bmatrix}$ is the solution of the HFHR dynamics with initial value $\bs{x}_0 = \begin{bmatrix} \bs{q}_0 \\ \bs{p}_0 \end{bmatrix}$ and $\hat{\bs{x}}(h) = \begin{bmatrix} \hat{\bs{q}}(h) \\ \hat{\bs{p}}(h) \end{bmatrix}$ is the solution of the implementable Strang's splitting with initial value $\bs{x}_0 = \begin{bmatrix} \bs{q}_0 \\ \bs{p}_0 \end{bmatrix}$, $p_1 = 2$ and $p_2 = 2$. More concretely, we have 
\[
  \widehat{C}_1 = L \max\{\alpha + 1.25, \gamma + 1\} \left(1.74 \norm{\bs{x}_0} + ( 1.26\sqrt{\alpha} + 2.84\sqrt{\gamma})\sqrt{hd}\right),
\]
\[
    \widehat{C}_2 =  L \max\{ \alpha + 1.25, \gamma + 1\} \left( 1.92 \norm{\bs{x}_0} + (1.30\sqrt{\alpha} + 3.22\sqrt{\gamma})\sqrt{hd} \right).
\]
\end{lemma}
\begin{proof}
The exact Strang's splitting integrator with step size $h$ reads as $\phi^{\frac{h}{2}} \circ \psi^{h} \circ \phi^{\frac{h}{2}}$ where
\begin{equation*}
    \phi: 
    \begin{cases}
        d\bs{q} = \bs{p} dt \\
        d\bs{p} = -\gamma \bs{p} dt + \sqrt{2\gamma} d\bs{B}
    \end{cases}
    \,
    \psi:
    \begin{cases}
        d\bs{q} = -\alpha \nabla f(\bs{q}) dt + \sqrt{2\alpha} d\bs{W} \\
        d\bs{p} = -\nabla f(\bs{q}) dt
    \end{cases}.
\end{equation*}
The $\phi$ flow can be explicitly solved and the solution is 
\begin{equation*}
    \begin{cases}
        \bs{q}(t) = \bs{q}_0 + \frac{1 - e^{-\gamma t}}{\gamma} \bs{p}_0 + \sqrt{2\gamma} \int_0^t \frac{1 - e^{-\gamma(t - s)}}{\gamma} d\bs{B}(s)\\
        \bs{p}(t) = e^{-\gamma t}\bs{p}_0 + \sqrt{2\gamma} \int_0^t e^{-\gamma (t - s)} d\bs{B}(s)
    \end{cases}.
\end{equation*}
The $\psi$ flow can be written as
\begin{equation*}
    \begin{cases}
        \bs{q}(t) = \bs{q}_0 - \int_0^t \alpha \nabla f(\bs{q}(s)) ds + \sqrt{2\alpha} \int_0^t d\bs{W}(s) \\
        \bs{p}(t) = \bs{p}_0 - \int_0^t \nabla f(\bs{q}(s)) ds
    \end{cases}.
\end{equation*}

The solution of one-step exact Strang's  splitting integrator with step size $h$ can be written as 
\[
    \begin{cases}
    \bs{q}_3 = \bs{q}_2(h) + \frac{1 - e^{-\gamma \frac{h}{2}}}{\gamma} \bs{p}_2(h) + \sqrt{2\gamma} \int_\frac{h}{2}^h \frac{1 - e^{-\gamma(h - s)}}{\gamma} d\bs{B}(s)\\
    \bs{p}_3 = e^{-\gamma \frac{h}{2}}\bs{p}_2(h) + \sqrt{2\gamma} \int_\frac{h}{2}^h e^{-\gamma (h - s)} d\bs{B}(s) \\
    \bs{q}_2(r) = \bs{q}_1 - \int_0^r \alpha \nabla f(\bs{q}_2(s)) ds + \sqrt{2\alpha}\int_0^r d\bs{W}(s)  \qquad (0 \le r \le h)\\
    \bs{p}_2(r) = \bs{p}_1 - \int_0^r \nabla f(\bs{q}_2(s)) ds \\
    \bs{q}_1 = \bs{q}_0 + \frac{1 - e^{-\gamma \frac{h}{2}}}{\gamma} \bs{p}_0 + \sqrt{2\gamma} \int_0^\frac{h}{2} \frac{1 - e^{-\gamma(\frac{h}{2} - s)}}{\gamma} d\bs{B}(s)\\
    \bs{p}_1 = e^{-\gamma \frac{h}{2}}\bs{p}_0 + \sqrt{2\gamma} \int_0^\frac{h}{2} e^{-\gamma (\frac{h}{2} - s)} d\bs{B}(s)
    \end{cases}
\]
Therefore, we have $\hat{\bs{q}}(h) =  \bs{q}_3, \hat{\bs{p}}(h) = \bs{p}_3$ and  
\begin{align*}
    \hat{\bs{q}}(h) =& \sqrt{2\gamma} \int_\frac{h}{2}^h \frac{1 - e^{-\gamma(h - s)}}{\gamma} d\bs{B}(s) 
    + \underbrace{\bs{q}_1 - \int_0^h \alpha \nabla f(\bs{q}_2(s)) ds + \sqrt{2\alpha}\int_0^h d\bs{W}(s)}_{\bs{q}_2(h)} \\
    & + \frac{1 - e^{-\gamma \frac{h}{2}}}{\gamma} \left[ \underbrace{\bs{p}_1 - \int_0^h \nabla f(\bs{q}_2(s)) ds}_{\bs{p}_2(h)} \right] \\
    =&  \sqrt{2\gamma} \int_\frac{h}{2}^h \frac{1 - e^{-\gamma(h - s)}}{\gamma} d\bs{B}(s)  - \int_0^h \alpha \nabla f(\bs{q}_2(s)) ds + \sqrt{2\alpha}\int_0^h d\bs{W}(s) - \frac{1 - e^{-\gamma \frac{h}{2}}}{\gamma} \int_0^h \nabla f(\bs{q}_2(s)) ds \\
    & +\underbrace{\bs{q}_0 + \frac{1 - e^{-\gamma \frac{h}{2}}}{\gamma} \bs{p}_0 + \sqrt{2\gamma} \int_0^\frac{h}{2} \frac{1 - e^{-\gamma(\frac{h}{2} - s)}}{\gamma} d\bs{B}(s)}_{\bs{q}_1} + \frac{1 - e^{-\gamma \frac{h}{2}}}{\gamma} \left[ \underbrace{e^{-\gamma \frac{h}{2}}\bs{p}_0 + \sqrt{2\gamma} \int_0^\frac{h}{2} e^{-\gamma (\frac{h}{2} - s)} d\bs{B}(s)}_{\bs{p}_1} \right] \\
    =& \bs{q}_0  + \frac{1 - e^{-\gamma h}}{\gamma}  \bs{p}_0 - \left(\alpha +  \frac{1 - e^{-\gamma \frac{h}{2}}}{\gamma}\right) \int_0^h \nabla f(\bs{q}_2(s))  ds \\
    &+ \sqrt{2\alpha}\int_0^h d\bs{W}(s) + \sqrt{2\gamma} \int_{\frac{h}{2}}^h \frac{1 - e^{-\gamma(h - s)}}{\gamma} d\bs{B}(s) + \sqrt{2\gamma} \int_0^\frac{h}{2} \frac{1 - e^{-\gamma (\frac{h}{2} -s)}}{\gamma}d\bs{B}(s) \\
    & + \frac{1 - e^{-\gamma \frac{h}{2}}}{\gamma}  \sqrt{2\gamma} \int_0^\frac{h}{2} e^{-\gamma (\frac{h}{2} - s)} d\bs{B}(s)  \\
    \hat{\bs{p}}(h) =& e^{-\gamma \frac{h}{2}} \left[ \underbrace{\bs{p}_1 - \int_0^h \nabla f(\bs{q}_2(s)) ds}_{\bs{p}_2(h)} \right] + \sqrt{2\gamma} \int_{\frac{h}{2}}^h e^{-\gamma (h - s)} d\bs{B}(s)  \\
    =& e^{-\gamma \frac{h}{2}} \left[ \underbrace{e^{-\gamma \frac{h}{2}}\bs{p}_0 + \sqrt{2\gamma} \int_0^\frac{h}{2} e^{-\gamma (\frac{h}{2} - s)} d\bs{B}(s)}_{\bs{p}_1} \right] - e^{-\gamma \frac{h}{2}} \int_0^h \nabla f(\bs{q}_2(s)) ds + \sqrt{2\gamma} \int_{\frac{h}{2}}^h e^{-\gamma (h - s)} d\bs{B}(s) \\
    =& e^{-\gamma h} \bs{p}_0 - e^{-\gamma \frac{h}{2}} \int_0^h \nabla f(\bs{q}_2(s)) ds + e^{-\gamma \frac{h}{2}} \sqrt{2\gamma} \int_0^\frac{h}{2} e^{-\gamma (\frac{h}{2} - s)} d\bs{B}(s) + \sqrt{2\gamma} \int_{\frac{h}{2}}^h e^{-\gamma (h - s)} d\bs{B}(s)
\end{align*}
It is clear that $\hat{\bs{q}}(h), \hat{\bs{p}}(h)$ should be compared with the exact solution of HFHR at time $h$, which can be written as 
\begin{align*}
    \bs{q}(h) =& \bs{q}_0 
    + \frac{1 - e^{-\gamma h}}{\gamma} \bs{p}_0 
    - \int_0^h \left(\frac{1 - e^{-\gamma (h - s)}}{\gamma} + \alpha \right) \nabla f(\bs{q}(s)) ds 
    + \sqrt{2\alpha} \int_0^h d\bs{W}_s
    + \sqrt{2\gamma} \int_0^h \frac{1 - e^{-\gamma (h-s)}}{\gamma} d\bs{B}_s\\
    \bs{p}(h) =& e^{-\gamma h} \bs{p}_0 - \int_0^h e^{-\gamma(h - s)} \nabla f(\bs{q}(s)) ds + \sqrt{2\gamma} \int_0^h e^{-\gamma(h -s )} d\bs{B}(s)
\end{align*}
Subtracting $\bs{q}(h), \bs{p}(h)$ from $\hat{\bs{q}}(h), \hat{\bs{p}}(h)$ respectively, we obtain
\begin{align*}
    \hat{\bs{q}}(h) - \bs{q}(h) =& -\left( \alpha +  \frac{1 - e^{-\gamma \frac{h}{2}}}{\gamma} \right) \int_0^h \nabla f(\bs{q}_2(s)) - \nabla f(\bs{q}(s)) ds  \\
    & + \int_0^h \left( \frac{1 - e^{-\gamma (h - s)}}{\gamma} - \frac{1 - e^{-\gamma\frac{h}{2}}}{\gamma} \right) \nabla f(\bs{q}(s)) ds \\
    \hat{\bs{p}}(h) - \bs{p}(h) =& -e^{-\gamma \frac{h}{2}}\int_0^h \nabla f(\bs{q}_2(s)) - \nabla f(\bs{q}(s)) ds 
    + \int_0^h \left( e^{-\gamma(h - s)} - e^{-\gamma \frac{h}{2}} \right) \nabla f(\bs{q}(s))ds
\end{align*}
It should be clear now that we will need to bound the term $\nabla f(\bs{q}_2) - \nabla f(\bs{q})$ and $\nabla f(\bs{q})$. Since
\begin{align*}
    \bs{q}_2(r) =& \bs{q}_0 + \frac{1 - e^{-\gamma \frac{h}{2}}}{\gamma}\bs{p}_0 + \sqrt{2\gamma}\int_0^\frac{h}{2} \frac{1 - e^{-\gamma (\frac{h}{2} - s)}}{\gamma}d\bs{B}(s) - \alpha \int_0^r \nabla f(\bs{q}_2(s))ds + \sqrt{2\alpha} \int_0^r d\bs{W}(s) \\
    \bs{q}(r) =& \bs{q}_0 
    + \frac{1 - e^{-\gamma r}}{\gamma} \bs{p}_0 
    - \int_0^r \left(\frac{1 - e^{-\gamma (r - s)}}{\gamma} +\alpha \right) \nabla f(\bs{q}(s)) ds 
    + \sqrt{2\alpha} \int_0^r d\bs{W}(s) \\
    &+ \sqrt{2\gamma} \int_0^r \frac{1 - e^{-\gamma (r-s)}}{\gamma} d\bs{B}(s) \nonumber,
\end{align*}
we then have
\begin{align*}
    \bs{q}_2(r) - \bs{q}(r) =& \frac{e^{-\gamma r} - e^{-\gamma \frac{h}{2}}}{\gamma}\bs{p}_0 - \alpha \int_0^r \nabla f(\bs{q}_2(s)) - \nabla f(\bs{q}(s)) ds + \int_0^r \frac{1 - e^{-\gamma (r - s)}}{\gamma} \nabla f(\bs{q}(s)) ds \\
    +& \sqrt{2\gamma}\int_0^\frac{h}{2} \frac{1 - e^{-\gamma (\frac{h}{2} - s)}}{\gamma}d\bs{B}(s) - \sqrt{2\gamma} \int_0^r \frac{1 - e^{-\gamma (r - s)}}{\gamma} d\bs{B}(s) \nonumber
\end{align*}
By Lemma \ref{lemma:lipschitz} and \ref{lemma:growth-bound}, when $0 < h < \frac{1}{4L^\prime}$, we have the following for the solution of HFHR dynamics
\[
    \mathbb{E}[\norm{\bs{x}_{0, \bs{x}_0}(h) - \bs{x}_0}^2] \le \widehat{C}_0 h
\]
where $\widehat{C}_0 = 5.14 \left\{(\alpha + \gamma)d +  h\left( L^\prime \right)^2 \norm{\bs{x}_0}^2 \right\}$ and hence
\begin{align}
    \mathbb{E}\left[ \int_0^r \norm{\nabla f(\bs{q}(s))}^2 ds \right] 
    \le& \mathbb{E} \left[ 2 \int_0^r \norm{\nabla f (\bs{q}(0))}^2 ds + 2\int_0^r \norm{\nabla f(\bs{q}(s)) - \nabla f(\bs{q}(0))}^2ds  \right] \nonumber \\
    \le& \mathbb{E}\left[ 2L^2r \norm{\bs{q}(0)}^2  + 2L^2 \int_0^r \norm{\bs{q}(s) - \bs{q}(0)}^2ds \right] \nonumber\\
    \le& 2L^2 r \norm{\bs{x}_0}^2 + 2L^2 \mathbb{E}\left[ \int_0^r \norm{\bs{q}(s) - \bs{q}(0)}^2ds \right] \nonumber\\
    \le& 2L^2 r \norm{\bs{x}_0}^2 + 2L^2\widehat{C}_0 \int_0^r sds \nonumber\\
    \le& L^2 r \left( 2\norm{\bs{x}_0}^2 + h \widehat{C}_0 \right) \nonumber\\
    \le& L^2 r \left( 2.33\norm{\bs{x}_0}^2 + 5.14 (\alpha + \gamma) dh \right) \label{eq:tool1}
\end{align}
Now $\mathbb{E}\left[ \norm{\bs{q}_2 - \bs{q} }^2 \right]$ can be bounded as follow
\begin{align*}
    & \mathbb{E} \left[\norm{ \bs{q}_2(r) - \bs{q}(r) }^2 \right] \\
    \le& 5 \left\{ \left(\frac{e^{-\gamma r} - e^{-\gamma \frac{h}{2}}}{\gamma} \right)^2\norm{\bs{p}_0}^2  + \alpha^2 \mathbb{E} \norm{\int_0^r \nabla f(\bs{q}_2(s)) - \nabla f(\bs{q}(s)) ds }^2 + \mathbb{E} \norm{ \int_0^r \frac{1 - e^{-\gamma(r-s)}}{\gamma} \nabla f(\bs{q}(s)) ds }^2 \right\} \\
    & + 5\left\{ 2\gamma \mathbb{E} \norm{ \int_0^\frac{h}{2} \frac{1 - e^{-\gamma (\frac{h}{2} - s)}}{\gamma} d\bs{B}(s) }^2 + 2\gamma \mathbb{E} \norm{ \int_0^r \frac{1 - e^{-\gamma (r - s)}}{\gamma} d\bs{B}(s) }^2\right\} \qquad \mbox{(Cauchy-Schwartz Inequality)}\nonumber \\
    \le& 5 \left\{\frac{h^2}{4}\norm{\bs{x}_0}^2  + \alpha^2 L^2 r  \int_0^r \mathbb{E} \norm{ \bs{q}_2(s) - \bs{q}(s) }^2 ds +  \int_0^r \left(\frac{1 - e^{-\gamma(r-s)}}{\gamma} \right)^2 ds \int_0^r \mathbb{E}  \norm{\nabla f(\bs{q}(s)) }^2 ds\right\} \\
    &+ 5\left\{ \frac{\gamma d h^3}{12} + \frac{2\gamma d}{3} r^3  \right\} \\
    \le& 5 \left\{ \frac{h^2}{4} \norm{ \bs{x}_0 }^2 + \alpha^2 L^2 r \int_0^r \mathbb{E} \norm{ \bs{q}_2(s) - \bs{q}(s) }^2 ds + \frac{h^3}{3} \mathbb{E} \left[ \int_0^r \norm{ \nabla f(\bs{q}(s)) }^2 \right] + \frac{3\gamma d}{4} h^3
    \right\} \\
    \le& 5 \left\{ 
    \frac{h^2}{4} \norm{ \bs{x}_0 }^2 
    + \frac{3\gamma d}{4} h^3 + \frac{h^3}{3} L^2 \left( 2.33\norm{\bs{x}_0}^2 + 5.14 (\alpha + \gamma) dh \right) r
    + \alpha^2 L^2 r \int_0^r \mathbb{E} \norm{ \bs{q}_2(s) - \bs{q}(s) }^2 ds 
    \right\} \\
    \le& 5 h^2 \left\{ 
    \frac{1}{4} \norm{ \bs{x}_0 }^2 
    + \frac{3\gamma d}{4} h 
    + \frac{h^2}{3} L^2 \left( 2.33\norm{\bs{x}_0}^2 + 5.14 (\alpha + \gamma) dh \right) \right\} 
    + 5 \alpha^2 L^2 h \int_0^r \mathbb{E} \norm{ \bs{q}_2(s) - \bs{q}(s) }^2 ds 
\end{align*}
By Gronwall's inequality and $0 < h \le \frac{1}{4 L^\prime}$, we have
\begin{align}
    \mathbb{E} \left[  \norm{ \bs{q}_2(r) - \bs{q}(r) }^2 \right] 
    \le& 5 h^2 \left\{ \frac{1}{4} \norm{\bs{x}_0}^2 + \frac{3\gamma d}{4} h + \frac{h^2}{3} L^2 \left( 2.33\norm{\bs{x}_0}^2 + 5.14 (\alpha + \gamma) dh \right) \right\}  \exp\{5\alpha^2 L^2 h^2\} \nonumber\\
    \le& 5 h^2 \left\{ \frac{1}{4} \norm{\bs{x}_0}^2 + \frac{3\gamma d}{4} h + \frac{h^2}{3} L^2 \left( 2.33\norm{\bs{x}_0}^2 + 5.14 (\alpha + \gamma) dh \right) \right\}  e^\frac{5}{32} \nonumber\\
    \le& 5.85 h^2 \left\{ 0.28 \norm{\bs{x}_0}^2 + (0.06\alpha + 0.81\gamma)h d \right\} \nonumber\\
    \le& h^2 \left\{ 1.64 \norm{\bs{x}_0}^2 + (0.36\alpha + 4.74\gamma)h d \right\}. \label{eq:tool2}
\end{align}
With bounds in Equation \eqref{eq:tool1} and \eqref{eq:tool2},  we are now ready to show $p_1$ and $p_2$. For $p_1$, i.e. the order of the mathematical expectation of deviation, we have
\begin{align*}
    &\norm{\mathbb{E}\left[ \begin{bmatrix}\hat{\bs{q}}(h) \\ \hat{\bs{p}}(h) \end{bmatrix} - \begin{bmatrix}\bs{q}(h) \\ \bs{p}(h) \end{bmatrix} \right] } \\
    \le& \norm{\mathbb{E}\left[ \hat{\bs{q}}(h)  -  \bs{q}(h) \right] } +  \norm{\mathbb{E}\left[ \hat{\bs{p}}(h)  -  \bs{p}(h) \right] } \\
    \le& \left( \alpha + \frac{1 - e^{-\gamma \frac{h}{2}}}{\gamma}  \right) \norm{ \int_0^h \mathbb{E} \left[ \nabla f(\bs{q}_2(s)) - \nabla f(\bs{q}(s)) \right]ds } + \norm{ \int_0^h \left( \frac{1 - e^{-\gamma (h - s)}}{\gamma} - \frac{1 - e^{-\gamma\frac{h}{2}}}{\gamma} \right) \mathbb{E}\left[\nabla f(\bs{q}(s))\right] ds }\\
    & + e^{-\gamma\frac{ h}{2}} \norm{\int_0^\frac{h}{2} \mathbb{E} \left[ \nabla f(\bs{q}_2(s)) - \nabla f(\bs{q}(s)) \right] ds} + \norm{\int_0^h \left( e^{-\gamma(h - s)} - e^{-\gamma \frac{h}{2}} \right)  \mathbb{E} \left[\nabla f(\bs{q}(s)) \right] ds } \\
    \le& \left( \alpha + 1 + \frac{h}{2}  \right) L \int_0^h \mathbb{E} \norm{ \bs{q}_2(s) - \bs{q}(s)} ds \\
    & + \int_0^h \left(\left| \frac{1 - e^{-\gamma (h - s)}}{\gamma} - \frac{1 - e^{-\gamma\frac{h}{2}}}{\gamma} \right| + \left|  e^{-\gamma(h - s)} - e^{-\gamma \frac{h}{2}} \right| \right) \norm{ \mathbb{E}\left[\nabla f(\bs{q}(s))\right]} ds \\
    \le& L\left( \alpha + 1 + \frac{h}{2}  \right) \int_0^h \mathbb{E} \norm{ \bs{q}_2(s) - \bs{q}(s) } ds \\
    & + \left\{ \left(\int_0^h \left| \frac{1 - e^{-\gamma (h - s)}}{\gamma} - \frac{1 - e^{-\gamma\frac{h}{2}}}{\gamma} \right|^2 ds\right)^\frac{1}{2} + \left( \int_0^h \left|  e^{-\gamma(h - s)} - e^{-\gamma \frac{h}{2}} \right|^2 ds \right)^\frac{1}{2} \right\} \left( \int_0^h \norm{ \mathbb{E}\left[\nabla f(\bs{q}(s))\right]}^2 ds \right)^\frac{1}{2} \\
    \le& L\left( \alpha + 1 + \frac{h}{2}  \right) \int_0^h \left(\mathbb{E} \norm{ \bs{q}_2(s) - \bs{q}(s) }^2 \right)^\frac{1}{2} ds + \frac{1 + \gamma}{2\sqrt{3}} h^\frac{3}{2} \left( \mathbb{E} \int_0^h \norm{ \left[\nabla f(\bs{q}(s))\right]}^2 ds \right)^\frac{1}{2} \\
    \le& L\left( \alpha + 1 + \frac{h}{2}  \right) h^2 \left\{ 1.64 \norm{\bs{x}_0}^2 + (0.36\alpha + 4.74\gamma)h d \right\}^\frac{1}{2} + \frac{1 + \gamma}{2\sqrt{3}} h^2 L \left( 2.33\norm{\bs{x}_0}^2 + 5.14 (\alpha + \gamma) dh \right)^\frac{1}{2} \\
    \le& L\left( \alpha + 1.25 \right) h^2 \left( 1.29 \norm{\bs{x}_0} + \sqrt{0.36\alpha + 4.74\gamma} \sqrt{h d} \right) + (1 + \gamma) h^2 L \left( 0.45\norm{\bs{x}_0} + \sqrt{0.43\alpha + 0.43\gamma} \sqrt{dh} \right) \\
    \le& L h^2 \max\{\alpha + 1.25, \gamma + 1\} \left(1.74 \norm{\bs{x}_0} + ( 1.26\sqrt{\alpha} + 2.84\sqrt{\gamma})\sqrt{hd}\right)
\end{align*}
The above derivation proves $p_1 = 2$ with 
\[
    \widehat{C}_1 = L \max\{\alpha + 1.25, \gamma + 1\} \left(1.74 \norm{\bs{x}_0} + ( 1.26\sqrt{\alpha} + 2.84\sqrt{\gamma})\sqrt{hd}\right).
\]

We now proceed with $p_2$, i.e. mean-square error
\begin{align*}
    &\mathbb{E} \norm{\begin{bmatrix}\hat{\bs{q}}(h) \\ \hat{\bs{p}}(h) \end{bmatrix} -  \begin{bmatrix}\bs{q}(h) \\ \bs{p}(h) \end{bmatrix}}^2 \\
    \le& 2\left( \alpha +  \frac{h}{2} \right)^2 \mathbb{E} \norm{\int_0^h \nabla f(\bs{q}_2(s)) - \nabla f(\bs{q}(s)) ds}^2 + 2\mathbb{E} \norm{ \int_0^h \left( \frac{1 - e^{-\gamma (h - s)}}{\gamma} - \frac{1 - e^{-\gamma\frac{h}{2}}}{\gamma} \right) \nabla f(\bs{q}(s)) ds }^2 \\
    & + 2 \mathbb{E} \norm{ \int_0^h \nabla f(\bs{q}_2(s)) - \nabla f(\bs{q}(s)) ds }^2 + 2 \mathbb{E} \norm{ \int_0^h \left( e^{-\gamma(h - s)} - e^{-\gamma \frac{h}{2}} \right) \nabla f(\bs{q}(s))ds }^2 \\
    \le& 2\left( (\alpha + \frac{h}{2})^2 + 1 \right) L^2 \mathbb{E} \left(\int_0^h |\bs{q}_2(s) - \bs{q}(s)| ds\right)^2 +  2\int_0^h \left| \frac{1 - e^{-\gamma (h - s)}}{\gamma} - \frac{1 - e^{-\gamma\frac{h}{2}}}{\gamma} \right|^2 ds \int_0^h \mathbb{E} \norm{\nabla f(\bs{q}(s))}^2 ds \\
    & + 2\int_0^h \left| e^{-\gamma(h - s)} - e^{-\gamma \frac{h}{2}} \right|^2 ds  \, \int_0^h \mathbb{E} \norm{ \nabla f(\bs{q}(s)) }^2 ds \\
    \le& 2\left( (\alpha + \frac{h}{2})^2 + 1 \right) L^2 h \int_0^h \mathbb{E}|\bs{q}_2(s) - \bs{q}(s)|^2 ds 
    + \frac{1 + \gamma^2}{6} h^3 \int_0^h \mathbb{E} |\nabla f(\bs{q}(s))|^2 ds\\
    \le& 2\left( (\alpha + \frac{h}{2})^2 + 1 \right) L^2  \left\{ 1.64\norm{\bs{x}_0}^2 + (0.36\alpha + 4.74\gamma) hd\right\} h^4 + \frac{1 + \gamma^2}{6} L^2 \left\{ 2.33\norm{\bs{x}_0}^2 + 5.14(\alpha + \gamma) hd \right\} h^4 \\
    \le& L^2 \max\{ (\alpha + 1.25)^2, 1 + \gamma^2\} \left( 3.67 \norm{\bs{x}_0}^2 + (1.68\alpha + 10.34\gamma)hd \right) h^4
\end{align*}
The above derivation implies $p_2 = 2$ with 
\[
    \widehat{C}_2 =  L \max\{ \alpha + 1.25, 1 + \gamma\} \left( 1.92 \norm{\bs{x}_0} + (1.30\sqrt{\alpha} + 3.22\sqrt{\gamma})\sqrt{hd} \right).
\]
\end{proof}

\subsection{Local error between Algorithm \ref{alg:HFHR} and the exact Strang's splitting method}
\label{sec_localError2}
\begin{lemma}\label{lemma:local-implementable}
Assume $f$ is $L$-smooth, $ \bs{0} \in \argmin_{\bs{x} \in \mathbb{R}^d} f(\bs{x})$, i.e. $\nabla f(\bs{0}) = \bs{0}$ and the operator $\nabla \Delta f$ grows at most linearly, i.e. $\norm{\nabla \Delta f(\bs{q})} \le G\sqrt{1 + \norm{\bs{q}}^2 }$. If $0 < h \le \frac{1}{4 L^\prime}$, then compared with the exact Strang's splitting method of HFHR dynamics, the implementable Strang's splitting method has local mathematical expectation of deviation of order $p_1 = 2$ and local mean-squared error of order $p_2 = 1.5$, i.e. there exist constants $\bar{C}_1, \bar{C}_2 > 0$ such that 
\[
    \norm{ \mathbb{E} \hat{\bs{x}}(h)  - \mathbb{E} \bar{\bs{x}}(h) } \le \bar{C}_1 h^{p_1}
\]
\[
    \left(\mathbb{E} \left[ \norm{ \hat{\bs{x}}(h) - \bar{\bs{x}}(h) }^2 \right]\right)^\frac{1}{2} \le \bar{C}_2 h^{p_2}
\]
where $\hat{\bs{x}}(h) = \begin{bmatrix} \hat{\bs{q}}(h) \\ \hat{\bs{p}}(h) \end{bmatrix}$ is the solution of the exact Strang's splitting method for HFHR with initial value $\bs{x}_0 = \begin{bmatrix} \bs{q}_0 \\ \bs{p}_0 \end{bmatrix}$ and $\bar{\bs{x}}(h) = \begin{bmatrix} \bar{\bs{q}}(h) \\ \bar{\bs{p}}(h) \end{bmatrix}$ is the one-step result of Algorithm \ref{alg:HFHR} with initial value $\bs{x}_0 = \begin{bmatrix} \bs{q}_0 \\ \bs{p}_0 \end{bmatrix}$, $p_1 = 2$ and $p_2 = 1.5$. More concretely, we have 
\[
    \bar{C}_1 = \alpha (\alpha + 1.125) (L+G) \left[ 0.5 + 0.71\norm{\bs{x}_0} + (1.14\sqrt{\alpha} + 0.21 \sqrt{\gamma} h) \sqrt{hd} \right]
\]
and 
\[
    \bar{C}_2 = L(\alpha + 0.73) \left( 2.30\sqrt{h} \alpha L \norm{\bs{x}_0} + (2.27\sqrt{\alpha} + 0.12 \sqrt{\gamma} h )\sqrt{d} \right).
\]
\end{lemma}

\begin{proof}
The solution of one-step exact Strang's  splitting integrator with step size $h$ can be written as 
\begin{equation*}
    \begin{cases}
    \bs{q}_3 = \bs{q}_2(h) + \frac{1 - e^{-\gamma \frac{h}{2}}}{\gamma} \bs{p}_2(h) + \sqrt{2\gamma} \int_\frac{h}{2}^h \frac{1 - e^{-\gamma(h - s)}}{\gamma} d\bs{B}(s)\\
    \bs{p}_3 = e^{-\gamma \frac{h}{2}}\bs{p}_2(h) + \sqrt{2\gamma} \int_\frac{h}{2}^h e^{-\gamma (h - s)} d\bs{B}(s) \\
    \bs{q}_2(r) = \bs{q}_1 - \int_0^r \alpha \nabla f(\bs{q}_2(s)) ds + \sqrt{2\alpha}\int_0^r d\bs{W}(s)  \qquad (0 \le r \le h)\\
    \bs{p}_2(r) = \bs{p}_1 - \int_0^r \nabla f(\bs{q}_2(s)) ds \\
    \bs{q}_1 = \bs{q}_0 + \frac{1 - e^{-\gamma \frac{h}{2}}}{\gamma} \bs{p}_0 + \sqrt{2\gamma} \int_0^\frac{h}{2} \frac{1 - e^{-\gamma(\frac{h}{2} - s)}}{\gamma} d\bs{B}(s)\\
    \bs{p}_1 = e^{-\gamma \frac{h}{2}}\bs{p}_0 + \sqrt{2\gamma} \int_0^\frac{h}{2} e^{-\gamma (\frac{h}{2} - s)} d\bs{B}(s)
    \end{cases}
\end{equation*}
and the solution of one-step implementable Strang's splitting integrator with step size $h$ can be written as 
\begin{equation*}
    \begin{cases}
    \bar{\bs{q}}_3 = \bar{\bs{q}}_2(h) + \frac{1 - e^{-\gamma \frac{h}{2}}}{\gamma} \bar{\bs{p}}_2(h) + \sqrt{2\gamma} \int_0^\frac{h}{2} \frac{1 - e^{-\gamma(\frac{h}{2} - s)}}{\gamma} d\bs{B}(\frac{h}{2} + s)\\
    \bar{\bs{p}}_3 = e^{-\gamma \frac{h}{2}}\bar{\bs{p}}_2(h) + \sqrt{2\gamma} \int_0^\frac{h}{2} e^{-\gamma (\frac{h}{2} - s)} d\bs{B}(\frac{h}{2} + s) \\
    \bar{\bs{q}}_2(r) = \bs{q}_1 - \int_0^r \alpha \nabla f(\bs{q}_1) ds + \sqrt{2\alpha}\int_0^r d\bs{W}(s)  \qquad (0 \le r \le h)\\
    \bar{\bs{p}}_2(r) = \bs{p}_1 - \int_0^r \nabla f(\bs{q}_1) ds \\
    \bs{q}_1 = \bs{q}_0 + \frac{1 - e^{-\gamma \frac{h}{2}}}{\gamma} \bs{p}_0 + \sqrt{2\gamma} \int_0^\frac{h}{2} \frac{1 - e^{-\gamma(\frac{h}{2} - s)}}{\gamma} d\bs{B}(s)\\
    \bs{p}_1 = e^{-\gamma \frac{h}{2}}\bs{p}_0 + \sqrt{2\gamma} \int_0^\frac{h}{2} e^{-\gamma (\frac{h}{2} - s)} d\bs{B}(s)
    \end{cases}
\end{equation*}
Note that in the implementable Strang's splitting method, $\phi$ flow can be explicitly integrated and hence $\bs{q}_1, \bs{p}_1$ are the same as that in the exact Strang's splitting method.

First, we will bound the deviation of mathematical expectation and mean squared error of $\bs{q}_2(h) - \bar{\bs{q}}_2(h)$ and $\bs{p}_2(h) - \bar{\bs{p}}_2(h)$. We have 
\begin{equation}\label{eq: tmp-1}
    \begin{cases}
        \bs{q}_2(h) - \bar{\bs{q}}_2(h) =& -\alpha \int_0^h \nabla f(\bs{q}_2(s)) - \nabla f(\bs{q}_1) ds \\
        \bs{p}_2(h) - \bar{\bs{p}}_2(h) =& -\int_0^h \nabla f(\bs{q}_2(s)) - \nabla f(\bs{q}_1) ds 
    \end{cases}
\end{equation}

Square both sides of the first equation in \eqref{eq: tmp-1} and take expectation, we obtain
\begin{align*}
    \mathbb{E}\norm{\bs{q}_2(h) - \bar{\bs{q}}_2(h)}^2 =& \alpha^2 \mathbb{E}\norm{\int_0^h \nabla f(\bs{q}_2(s)) - \nabla f(\bs{q}_1) ds}^2 \\ 
    \le& \alpha^2 \mathbb{E}\left(\int_0^h \norm{\nabla f(\bs{q}_2(s)) - \nabla f(\bs{q}_1)} ds\right)^2 \\
    \le& \alpha^2 L^2 \mathbb{E}\left(\int_0^h \norm{\bs{q}_2(s) - \bs{q}_1} ds\right)^2 \\
    \le& \alpha^2 L^2 h \int_0^h \mathbb{E}\norm{\bs{q}_2(s) - \bs{q}_1}^2 ds
\end{align*}
Note that $\bs{q}_2$ is the solution of a rescaled overdamped Langevin dynamics whose drift vector field is $\alpha L$-Lipschitz, by conditional expectation version of Lemma \ref{lemma:growth-bound}, for $0 < h < \frac{1}{4L^\prime} < \frac{1}{4\alpha L}$, we have
$
    \mathbb{E}\norm{\bs{q}_2(h) - \bs{q}_1}^2 \le \bar{C}_0 h
$ with $\bar{C}_0 = 5.14\left\{ \alpha d +  h(\alpha L)^2 \mathbb{E}\norm{\bs{q}_1}^2 \right\}$ and it follows that 
\[  
    \begin{cases}
        \mathbb{E}\norm{\bs{q}_2(h) - \bar{\bs{q}}_2(h)}^2 \le& \alpha^2 L^2 \bar{C}_0  h^3 \\
        \mathbb{E}\norm{\bs{p}_2(h) - \bar{\bs{p}}_2(h)}^2 \le& L^2 \bar{C}_0  h^3.  
    \end{cases}
\]
Now consider $p_1$, i.e., the deviation of mathematical expectation. By Ito's lemma, we have
\begin{align}
    & \bs{q}_2(h) - \bar{\bs{q}}_2(h) \nonumber \\
    =& -\alpha \int_0^h \nabla f(\bs{q}_2(s)) - \nabla f(\bs{q}_1) ds \nonumber\\
    =& -\alpha \int_0^h \left[\int_0^s -\alpha \nabla^2 f(\bs{q}_2(r)) \nabla f(\bs{q}_2(r)) dr  +  \alpha \int_0^s \nabla \Delta f(\bs{q}_2(r)) dr + \rho \right]\, ds  \label{eq:tmp-2}
\end{align}
where $\rho$ is a stochastic integral term. Take expectation and norm for Equation \eqref{eq:tmp-2}, we have 
\begin{align*}
    & \norm{\mathbb{E} \left[ \bs{q}_2(h) - \bar{\bs{q}}_2(h)   \right]} \\
    =& \alpha^2 \norm{\int_0^h \mathbb{E}\left[\int_0^s  \nabla^2 f(\bs{q}_2(r)) \nabla f(\bs{q}_2(r)) dr  - \int_0^s \nabla \Delta f(\bs{q}_2(r)) dr\right]\, ds} \\
    \le& \alpha^2 \int_0^h \mathbb{E}\left[\int_0^s  \|\nabla^2 f(\bs{q}_2(r))\|_2 \norm{\nabla f(\bs{q}_2(r))} dr  +  \int_0^s \norm{\nabla \Delta f(\bs{q}_2(r))} dr\right]\, ds \\
    \le& \alpha^2 \int_0^h \mathbb{E}\left[L  \int_0^s  \norm{\bs{q}_2(r)} dr  +  \int_0^s G (1 + \norm{\bs{q}_2(r)}) dr\right]\, ds \\
    =& \alpha^2 (L + G) \int_0^h  \int_0^s \mathbb{E}\norm{\bs{q}_2(r)}dr + \alpha^2 G\frac{h^2}{2} \\
    \le& \alpha^2 (L + G) \int_0^h  \int_0^s \mathbb{E}\norm{\bs{q}_2(r) - \bs{q}_1} + \mathbb{E} \norm{\bs{q}_1}dr + \alpha^2 G\frac{h^2}{2} \\
    \le& \alpha^2 (L + G) \int_0^h  \int_0^s \sqrt{\mathbb{E}\norm{\bs{q}_2(r) - \bs{q}_1}^2} + \mathbb{E} \norm{\bs{q}_1}dr + \alpha^2 G\frac{h^2}{2} \\
    \le& \alpha^2 (L + G) \sqrt{\bar{C}_0 h}  \frac{h^2}{2} + \alpha^2(L + G) \frac{h^2}{2} \mathbb{E}\norm{\bs{q}_1} + \alpha^2 G\frac{h^2}{2} \\
    \le& \alpha^2 \left\{ \frac{\sqrt{\bar{C}_0 h} + \mathbb{E}\norm{\bs{q}_1}}{2} (L+G) + \frac{G}{2} \right\} h^2 \\
    \le&  \frac{1}{2}\alpha^2   (L+G)  \left\{\sqrt{\bar{C}_0 h} + \mathbb{E}\norm{\bs{q}_1} + 1\right\} h^2
\end{align*}

Similarly, we have $\norm{\mathbb{E} \left[ \bs{p}_2(h) - \bar{\bs{p}}_2(h)   \right]} \le \frac{1}{2}\alpha   (L+G)  \left\{\sqrt{\bar{C}_0 h} + \mathbb{E}\norm{\bs{q}_1} + 1\right\} h^2$.

For $p_2$, i.e., mean-square error,  we have
\begin{align*}
    \mathbb{E} \norm{ \bs{q}_2(h) - \bar{\bs{q}}_2(h)}^2 \le& \alpha^2 \mathbb{E}\left\{\int_0^h \norm{\nabla f(\bs{q}_2(s)) - \nabla f(\bs{q}_1)} ds\right\}^2 \\
    \le& \alpha^2 \mathbb{E} \left\{\int_0^h 1 ds \int_0^h  \norm{\nabla f(\bs{q}_2(s)) - \nabla f(\bs{q}_1)}^2 ds \right\} \\
    \le& \alpha^2 L^2 h \int_0^h  \mathbb{E}\norm{\bs{q}_2(s) - \bs{q}_1}^2 ds \\
    \le& \frac{\alpha^2 L^2 \bar{C}_0}{2}  h^3
\end{align*}
Similarly we obtain $\mathbb{E} \norm{ \bs{p}_2(h) - \bar{\bs{p}}_2(h)}^2 \le \frac{L^2 \bar{C}_0}{2} h^3$. Recall
\[
    \begin{cases}
        \bs{q}_3 - \bar{\bs{q}}_3 = \bs{q}_2(h) - \bar{\bs{q}}_2(h) + \frac{1 - e^{-\gamma \frac{h}{2}}}{\gamma} (\bs{p}_2(h) -\bar{\bs{p}}_2(h) ) \\
        \bs{p}_3 - \bar{\bs{p}}_3 = e^{-\gamma \frac{h}{2}} (\bs{p}_2(h) - \bar{\bs{p}}_2(h))
    \end{cases}.
\]
and it follows that when $0 < h \le \frac{1}{4L^\prime} < 1$
\begin{align}
    \norm{ \mathbb{E}\begin{bmatrix} \bs{q}_3 - \bar{\bs{q}}_3 \\ \bs{p}_3 - \bar{\bs{p}}_3 \end{bmatrix}} \le& \alpha (\alpha + 1 + \frac{h}{2}) (L+G) \frac{\sqrt{\bar{C}_0 h} + \mathbb{E}\norm{\bs{q}_1} + 1 }{2}  h^2 \label{eq:tool3}\\
    \mathbb{E} \norm{ \begin{bmatrix} \bs{q}_3 - \bar{\bs{q}}_3 \\ \bs{p}_3 - \bar{\bs{p}}_3 \end{bmatrix}}^2 \le& L^2 \bar{C}_0 \left( \alpha^2 + \frac{1}{2} + \frac{h^2}{4} \right) h^3  \label{eq:tool4}.
\end{align}
Finally we need to bound $\mathbb{E}\norm{\bs{q}_1}^2$ by $\mathbb{E}\norm{\bs{x}_0}^2$, to this end, we have
\begin{align}
    \mathbb{E}\norm{\bs{q}_1}^2 =& \mathbb{E}\norm{\bs{q}_0 + \frac{1 - e^{-\gamma \frac{h}{2}}}{\gamma} \bs{p}_0 + \sqrt{2\gamma} \int_0^\frac{h}{2} \frac{1 - e^{-\gamma(\frac{h}{2} - s)}}{\gamma} d\bs{B}(s)}^2 \nonumber\\
    \le& (1 + \frac{h^2}{4}) \mathbb{E} \norm{\bs{q}_0}^2 + (1 + \frac{h^2}{4}) \mathbb{E} \norm{\bs{p}_0}^2 +  2\gamma d \int_0^\frac{h}{2} \left(\frac{1 - e^{-\gamma (\frac{h}{2} - s)}}{\gamma} \right)^2 ds \nonumber\\
    \le& (1 + \frac{h^2}{4}) \mathbb{E} \norm{\bs{x}_0}^2 + \frac{\gamma d}{12}h^3 \\
    =& (1 + \frac{h^2}{4}) \norm{\bs{x}_0}^2 + \frac{\gamma d}{12}h^3  \label{eq:tool5} 
\end{align}
Collecting all pieces together, including \eqref{eq:tool3}, \eqref{eq:tool4}, \eqref{eq:tool5}, the definition of $\bar{C}_0$ and $0 < h < \frac{1}{4L^\prime}$, it is not difficult to obtain the following 
\begin{align*}
    \norm{ \mathbb{E}\begin{bmatrix} \bs{q}_3 - \bar{\bs{q}}_3 \\ \bs{p}_3 - \bar{\bs{p}}_3 \end{bmatrix}} \le& \bar{C}_1  h^2 \\
    \left(\mathbb{E} \norm{ \begin{bmatrix} \bs{q}_3 - \bar{\bs{q}}_3 \\ \bs{p}_3 - \bar{\bs{p}}_3 \end{bmatrix}}^2\right)^\frac{1}{2} \le& \bar{C}_2 h^\frac{3}{2}
\end{align*}
with 
\[
    \bar{C}_1 = \alpha (\alpha + 1.125) (L+G) \left[ 0.5 + 0.71\norm{\bs{x}_0} + (1.14\sqrt{\alpha} + 0.21 \sqrt{\gamma} h) \sqrt{hd} \right]
\]
and 
\[
\bar{C}_2 = L(\alpha + 0.73) \left( 2.30\sqrt{h} \alpha L \norm{\bs{x}_0} + (2.27\sqrt{\alpha} + 0.12 \sqrt{\gamma} h )\sqrt{d} \right)
\]

\end{proof}

\subsection{Local error between Algorithm \ref{alg:HFHR} and HFHR dynamics}
\label{sec_localError3}
\begin{lemma}\label{lemma:local}
Assume $f$ is $L$-smooth, $\bs{0} \in \argmin_{\bs{x} \in \mathbb{R}^d} f(\bs{x})$, i.e. $\nabla f(\bs{0}) = \bs{0}$ and the operator $\nabla \Delta f$ grows at most linearly, i.e. $\norm{\nabla \Delta f(\bs{q}) }\le G\sqrt{1 + \norm{\bs{q}}^2 }$. If $0 < h \le \frac{1}{4 L^\prime}$, then compared with the HFHR dynamics, the implementable Strang's splitting method has local weak error  of order $p_1 = 2$ and local mean-squared error of order $p_2 = 1.5$, i.e. there exist constants $C_1, C_2 > 0$ such that 
\[
    \norm{ \mathbb{E} \bs{x}(h)  - \mathbb{E} \bar{\bs{x}}(h) } \le C_1 h^{p_1}
\]
\[
    \left(\mathbb{E} \left[ \norm{ \bs{x}(h) - \bar{\bs{x}}(h) }^2 \right]\right)^\frac{1}{2} \le C_2 h^{p_2}
\]
where $\bs{x}(h) = \begin{bmatrix} \bs{q}(h) \\ \bs{p}(h) \end{bmatrix}$ is the solution of HFHR with initial value $\bs{x}_0 = \begin{bmatrix} \bs{q}_0 \\ \bs{p}_0 \end{bmatrix}$ and $\bar{\bs{x}}(h) = \begin{bmatrix} \bar{\bs{q}}(h) \\ \bar{\bs{p}}(h) \end{bmatrix}$ is the solution of the implementable Strang's splitting with initial value $\bs{x}_0 = \begin{bmatrix} \bs{q}_0 \\ \bs{p}_0 \end{bmatrix}$, $p_1 = 2$ and $p_2 = 1.5$. More concretely, we have 
\[
    C_1 = (L+G) \max\{\alpha + 1.25, \gamma + 1\} \left[ 0.5\alpha + (1.74 + 0.71\alpha)\norm{\bs{x}_0} + \left( 1.26\sqrt{\alpha} + 1.14\alpha \sqrt{\alpha} + 2.32\sqrt{\gamma}\right) \sqrt{hd} \right]
\]
and 
\[
    C_2 = L \max\{\alpha + 1.25, \gamma + 1\} \left[ (1.92 + 2.30\alpha L) \sqrt{h} \norm{\bs{x}_0} + (2.60\sqrt{\alpha} + 3.34\sqrt{\gamma} h) \sqrt{d} \right]
\]
\end{lemma}
\begin{proof}
Denote by $\hat{\bs{x}}(h) = \begin{bmatrix} \hat{\bs{q}}(h) \\ \hat{\bs{p}}(h) \end{bmatrix}$ the solution of the exact Strang's splitting method with initial value $\bs{x}_0 = \begin{bmatrix} \bs{q}_0 \\ \bs{p}_0 \end{bmatrix}$. By triangle inequality and Minkowski's inequality, we have
\begin{align*}
    \norm{\mathbb{E} \bs{x}(h) - \mathbb{E}\bar{\bs{x}}(h)} \le& \norm{\mathbb{E} \bs{x}(h) - \mathbb{E}\hat{\bs{x}}(h)} + \norm{\mathbb{E} \hat{\bs{x}}(h) - \mathbb{E}\bar{\bs{x}}(h)}, \\
    \left(  \mathbb{E} \norm{\bs{x}(h) - \bar{\bs{x}}(h)}^2 \right)^\frac{1}{2} \le& \left(  \mathbb{E} \norm{\bs{x}(h) - \hat{\bs{x}}(h)}^2 \right)^\frac{1}{2} + \left(  \mathbb{E} \norm{\hat{\bs{x}}(h) - \bar{\bs{x}}(h)}^2 \right)^\frac{1}{2}.
\end{align*}
By Lemma \ref{lemma:local-exact} and \ref{lemma:local-implementable}, we have
\begin{align*}
     \norm{\mathbb{E} \bs{x}(h) - \mathbb{E}\hat{\bs{x}}(h)} \le \widehat{C}_1 h^2, &\quad \norm{\mathbb{E}  \hat{\bs{x}}(h) - \mathbb{E}\bar{\bs{x}}(h)} \le \bar{C}_1 h^2 \\
     \left(  \mathbb{E} \norm{\bs{x}(h) - \hat{\bs{x}}(h)}^2 \right)^\frac{1}{2} \le \widehat{C}_2 h^\frac{3}{2}, &\quad \left(  \mathbb{E} \norm{\hat{\bs{x}}(h) - \bar{\bs{x}}(h)}^2 \right)^\frac{1}{2} \le \bar{C}_2 h^\frac{3}{2}
\end{align*}
and hence
\begin{align*}
    \norm{\mathbb{E} \bs{x}(h) - \mathbb{E}\bar{\bs{x}}(h)} \le& (\widehat{C}_1 + \bar{C}_1) h^2 \\
    \left(  \mathbb{E} \norm{\bs{x}(h) - \bar{\bs{x}}(h)}^2 \right)^\frac{1}{2} \le& (\widehat{C}_2 + \bar{C}_2) h^\frac{3}{2}
\end{align*}
with 
\begin{align*}
    \widehat{C}_1 + \bar{C}_1 \le&  C_1  \\
    \triangleq& (L+G) \max\{\alpha + 1.25, \gamma + 1\} \left[ 0.5\alpha + (1.74 + 0.71\alpha)\norm{\bs{x}_0} + \left( 1.26\sqrt{\alpha} + 1.14\alpha \sqrt{\alpha} + 2.32\sqrt{\gamma}\right) \sqrt{hd} \right] \\
    \widehat{C}_2 + \bar{C}_2 \le& C_2 \triangleq L \max\{\alpha + 1.25, \gamma + 1\} \left[ (1.92 + 2.30\alpha L) \sqrt{h} \norm{\bs{x}_0} + (2.60\sqrt{\alpha} + 3.34\sqrt{\gamma} h) \sqrt{d} \right]
\end{align*}

\end{proof}

\section{$\alpha$ does create acceleration even after discretization: an analytical demonstration}
\label{sec_onAlpha}
If $\alpha \to \infty$ while $\gamma$ remains fixed, then $dq=-\alpha \nabla f(q)+\sqrt{2\alpha}dW$ is the dominant part of the dynamics, and in this case the role of $\alpha$ could be intuitively understood as to simply rescale the time of gradient flow, which does not create any algorithmic advantage, as the timestep of discretization has to scale like $1/\alpha$ in this case. However, finite $\alpha$ no longer corresponds to solely a time-scaling, but closely couples with the dynamics and creates acceleration. This is true even after the continuous dynamics is discretized by an algorithm .

We will analytically illustrate this point by considering quadratic $f$. In this case, the diffusion process remains Gaussian, and it suffices to quantify the convergence of its mean and covariance. In fact, it can be shown that both have the same speed of convergence, and therefore for simplicity we will only consider the mean process. Two demonstrations (with different focuses) will be provided.

\paragraph{Demonstration 1 (1D, $\gamma$ given; infinite acceleration).}
Consider $f(x)=x^2/2$, $\gamma$ fixed. The mean process is
\[ \begin{cases}
    \dot{q} &= p-\alpha q \\
    \dot{p} &= -q-\gamma p
\end{cases} \]
Consider, for simplicity, an Euler-Maruyama discretization of the HFHR dynamics, which coressponds to a Forward Euler discretization of the mean process (other numerical methods can be analyzed analogously):
\[ \begin{bmatrix} q_{k+1} \\ p_{k+1} \end{bmatrix} =
A \begin{bmatrix} q_k \\ p_k \end{bmatrix}, \qquad
A=\begin{bmatrix} 1-\alpha h & h \\ -h & 1-\gamma h \end{bmatrix}.
\]
We will show that, unless $\gamma=2$, an appropriately chosen $\alpha$ will converge infinitely faster than the case with $\alpha=0$, if both cases use the optimal $h$.

To do so, let us compute $A$'s eigenvalues, which are
\[
\frac{1}{2} \left( 2 - (\alpha+\gamma)h \pm h \sqrt{-4+(\alpha-\gamma)^2} \right)
\]
Consider the case where $|\alpha-\gamma|\leq 2$, then the eigenvalues are a pair of complex conjugates. Their modulus 
determines the speed of convergence, and it can be computed to be
\[
    \frac{1}{2}\sqrt{(2-(\alpha+\gamma)h)^2+h^2(4-(\alpha-\gamma)^2)} = \sqrt{1-(\alpha+\gamma)h+(1+\alpha\gamma)h^2}
\]
Minimizing the quadratic function gives the optimal $h$ that ensures the fastest speed of convergence, and the optimal $h$ is
\[
  h=\frac{\alpha+\gamma}{2(1+\alpha\gamma)}
\]
and the optimal spectral radius is
\[
    \sqrt{1-\frac{(\alpha+\gamma)^2}{4(1+\alpha\gamma)}}.
\]
When one uses low-resolution ODE, in which $\alpha=0$, the optimal rate is $1-\gamma^2/4$ (note it is not surprising that the critically damped case, i.e., $\gamma=2$, will give the fastest convergence).

If $\gamma\neq 2$, the additional introduction of $\alpha$ can accelerate the convergence by reducing the spectral radius. For instance, if $\alpha=\gamma+2$, upon choosing the optimal $h=\frac{1}{1+\gamma}$, the optimal spectral radius is 0 (note in this case $A$ actually has Jordan canonical form of $\begin{bmatrix} 0 & 1 \\ 0 & 0 \end{bmatrix}$ and thus the discretization converges in 2 steps instead of 1, irrespective of the initial condition).

\paragraph{Demonstration 2 (multi-dim, $\gamma$, $\alpha$ and $h$ all to be chosen; acceleration quantified in terms of condition number).}

Consider quadratic $f$ with positive definite Hessian, whose eigenvalues are $1=\lambda_1 < \cdots < \lambda_n = \epsilon^{-1}$ for some $0 < \epsilon \ll 1$. Assume without loss of generality that $f=q_1^2/2+\epsilon^{-1} q_2^2/2$. Similar to Demonstration 1, the forward Euler discretization of the mean process is
\begin{equation} \begin{bmatrix} q_{1,k+1} \\ p_{1, k+1} \\ q_{2,k+1} \\ p_{2, k+1} \end{bmatrix} = 
\begin{bmatrix} A_1 & 0 \\ 0 & A_2 \end{bmatrix} \begin{bmatrix} q_{1,k} \\ p_{1, k} \\ q_{2,k} \\ p_{2, k} \end{bmatrix}, \quad
A_1=\begin{bmatrix} 1-\alpha h & h \\ -h & 1-\gamma h \end{bmatrix}, \quad
A_2=\begin{bmatrix} 1-\alpha \epsilon^{-1} h & h \\ -\epsilon^{-1} h & 1-\gamma h \end{bmatrix}
\label{eq_conditionNumberSec_iter}
\end{equation}
We will (i) find $h$ and $\gamma$ that lead to fastest convergence of the ULD discretization, i.e. the above iteration with $\alpha=0$, and then (ii) constructively show the existence of $h$, $\gamma$ and $\alpha$ that lead to faster convergence than the optimal one in (i) --- note these may not even be the optimal choices for HFHR, but they already lead to significant acceleration. More specifically,

\textbf{\underline{(i)}} In a ULD setup, $\alpha=0$. It can be computed that the eigenvalues of $A_1$ and $A_2$ are respectively
\[
\frac{1}{2} \left( 2-h\gamma\pm h\sqrt{-4+\gamma^2} \right)   \qquad\text{and}\qquad
\frac{1}{2} \left( 2-h\gamma\pm h\sqrt{-4\epsilon^{-1}+\gamma^2} \right)
\]
We now seek $\gamma>0,h>0$ to minimize the maximum of their norms for obtaining the optimal convergence rate. This is done in cases.

\noindent \underline{Case (i1)} When $\gamma \leq 2$, both $A_1$ and $A_2$ eigenvalues are complex conjugate pairs. To minimize the maximum of their norms, let's first see if their norms could be made equal.

$A_1$ eigenvalue's norm squared $\times 4$ is
\begin{equation}
    (2-h\gamma)^2-h^2(-4+\gamma^2) = 4(h-\gamma/2)^2+4-\gamma^2
    \label{ugv9317ughoiuhoig1314}
\end{equation}
$A_2$ eigenvalue's norm squared $\times 4$ is
\begin{equation}
    (2-h\gamma)^2-h^2(-4\epsilon^{-1}+\gamma^2) = 4\epsilon^{-1}(h-\epsilon\gamma/2)^2+4-\epsilon\gamma^2
    \label{qbogoyiuhb31ogyiubo1}
\end{equation}
It can be seen that for \eqref{ugv9317ughoiuhoig1314} is always strictly smaller than \eqref{qbogoyiuhb31ogyiubo1} for any $h>0$. Therefore, the max of the two is minimized when $h=\epsilon \gamma/2$, and the corresponding max value is $4-\epsilon \gamma^2$. $\gamma$ that minimizes this max value is $\gamma=2$. Corresponding rate of convergence is
\[
    \sqrt{1-\epsilon}.
\]

\noindent \underline{Case (i2)} When $\gamma \geq 2\epsilon^{-1/2}$, both $A_1$ and $A_2$ eigenvalues are real. Since $\epsilon\ll 1$, we can order them$\times 2$ as
\[
    2-h\gamma-h\sqrt{-4+\gamma^2} < 
    2-h\gamma-h\sqrt{-4\epsilon^{-1}+\gamma^2} <
    2-h\gamma+h\sqrt{-4\epsilon^{-1}+\gamma^2} <
    2-h\gamma+h\sqrt{-4+\gamma^2} < 2.
\]
To minimize the max of their norms, consider cases in which the smallest of four is negative, in which case at optimum one should have
\[
    -(2-h\gamma-h\sqrt{-4+\gamma^2}) = 2-h\gamma+h\sqrt{-4+\gamma^2}.
\]
This gives $h=2/\gamma$ (which does verify the assumption that the smallest of four is negative). Corresponding max of their norms is thus $\sqrt{1-4/\gamma^2}$. $\gamma$ that minimizes this max value is $\gamma=2\epsilon^{-1/2}$, which gives rate of convergence of
\[
    \sqrt{1-\epsilon}.
\]

\noindent \underline{Case (i3)} When $2 \leq \gamma \leq 2\epsilon^{-1/2}$, $A_1$ eigenvalues are real and $A_2$ eigenvalues are complex conjugates. Again, the max of their norms is minimized if the norms can be made all equal.

Note $A_1$ eigenvalues cannot be of the same sign, because otherwise $2-h\gamma-h\sqrt{-4+\gamma^2}=2-h\gamma+h\sqrt{-4+\gamma^2}$, which means either $h=0$ or $\gamma=2$, but if $\gamma=2$ then $2-h\gamma+h\sqrt{-4+\gamma^2}$ being equal to 2*norm of $A_2$ eigenvalue, which is $\sqrt{4\epsilon^{-1}(h-\epsilon\gamma/2)^2+4-\epsilon\gamma^2}$, leads to $h=0$ again.

Therefore, the equality of norms of $A_1$, $A_2$ eigenvalues means
\[
  -(2-h\gamma-h\sqrt{-4+\gamma^2})
  =2-h\gamma+h\sqrt{-4+\gamma^2}
  =\sqrt{4\epsilon^{-1}(h-\epsilon\gamma/2)^2+4-\epsilon\gamma^2}.
\]
The first equality gives $h\gamma=2$, which, together with the second equality, gives $h=\pm\sqrt{\frac{2\epsilon}{1+\epsilon}}$. Selecting the positive value of optimal $h$, we also obtain optimal $\gamma=\sqrt{2(1+\epsilon)}\epsilon^{-1/2}$, which is $\leq 2 \epsilon^{-1/2}$ and thus satisfying our assumption ($2 \leq \gamma \leq 2\epsilon^{-1/2}$). The corresponding rate of convergence is thus
\[
    \frac{1}{2} \left(2-h\gamma+h\sqrt{-4+\gamma^2}\right) = \sqrt{\frac{1-\epsilon}{1+\epsilon}}.
\]

\noindent \underline{Summary of (i)} Since $\sqrt{\frac{1-\epsilon}{1+\epsilon}} < \sqrt{1-\epsilon}$, the ULD Euler-Maruyama discretization converges the fastest when
\[
    h=\sqrt{\frac{2\epsilon}{1+\epsilon}}, \qquad
    \gamma=\sqrt{2(1+\epsilon)}\epsilon^{-1/2},
\]
and the corresponding discount factor of convergence (i.e. base of exponential convergence) is
\begin{equation}
    \sqrt{\frac{1-\epsilon}{1+\epsilon}},   \qquad \text{where $\epsilon=1/\kappa$ with $\kappa$ being Hessian's condition number.}
    \label{eq_ULDoptimalRate_conditionNum}
\end{equation}

\textbf{\underline{(ii)}} Now consider the HFHR setup. Let's first state a result: when
\begin{align}
    & \gamma=\frac{\sqrt{4 c^2 \epsilon ^4+8 c^2 \epsilon ^3+4 c^2 \epsilon ^2+\epsilon ^2-2 \epsilon +1}+\epsilon +3}{2 c \epsilon ^2+2 c \epsilon} > 0, \\
    & \alpha=\frac{-\sqrt{4 c^2 \epsilon ^4+8 c^2 \epsilon ^3+4 c^2 \epsilon ^2+\epsilon ^2-2 \epsilon +1}+3 \epsilon +1}{2 c \epsilon ^2+2 c \epsilon} > 0, \qquad
    h=c \epsilon
    \label{eq_HFHRnonoptimalParameters}
\end{align}
for any $c>0$ independent of $\epsilon$, the iteration \eqref{eq_conditionNumberSec_iter} converges with discount factor
\begin{equation}
    \frac{1}{\sqrt{2} (1+\epsilon)} \sqrt{(1-\epsilon) \left(1-\epsilon+\sqrt{4 c^2 \epsilon ^4+8 c^2 \epsilon ^3+\left(4 c^2+1\right) \epsilon ^2-2 \epsilon +1}\right)}.
    \label{eq_HFHRnonoptimalRate_conditionNum}
\end{equation}
While the exact expression is lengthy, it can proved that the HFHR non-optimal discount factor \eqref{eq_HFHRnonoptimalRate_conditionNum} is strictly smaller than the ULD optimal discount factor \eqref{eq_ULDoptimalRate_conditionNum} for not only small but also large $\epsilon$'s.

For some quantitative intuition, discount factors respectively have the following Taylor expansions in $\epsilon$:
\begin{align}
    & \text{HFHR non-optimal:}   &&
    1-2 \epsilon +\left(\frac{c^2}{2}+2\right) \epsilon ^2 +\mathcal{O}\left(\epsilon ^3\right) \qquad\qquad \\
    & \text{ULD optimal:}     && 1-\epsilon+\frac{\epsilon^2}{2} +\mathcal{O}\left(\epsilon ^3\right) \qquad\qquad
\end{align}

\begin{figure}[h]
    \vspace{-135pt}
    \centering
    \includegraphics[width=0.8\textwidth]{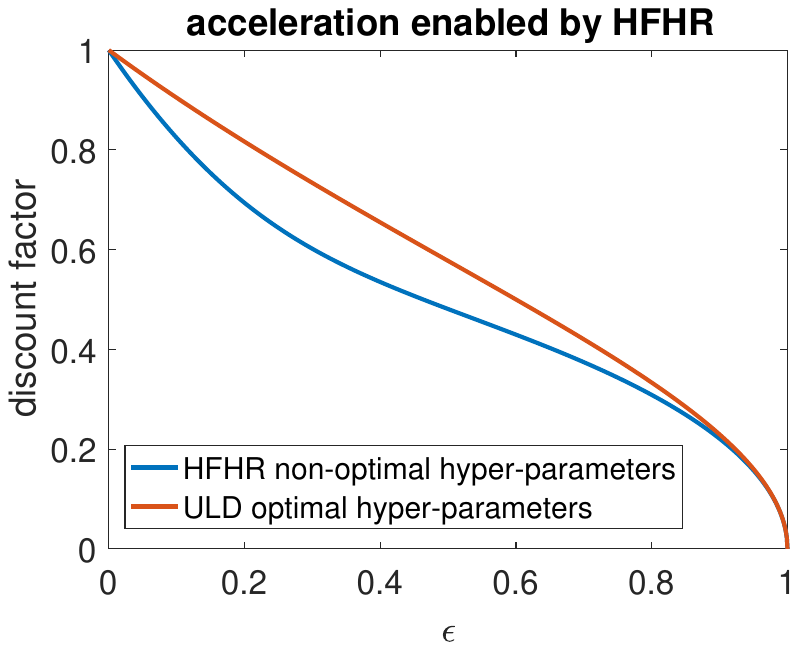}
    \vspace{-135pt}
    \caption{Acceleration of HFHR algorithm over ULD algorithm (despite of an additional constraint $\alpha$ may place on $h$) for multi-dimensional quadratic objectives. $1/\epsilon$ is the condition number.}
    \label{fig_discountFactors}
\end{figure}
The exact expressions of discount factors are also plotted  in Fig.\ref{fig_discountFactors} ($c=1$ was arbitrarily chosen) and one can see acceleration for any (not necessarily small) $\epsilon$.

\textbf{\underline{(ii details)}} How were values in \eqref{eq_HFHRnonoptimalParameters} chosen? Following the idea detailed in (i), we consider a case where $A_1$ eigenvalues are both real, $A_2$ eigenvalues are complex conjugates, and all their norms are equal. Note there are 3 more cases, namely real/real, complex/real, and complex/complex, but we do not optimize over all cases for simplicity --- the real/complex case is enough for outperforming the optimal ULD.

This case leads to at least the following equations
\begin{equation} \begin{cases}
    \text{tr} A_1 &= 0 \\
    \det A_1 + \det A_2 &= 0
\end{cases} 
\label{agijqhrogubvqorehgb97413hbt8gob}
\end{equation}
One can solve this system of equations to obtain $\alpha$ and $\gamma$ as functions of $h$. Following the idea of choosing $h$ small enough to resolve the stiffness of the ODE
\[\begin{cases}
    \dot{q}_2 &= p_2 - \alpha \epsilon^{-1} q_2 \\
    \dot{p}_2 &= -\epsilon^{-1} q_2 - \gamma p_2
\end{cases},\]
pick $h=c \epsilon$. Then \eqref{agijqhrogubvqorehgb97413hbt8gob} gives
\begin{align*}
    & \gamma=\frac{\sqrt{4 c^2 \epsilon ^4+8 c^2 \epsilon ^3+4 c^2 \epsilon ^2+\epsilon ^2-2 \epsilon +1}+\epsilon +3}{2 c \epsilon ^2+2 c \epsilon} \\
    & \alpha=\frac{-\sqrt{4 c^2 \epsilon ^4+8 c^2 \epsilon ^3+4 c^2 \epsilon ^2+\epsilon ^2-2 \epsilon +1}+3 \epsilon +1}{2 c \epsilon ^2+2 c \epsilon}
\end{align*}
or
\begin{align*}
    & \gamma = \frac{-\sqrt{4 c^2 \epsilon ^4+8 c^2 \epsilon ^3+4 c^2 \epsilon ^2+\epsilon ^2-2 \epsilon +1}+\epsilon +3}{2 c \epsilon ^2+2 c \epsilon } \\
    & \alpha = \frac{\sqrt{4 c^2 \epsilon ^4+8 c^2 \epsilon ^3+4 c^2 \epsilon ^2+\epsilon ^2-2 \epsilon +1}+3 \epsilon +1}{2 c \epsilon ^2+2 c \epsilon }
\end{align*}
The former is our choice \eqref{eq_HFHRnonoptimalParameters} because it can be checked that the latter leads to $\det A_1 > 0$ which violates the assumption of a pair of plus and minus real eigenvalues.

It is possible to find optimal $\alpha,\gamma,h$ for HFHR for the Gaussian cases. One has to minimize $\det A_2$ under the constraint $\det A_2 > 0$ in addition to \eqref{agijqhrogubvqorehgb97413hbt8gob}. And then do similar calculations for the other 3 cases, and then finally the best among the 4 cases. Doing so however does not give enough insights to determine optimal hyperparameters for sampling general distributions.

\section{Randomized Midpoint Discretization of HFHR}
\label{sec:RMA_HFHR}
\subsection{The algorithm}
HFHR is based on a continuous dynamics that adds HFHR corrections to the Underdamped Langevin Dynamics (ULD). It can be turned into a sampling algorithm via either a low-order time discretization (e.g., HFHR Algorithm \ref{alg:HFHR}) or a more accurate one. To complement the main text, this section demonstrates the latter, based on a powerful recent progress in discretizing ULD, known as Randomized Midpoint Algorithm (RMA) \citep{shen2019randomized}, and shows that the acceleration created by the HFHR correction terms persists.

More specifically, RMA is a high-order discretization scheme for ULD that achieved a better $\mathcal{O}(d^\frac{1}{3})$ dimension dependence of mixing time than first-order discretization of ULD, e.g., 1st-order KLMC \citep{dalalyan2018sampling}. Although RMA is originally designed specifically for ULD only, it is a general idea and already adapted to overdamped Langevin \citep{he2020ergodicity}. Here we show it can be easily adapted to HFHR as well, as illustrated by the following Algorithm \ref{alg:HFHR-RMA}. Red highlights algorithmic changes we made to account for the HFHR corrections of ULD.

\begin{algorithm}[h]
\caption{Randomized Midpoint Algorithm from \cite{shen2019randomized}, adapted for HFHR}\label{alg:HFHR-RMA}
\begin{algorithmic}[1]
\STATE \textbf{Input}: potential function $f$ and its gradient $\nabla f$, damping coefficients $\alpha$ and $\gamma$, step size $h$, initial condition $(\qbold_0, \pbold_0)$
\PROCEDURE{RMA-HFHR}{$f, \nabla f, \alpha, \gamma, h, \qbold_0, \pbold_0$}
    \STATE $k=0$ and initialize $\begin{bmatrix} \bs{q}_0 \\ \bs{p}_0 \end{bmatrix}$
    \WHILE{not converged}
        \STATE Generate an independent uniform random variable $\theta_k \sim U(0,1)$
        \STATE Generate Gaussian random vectors $\left( \bs{W}^1_{k+1}, \bs{W}^2_{k+1}, \bs{W}^3_{k+1} \right) \in \mathbb{R}^{3d}$ as in \citep[Appendix A]{shen2019randomized}
        \STATE 
        {\color{red} Generate Gaussian random vectors $\bs{B}^1_{k+1}, 
        \bs{B}^2_{k+1} \in \mathbb{R}^d$ as described by \eqref{eq:gfvudihgo8u134yb4ogfv1li}}
        \STATE $\bs{q}_{k+\frac{1}{2}} = \bs{q}_k + \frac{1}{\gamma}(1-e^{-\gamma\theta_k h})\bs{p}_k - \frac{1}{\gamma}\left(\theta_k h - \frac{1}{\gamma}(1-e^{-\gamma\theta_k h})\right) \nabla f(\bs{q}_k) + \bs{W}^1_{k+1} $ $\color{red} - \alpha \theta_k h \nabla f(\bs{q}_k) + \sqrt{2\alpha} \bs{B}^1_{k+1} $
        \STATE $\bs{q}_{k+1} = \bs{q}_k + \frac{1}{\gamma}(1-e^{-\gamma h}) \bs{p}_k - \frac{1}{\gamma}h(1-e^{-\gamma(h-\theta_k h)}) \nabla f(\bs{q}_{k+\frac{1}{2}}) + \bs{W}^2_{k+1} \color{red} - \alpha h \nabla f(\bs{q}_{k+\frac{1}{2}}) + \sqrt{2\alpha} (\bs{B}^1_{k+1} + \bs{B}^2_{k+1})$
        \STATE $\bs{p}_{k+1} = \bs{p}_{k} e^{-\gamma h} - h e^{-\gamma(h-\theta_k h)} \nabla f(\bs{q}_{k+\frac{1}{2}}) + 2\bs{W}^3_{k+1}$
        \STATE $k \gets k + 1$
    \ENDWHILE
\ENDPROCEDURE
\end{algorithmic}
\end{algorithm}
The red parts basically correspond to two Euler-Maruyama time-steppings of an auxiliary dynamics that contains only the HFHR correction terms
\begin{equation}
    d\bs{q} = -\alpha \nabla f(\bs{q}) dt + \sqrt{2\alpha} d\bs{B}_t,
    \label{eq:HFHRcorrectionOnly}
\end{equation}
first over a $\theta_k h$ timestep, and then over an $h$ timestep. These two steps originate from an operator splitting treatment of the full HFHR dynamics (eq.\ref{eq:HFHR}), which is split into ULD and 
\eqref{eq:HFHRcorrectionOnly}. Therefore, it is natural to see that
\[
    \bs{B}_{k+1}^1 = \int_{hk}^{h(k+\theta_k)} d\bs{B}_t, \qquad
    \bs{B}_{k+1}^2 = \int_{h(k+\theta_k)}^{h(k+1)} d\bs{B}_t,
\]
and therefore $B_{k+1}^1$ and $B_{k+1}^2$ are, when conditioned on $\theta_k$, centered Gaussian vectors independent from each other and the $\bs{W}$'s, each being $d$-dimensional with i.i.d. entries, and they can be generated via
\begin{equation}
    \bs{B}_{k+1}^1 = \sqrt{\theta_k h} \bs{\xi}_{k+1}^1, \qquad
    \bs{B}_{k+1}^2 = \sqrt{h-\theta_k h} \bs{\xi}_{k+1}^2,
    \label{eq:gfvudihgo8u134yb4ogfv1li}
\end{equation}
where $\bs{\xi}_{k+1}^1$ and $\bs{\xi}_{k+1}^2$ are i.i.d. standard d-dimensional Gaussian vectors.

\begin{remark}
In the original RMA \citep[Algorithm 1]{shen2019randomized}, the uniform random variable for the midpoint's proportional location was denoted by $\alpha$. However, since we have already used this letter for the HFHR correction coefficient, we use instead $\theta$ to denote this uniform random variable. \end{remark}

\begin{remark}
From the red text, it is easy to see that if $\alpha = 0$, Algorithm \ref{alg:HFHR-RMA} degenerates to RMA for ULD. Nevertheless, Algorithm \ref{alg:HFHR-RMA} is again just one RMA discretization of HFHR but not the only one.
\end{remark}

\subsection{Numerical results: HFHR again accelerates}
To numerically compare the RMA discretization of HFHR dynamics and ULD dynamics (note we don't compare 1st-order HFHR Algorithm \ref{alg:HFHR} with RMA-ULD as we'd like to compare apple with apple), we conduct an experiment very similar to that in Sec.\ref{subsec:verify_algorithm}, with the same nonlinear potential function. We run both RMA for ULD and RMA for HFHR with dimension $d=10$,  initial value $(100 \times \bs{1}_d, \bs{0}_d)$, $h=1$ (chosen to be near the stability limit of RMA-ULD), a family of $\gamma \in \{0.1, 0.2, 0.5, 1, 2, 5, 10, 20, 50, 100\}$ and $\alpha \in \{$0, 0.001, 0.002, 0.005, 0.01, 0.02, 0.05, 0.1, 0.2, 0.5, 0.55, 0.6, 0.65,0.7,0.75,0.8,0.85,0.9,0.95, 1, 2, 5, 10, 20, 50, 100$\}$. For each algorithm and each set of parameter values, we run 1,000 independent realizations to compute statistics and estimate the mean time of reaching $\varepsilon=0.1$ neighborhood of the target distribution. Then, for each $\alpha$ (including $\alpha=0$, which is the original RMA), we optimize over $\gamma$ choices to get the best results. To further reduce variance, we also repeat the experiment with 100 different random seeds.
 
Too large $\alpha$ values with which Algorithm \ref{alg:HFHR-RMA} fails to reach $\epsilon$-neighborhood are not plotted and the final results are shown in Figure \ref{fig:HFHR_RMA_iteration_complexity}. It clearly suggests that with appropriated chosen $\alpha$ ($\alpha=0.5$ in our case), RMA discretized HFHR dynamics requires fewer iterations than RMA discretized ULD, which suggests a better iteration complexity. 
\begin{figure}[h]
    \centering
    \includegraphics[width=0.8\textwidth]{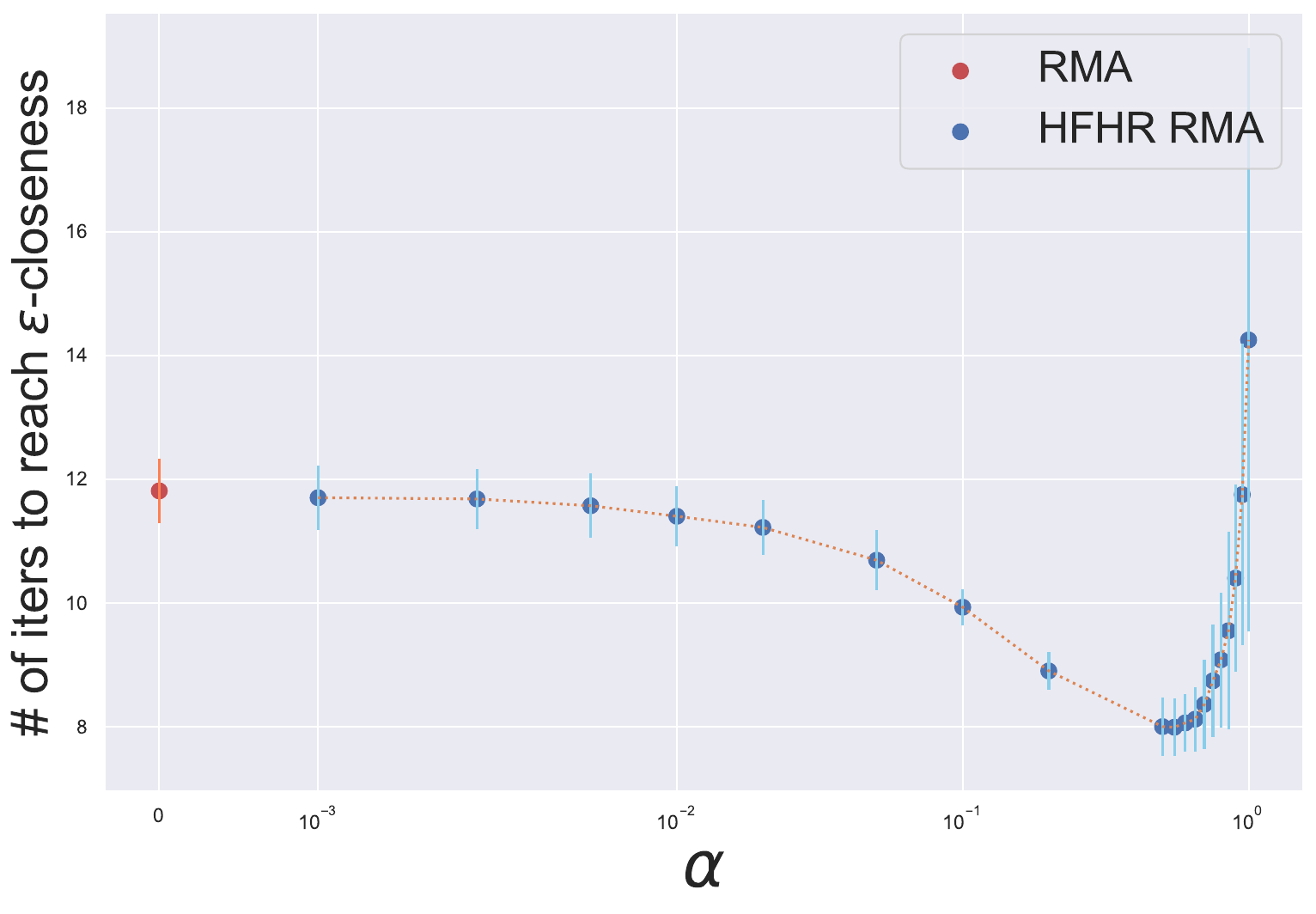}
    \caption{Improvement of RMA for HFHR (Algorithm \ref{alg:HFHR-RMA}) over the original RMA (for ULD) in iteration  complexity. (vertical bar = 1 standard deviation)}
    \label{fig:HFHR_RMA_iteration_complexity}
\end{figure}
 
\end{document}